\documentclass[accepted]{uai2023} 

\usepackage[american]{babel}

\usepackage{natbib} 
    
\usepackage{mathtools} \usepackage{booktabs} \usepackage{tikz}

\usepackage{microtype}
\usepackage{graphicx}
\usepackage{subcaption}
\usepackage{booktabs} \usepackage{multirow}

\usepackage[algo2e,linesnumbered,lined,boxed,commentsnumbered,noend]{algorithm2e}

\SetCommentSty{mycommfont}
\usepackage{amsmath}
\usepackage{amssymb}
\usepackage{mathtools}
\usepackage{amsthm}

\usepackage[capitalize,noabbrev]{cleveref}

\theoremstyle{plain}
\newtheorem{theorem}{Theorem}[section]
\newtheorem{proposition}[theorem]{Proposition}
\newtheorem{lemma}[theorem]{Lemma}
\newtheorem{corollary}[theorem]{Corollary}
\theoremstyle{definition}
\newtheorem{definition}[theorem]{Definition}

\theoremstyle{remark}

\usepackage{nicefrac}
\definecolor{offWhite}{RGB}{240,240,240}
\definecolor{grey}{RGB}{180,180,180}
\definecolor{lightGrey}{RGB}{220,220,220}
\definecolor{darkgreen}{RGB}{0,125,0}
\definecolor{lime}{RGB}{255,200,0}

\definecolor{amiiBlue}{RGB}{16,72,118}
\definecolor{amiiPink}{RGB}{241,97,119}
\definecolor{amiiYellow}{RGB}{248,209,109}
\definecolor{amiiPurple}{RGB}{123,105,145}

\colorlet{yes}{cyan!50!white}
\colorlet{newYes}{cyan!75!white}
\colorlet{no}{red!50!white}
\colorlet{newNo}{red!75!white}

\usepackage{amsmath, amssymb, amsfonts, amsthm}
\makeatletter
\ifdefined\theorem
\else
  \newtheorem{theorem}{Theorem}
\fi
\ifdefined\lemma
\else
  \newtheorem{lemma}{Lemma}
\fi
\ifdefined\corollary
\else
  \newtheorem{corollary}{Corollary}
\fi
\ifdefined\definition
\else
  \newtheorem{definition}{Definition}
\fi
\ifdefined\proposition
\else
  \newtheorem{proposition}{Proposition}
\fi
\makeatother

\makeatletter
\newcommand\safeIncCounter[1]{\@ifundefined{c@#1}{\newcounter{#1}\stepcounter{#1}}{\stepcounter{#1}}}
\makeatother

\usepackage{xargs}
 \usepackage{mathtools}

\DeclarePairedDelimiter{\abs}{\lvert}{\rvert}
\DeclarePairedDelimiter{\norm}{\lVert}{\rVert}
\DeclarePairedDelimiter{\subex}{(}{)}
\DeclarePairedDelimiter{\subblock}{[}{]}
\DeclarePairedDelimiter{\tuple}{\langle}{\rangle}
\DeclarePairedDelimiter{\set}{\{}{\}}

\DeclarePairedDelimiter{\ramp}{[}{]_+}

\usepackage{stmaryrd}

\usepackage{bbold}
\usepackage{bm}

\newcommand{\reals}{\mathbb{R}}
\newcommand{\naturals}{\mathbb{N}}

\newcommand{\Simplex}{\triangle}
\newcommand{\simplex}{\Simplex}

\newcommand{\bs}[1]{\bm{#1}}

\newcommand{\expectation}{\mathbb{E}}
\newcommand{\E}{\expectation}
\newcommand{\probability}{\mathbb{P}}
\newcommand{\Prob}{\probability}

\newcommand{\as}{\doteq}

\newcommand{\ones}{\bs{1}}
\newcommand{\zeros}{\bs{0}}

\newcommand{\bigO}{\operatorname{\mathcal{O}}}

\newcommand{\given}{\,|\,}

\newcommand{\where}{\;|\;}

\DeclareMathOperator*{\argmin}{arg\,min}
\DeclareMathOperator*{\argmax}{arg\,max}
\DeclareMathOperator*{\e}{e}
\DeclareMathOperator*{\unif}{Unif}
\DeclareMathOperator*{\Unif}{\unif}

\newcommand{\gap}{\varepsilon}

\newcommand{\Actions}{\mathcal{A}}

\newcommand{\regret}{\rho}
\newcommand{\Regret}{R}

\newcommand{\maxLoss}{L}
\newcommand{\grad}{\nabla}

\usepackage{amssymb}

\newcommand{\cfIv}{v}

\newcommand{\cfv}{\cfIv}

\usepackage{upgreek}

\newcommand{\DecisionSet}{\Theta}
\newcommand{\odpDecision}{\theta}

\newcommand{\policy}{\pi}

\newcommand{\PolicySet}{\Pi}

\newcommand{\kOfN}{k\text{-of-}N}
\newcommand{\kOfNMeasure}{\mu_{\kOfN}}

\newcommand{\LossSet}{\mathcal{L}}

\DeclareMathOperator*{\SortFn}{Sort}
\DeclareMathOperator*{\SortBy}{SortBy}

\DeclareMathOperator*{\SuccessiveRejects}{SR}
\DeclareMathOperator*{\WorstKOfNLossesFn}{\LossSet_{\kOfN}}
\def\percentile{\eta}
\DeclareMathOperator*{\WorstFactileLossesFn}{\LossSet_{\percentile}}
\DeclareMathOperator*{\supportFn}{supp}

\newcommand{\mTestCasesToSelect}{m}
\newcommand{\jointTestCaseDistributionPolicyUncertainty}{\Psi}
\newcommand{\testCaseSet}{\mathcal{T}}
\newcommand{\testCaseGroupDecisionLabel}{\testCaseSet}
\newcommand{\testCase}{c}
\newcommand{\testCaseGroup}{\tau}
\newcommand{\testCaseDistribution}{\sigma}
\newcommand{\testTuple}{\tuple{\testCaseGroup, \hat{\testCaseDistribution}_{\testCaseGroup}}}

\newcommand{\tuningLabel}{\textsc{tnp}}
\newcommand{\deploymentLabel}{\textsc{cdp}}
\newcommand{\simLabel}{\textsc{sim}}
\newcommand{\seqLabel}{\textsc{seq}}
\newcommand{\differenceFn}{\mathit{\Delta}}
\newcommand{\maxGrad}{G}
\def\cvarPercentile{1\%}
\def\probMeasure{\mu}
\def\IntegrableFnSet{\mathcal{Y}}
\def\integrableFn{y}
\def\numHoldoutReplicas{100}
 
\usepackage{xspace}

\makeatletter
\DeclareRobustCommand\onedot{\futurelet\@let@token\@onedot}
\def\@onedot{\ifx\@let@token.\else.\null\fi\xspace}

\def\eg/{\emph{e.g}\onedot} \def\Eg/{\emph{E.g}\onedot}
\def\ie/{\emph{i.e}\onedot} \def\Ie/{\emph{I.e}\onedot}
\def\cf/{\emph{c.f}\onedot} \def\Cf/{\emph{C.f}\onedot}
\def\vs/{\emph{vs}\onedot} \def\Vs/{\emph{Vs}\onedot}
\def\etc/{\emph{etc}\onedot}
\def\wrt/{with respect to} \def\dof/{d.o.f\onedot}
\def\etal/{\emph{et al}\onedot}
\def\viceversa/{\emph{vice-versa}}
\def\ow/{\emph{o.w}\onedot}
\def\whp/{w.h.p\onedot}
\def\apriori/{\emph{a priori}} \def\Apriori/{\emph{A priori}}
\def\ala/{\`{a} la}

\def\naive/{na\"{\i}ve} \def\Naive/{Na\"{\i}ve}
\def\rmPlus/{regret matching\textsuperscript{+}}
\def\rrmPlus/{RRM\textsuperscript{+}}
\def\rcfrPlus/{RCFR\textsuperscript{+}}
\def\cfrPlus/{CFR\textsuperscript{+}}
\@ifdefinable{\Politex/}{\def\Politex/{\textsc{Politex}}}

\def\NashConv/{\textsc{NashConv}}
\def\NashConvAUC/{$\overline{\textsc{NashConv}}$}

\def\heads/{\textsc{heads}}
\def\tails/{\textsc{tails}}
\def\even/{\textsc{even}}
\def\odd/{\textsc{odd}}

\makeatother

\def\rampName/{ramp}
\def\RampName/{Ramp}
\def\GranTurismo7/{Gran Turismo\textsuperscript{\texttrademark} 7}
 \expandafter\newif\csname ifGin@setpagesize\endcsname

\title{Composing Efficient, Robust Tests for Policy Selection}

\author[1]{\href{mailto:<dustin.morrill@sony.com>?Subject=Your UAI 2023 paper}{Dustin Morrill}}
\author[1]{Thomas J. Walsh}
\author[1]{Daniel Hernandez}
\author[1]{Peter R. Wurman}
\author[1,2]{Peter Stone}
\affil[1]{Sony AI\\
    New York, NY, USA
}
\affil[2]{Department of Computer Science\\
    The University of Texas at Austin\\
    Austin, TX USA
}

\begin{document}
\maketitle
\begin{abstract}
Modern reinforcement learning systems produce many high-quality policies throughout the learning process. However, to choose which policy to actually deploy in the real world, they must be tested under an intractable number of environmental conditions. We introduce RPOSST, an algorithm to select a small set of test cases from a larger pool based on a relatively small number of sample evaluations. RPOSST treats the test case selection problem as a two-player game and optimizes a solution with provable $k$-of-$N$ robustness, bounding the error relative to a test that used all the test cases in the pool.
Empirical results demonstrate that RPOSST finds a small set of test cases that identify high quality policies in a toy one-shot game, poker datasets, and a high-fidelity racing simulator.
\end{abstract}

\section{Introduction}\label{section:introduction}

\begin{figure*}[h]
  \centering
  \begin{subfigure}[t]{0.16\linewidth}
\includegraphics[width=\linewidth, clip, trim=0em -9em 0em 0cm, keepaspectratio]{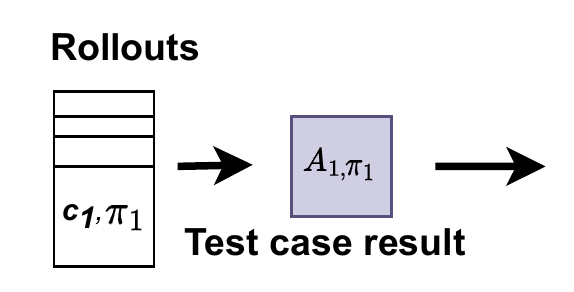}
      \label{fig:test_result_compilation}
  \end{subfigure}
  \begin{subfigure}[t]{0.26\linewidth}
\includegraphics[width=\linewidth, clip, trim=0em 0em 0.5em 0cm, keepaspectratio]{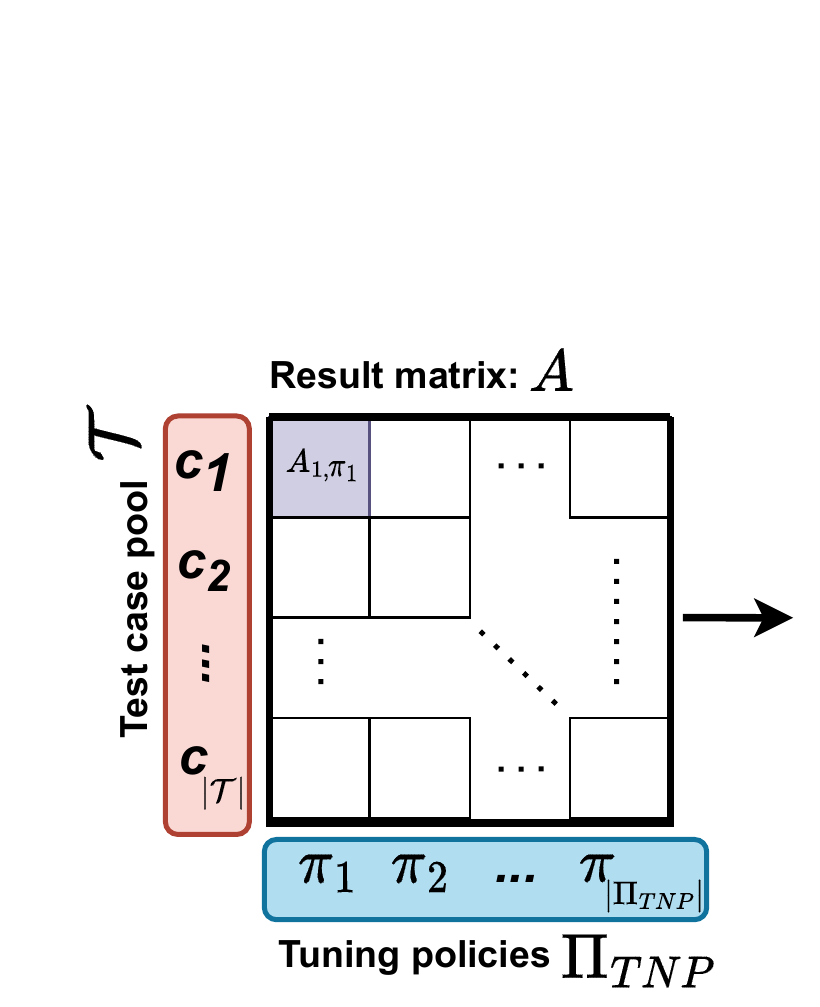}
      \label{fig:payoff_matrix_generation}
  \end{subfigure}
  \begin{subfigure}[t]{0.26\linewidth}
      \includegraphics[width=\linewidth, clip, trim=0em 2.5em 11.5em 0cm, keepaspectratio]{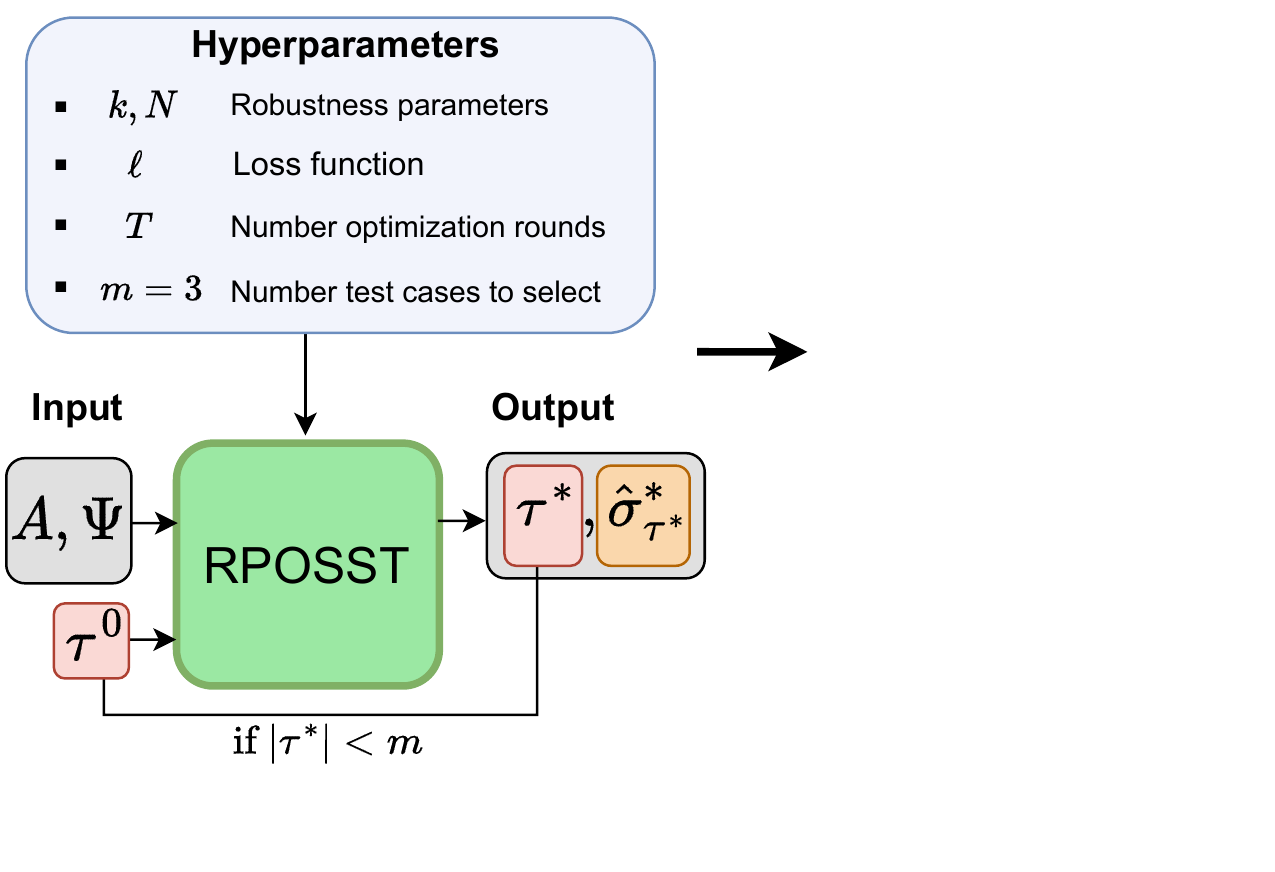}
      \label{fig:algorithm_description}
  \end{subfigure}
  \begin{subfigure}[t]{0.26\linewidth}
      \includegraphics[width=\linewidth, clip, trim=2.0em 0em 1em 0cm, keepaspectratio]{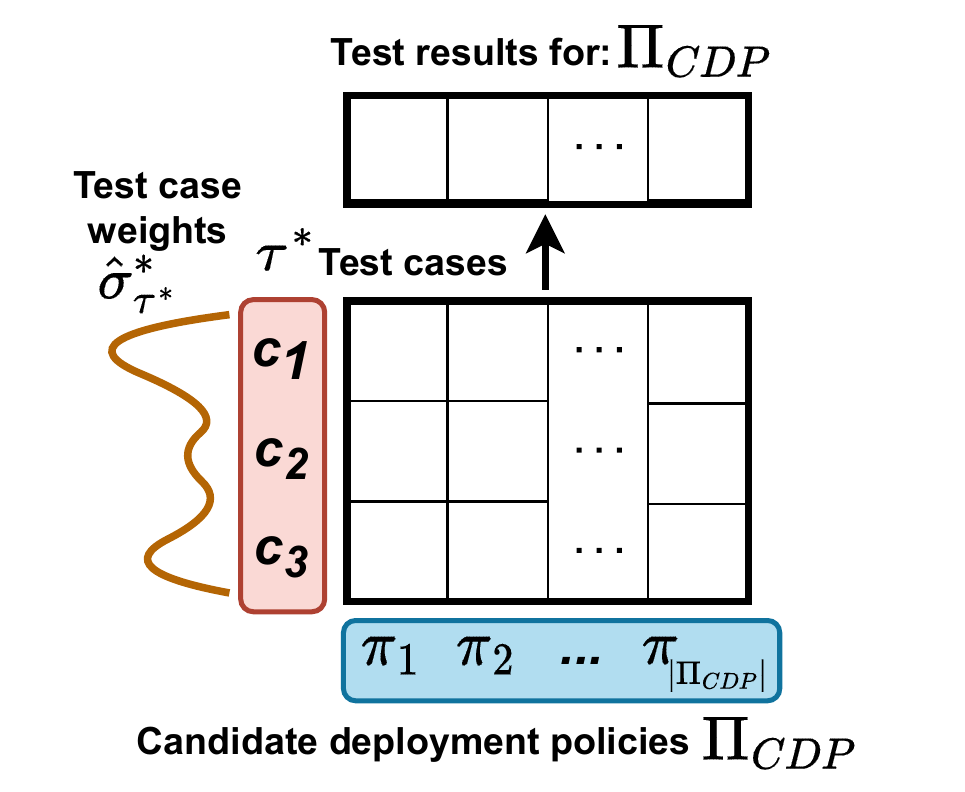}
      \label{fig:algorithm_output}
  \end{subfigure}
    \caption{Policy testing with RPOSST.
        From left to right, the result matrix $A$ is constructed from rollouts, \ie/, $A_{i,j}$ is the average rollout outcome for tuning policy $j$ ($\PolicySet_{\tuningLabel}$ is the set of tuning policies) on test case $i$.
        RPOSST analyzes $A$, taking into account uncertainty distribution $\jointTestCaseDistributionPolicyUncertainty$ and a (possibly empty) initial set of test cases that must be used, $\testCaseGroup^0$. RPOSST outputs an efficient robust test $\tuple{\testCaseGroup^*, \hat{\testCaseDistribution}_{\testCaseGroup^*}^*}$, here using only $\mTestCasesToSelect = 3$ test cases
        (if $\testCaseSet$ is too large to select all $3$ test cases at once, $\testCaseGroup^*$ can be fed back into RPOSST as $\testCaseGroup^0$).
        New candidate deployment policies, $\PolicySet_{\deploymentLabel}$, are tested against each test case in $\testCaseGroup^*$ and each result is weighted according to $\hat{\testCaseDistribution}^*_{\testCaseGroup^*}$, producing a test score for each candidate deployment policy.}
    \label{fig:algorithm_flowchart}
\end{figure*}

Reinforcement learning (RL)~\citep{sutton2018reinforcement} policies have made a number of stunning breakthroughs in multiplayer games~\citep{silver2016mastering,moravvcik2017deepstack,brown2018superhuman,vinyals2019grandmaster,brown2019superhuman,wurman2022outracing,meta2022diplomacy,perolat2022mastering}.
However, the process of choosing an RL policy for production usage, either in an exhibition or deployment for end users, is challenging.
Practitioners often generate many policies that perform well during training but which require thorough vetting on alternative conditions or opponents.
Ideally, we would construct a test case for every conceivable deployment scenario, evaluate each policy on each test case, and rank each policy according to a weighted average of test case results.
However, such a procedure is typically infeasible because of the sheer numbers of policies and deployment scenarios, especially if test cases are lengthy or involve people.
In this work, we present a method for selecting a small number of test cases from a larger pool that minimizes the reduction in test quality.

Practitioners from other fields, \eg/, educational testing~\citep{vanderlinden2005linear}, will recognize this problem as \emph{test construction}--selecting a small yet robust set of test cases, based on limited data, to evaluate many candidates.
This set of test cases should contain enough information to indicate performance over the whole test case pool.
For instance, if a policy can defeat a skilled opponent, we can infer that it can defeat an unskilled opponent.
However, complicated domains contain complex intransitive relationships between policies, necessitating test case diversity.
In addition, there is considerable uncertainty over what policies may be produced in the future and what test cases are the most important to game designers.
This uncertainty needs to be considered because once test cases are chosen, the future policies to assess may be the most difficult ones for the test to evaluate accurately.
Therefore, a robust solution is required.

We introduce a framework, \emph{robust population optimization for a small set of test cases} (\emph{RPOSST}), to compose an efficient robust test of a fixed size.
RPOSST tunes its test to approximate the test scores of adversarially selected policies and test case averaging weights, given test case results on a small set of policies.
We present two RPOSST algorithms representing different use cases, focusing on RPOSST$_{\seqLabel}$, which is better suited to current RL deployment pipelines.
We provide robustness guarantees for RPOSST$_{\seqLabel}$ and CVaR RPOSST$_{\seqLabel}$ (a convenient special case) for $k$-of-$N$ robustness measures~\citep{chen2012tractable}.
These guarantees provide confidence that RPOSST test scores for future deployment candidates are reliable.

Our contributions include the RPOSST framework, including two algorithm versions, robustness guarantees, and empirical validation in domains widely ranging in complexity.
Empirical results are presented for a toy one-shot game simulating race car passing, computer poker competition datasets, and the high fidelity racing simulator, \GranTurismo7/.
They show that RPOSST can dramatically reduce (compared to the full set) the number of test cases needed to identify good deployment policies.
 
\section{Problem Definition}
The goal of \emph{policy testing} is to evaluate the strengths and weaknesses of a large set of \emph{candidate deployment policies}, $\PolicySet_{\deploymentLabel} \subseteq \PolicySet$, in order to choose one for deployment.
A \emph{policy} $\policy \in \PolicySet$ in this setting can be any mapping from environment observations to a distribution over actions (\eg/, Markov policies; ~\citet{sutton2018reinforcement}).
A policy is evaluated on a \emph{test} consisting of various test cases chosen from a pool, $\testCaseSet$.
Each \emph{test case} simulates an important aspect of the deployment environment, for example, different parameter settings like weather conditions or different opponent policies in a competative game.
For straightforward comparisons between policies, we summarize a policy $\policy$'s test results with a scalar \emph{test score}, computed as the weighted average of $\policy$'s test case results according to \emph{test case weights}, $\testCaseDistribution \in \simplex^{\abs{\testCaseSet}}$.

If $\testCaseSet$ is small, then right before deployment we could simply test each policy, rank the policies in $\PolicySet_{\deploymentLabel}$ according to the test scores, and deploy the best one.
However, if policies will encounter a wide range of conditions during deployment, \eg/, hundreds or thousands of different players for a policy deployed to a popular video game, then $\testCaseSet$ ostensibly needs to be large in order to adequately reflect such diversity.
The linear scaling in $\abs{\testCaseSet}$ presents not just a computational burden, but also costs in sample complexity (if the test cases are lengthy) or even in person-time if human quality assurance testers might be needed for test cases.

This work addresses the problem of composing an efficient \emph{test}, $\testTuple$, by selecting a small number of test cases $\testCaseGroup \subset \testCaseSet$ and test case weights $\hat{\testCaseDistribution} \in \simplex^{\abs{\testCaseGroup}}$ to approximate a full test, $\tuple{\testCaseSet, \testCaseDistribution \in \simplex^{\abs{\testCaseSet}}}$.
Complicating this task are two sources of uncertainty to which the efficient test must be robust.
First, $\testTuple$ ought to be used on new candidate deployment policies, so $\PolicySet_{\deploymentLabel}$ is unknown before $\testTuple$ is chosen.
Second, the desired \emph{target distribution}, $\testCaseDistribution$, defining the full test to approximate may drift after $\testTuple$ is chosen.

We assume access to a small set of representative \emph{tuning policies} $\PolicySet_{\tuningLabel} \subset \PolicySet$ for immediate testing (\cref{section:rposst} discusses practical considerations in the composition of $\PolicySet_{\tuningLabel}$).
Additionally, our algorithm takes as input a joint distribution $\jointTestCaseDistributionPolicyUncertainty$ over $\PolicySet_{\tuningLabel}$ and $\simplex^{\abs{\testCaseSet}}$ to represent the combined uncertainty about which policies the output test will be applied to and which target distribution to approximate.
See \cref{fig:algorithm_flowchart} for an illustration of the test composition pipeline.

As a concrete example of the terms above and the need for robustness in the face of uncertainty, consider a car-racing agent developed for a one-on-one racing game. The first source of uncertainty is over the future policies we may want to test. Consider the case where, at test construction time, we have policies from two training runs--one that produces aggressive (collision-prone) policies, and another that produces more polite policies, but we are uncertain about which type will be best suited for the game. In this case, we want the selected test cases to provide good evaluations on policies from either set, and thus require $\jointTestCaseDistributionPolicyUncertainty$ to reflect this uncertainty. Policies from both sets should be included in $\PolicySet_{\tuningLabel}$ and our algorithm needs to be robust to policies within $\PolicySet_{\tuningLabel}$.

The second source of uncertainty is over which test cases are most important. Imagine that we have some test cases that specifically target and penalize off-track infractions. In the future, game designers could request fewer infractions or allow for more risky racing lines. To hedge against both of these possibilities we can add two target distributions to $\jointTestCaseDistributionPolicyUncertainty$, one where off-track tests cases have higher weights than the other test cases and another where they have lower weights. The job of an algorithm (such as RPOSST) is then to ensure its tests are accurate according to both target distributions.
  
\section{Background}
\label{section:background}
In order to compose an efficient and robust test, we utilize established game-theoretic frameworks for modeling robustness and learning optimal decisions (specifically, regret minimization).
The following subsections present background material on these two topics.

\subsection{Robustness}

The idea of \emph{robustness} is to prepare for an unfavorable portion of possible outcomes sampled from an uncertainty distribution.
In our formulation of policy testing the uncertainty distribution covers the future policies in $\PolicySet_{\deploymentLabel}$ and the target distribution.
A \emph{percentile robustness measure}~\citep{charnes1959chanceConstrainedProgramming}, $\probMeasure$, is a formal representation of a robustness criterion as a probability distribution over percentiles.
For example, if $\probMeasure$  has all of its weight on 0.01, then an $\mTestCasesToSelect$-size test with weights $\hat{\testCaseDistribution}_{\testCaseGroup}$ that is robust according to $\probMeasure$, then the test minimizes test score error on $\hat{\testCaseDistribution}_{\testCaseGroup}$'s worst 1\% of policy--target-distribution pairs sampled from $\jointTestCaseDistributionPolicyUncertainty$.

The \emph{$k$-of-$N$ robustness measures}~\citep{chen2012tractable} are percentile robustness measures defined by parameters $k, N \in \naturals$, $1 \le k \le N$, that permit \emph{tractable} optimization procedures.
This parameterization reflects the mechanics of how an efficient test $\testTuple$ is evaluated on such a measure: $N$ policy--target-distribution pairs are sampled from  $\jointTestCaseDistributionPolicyUncertainty$ and $\hat{\testCaseDistribution}_{\testCaseGroup}$'s performance is averaged over the $k$ worst pairs for $\hat{\testCaseDistribution}_{\testCaseGroup}$.
Every $k$-of-$N$ robustness measure is a non-increasing function, \ie/, more weight is placed on smaller percentiles, and the fraction $\nicefrac{k}{N}$ represents the percentile (technically the fractile) around which the measure decreases.

In our test construction setting, the choice of $k$ and $N$ reflects the designer's tolerance for test scores that are bad because of ``unlucky'' outcomes from $\jointTestCaseDistributionPolicyUncertainty$ (that is, test scores with large error on policy--target-distribution pairs sampled from $\jointTestCaseDistributionPolicyUncertainty$, even if they are sampled infrequently).
Optimizing for performance under small percentiles (\eg/, setting $k = 1, \;N = 100$) yields tests with a small maximum test score error across $\PolicySet_{\tuningLabel}$.
Then, even if each candidate deployment policy resembles the tuning policy that has the largest test score error, the optimized test will yield small test score errors.
In contrast, optimizing for the uniform measure ($k = N$) optimizes for mean performance across $\PolicySet_{\tuningLabel}$, essentially assuming $\PolicySet_{\deploymentLabel} = \PolicySet_{\tuningLabel}$, which can lead to large test score error on the actual candidate deployment policies.

As $N \to \infty$, the $k$-of-$N$ robustness measure approaches the \emph{conditional value at risk} (\emph{CVaR}) robustness measure at the $\nicefrac{k}{N}$ fractile~\citep{chen2012tractable}, which evenly weights all of the fractiles $\le \nicefrac{k}{N}$ and puts a weight of zero on all larger fractiles.
Formally, the robustness optimization objective is to minimize the \emph{percentile performance loss}:
\begin{align}
    L_{\probMeasure, \jointTestCaseDistributionPolicyUncertainty}(\hat{\testCaseDistribution}_{\testCaseGroup})
        =
            \inf_{\integrableFn \in \IntegrableFnSet}
                \hspace{-3.5em}
                \underset{
                    \percentile \in [0, 1], \,
                    \Prob \subblock*{
                        \ell(\hat{\testCaseDistribution}_{\testCaseGroup}; \policy, \testCaseDistribution) \le \integrableFn(\eta)
                    } \ge \percentile
                }{\int}
                    \hspace{-3.5em}
                    \integrableFn(\eta)
                    \probMeasure(d\eta),
    \label{eq:percentile-performance-loss}
\end{align}
under a loss function
$\ell : \simplex^{\abs{\testCaseSet}} \times \PolicySet_{\tuningLabel} \times \simplex^{\abs{\testCaseSet}} \to \reals$
where we overload $\ell$ for incomplete test case weight vectors by filling in zeros for missing elements,
$\tuple{\policy, \testCaseDistribution} \sim \jointTestCaseDistributionPolicyUncertainty$,
and $\IntegrableFnSet$ is the class of real-valued, bounded, $\probMeasure$-integrable functions on $[0, 1]$.
An \emph{efficient} ($\mTestCasesToSelect$-size) $\probMeasure$-\emph{robust test} is a minimizer of
$L_{\probMeasure, \jointTestCaseDistributionPolicyUncertainty}$
across all $\hat{\testCaseDistribution}_{\testCaseGroup}$ where $\testCaseGroup = \mTestCasesToSelect$.

The optimization of the percentile performance loss under $k$-of-$N$ robustness measure, $\kOfNMeasure$, can be  modeled as a zero-sum imperfect information game~\citep{chen2012tractable}.
Here, a protagonist player constructs efficient tests and an antagonist chooses a tuning policy to test and a target distribution.
For their payoffs, the antagonist receives the test score error of the protagonist's test given the antagonist's tuning policy and target distribution while the protagonist receives the negation.
The $k$ and $N$ parameters determine which target distributions and tuning policies that the antagonist can choose from and how many pairs must be averaged across.
At the start of the game, $N$ target-distribution--tuning-policy pairs are sampled.
From these $N$ pairs, the antagonist must select $k$ of them.
Finally, one of these $k$ pairs is sampled, both players receive their payoffs, and the game ends.
A \emph{minimax} test for the protagonist, \ie/, one that minimizes the protagonist's maximum loss in this game is a $\kOfNMeasure$-robust test.

\subsection{Regret}
While the game above models the optimization process, it does not instruct the protagonist on \emph{how} to choose test cases to win. A no-regret \emph{online decision process} (\emph{ODP}) algorithm can find approximate minimax decisions by repeatedly playing out the game and improving over time from payoff feedback.
Formally, on each round $t$ of the game, an ODP algorithm chooses an efficient test
$\tuple{\testCaseGroup^t, \hat{\testCaseDistribution}_{\testCaseGroup^t}^t}$
and receives the \emph{payoff function}
$\cfv^t = -\grad_{\hat{\testCaseDistribution}^t_{\testCaseGroup^t}} \ell(\hat{\testCaseDistribution}^t_{\testCaseGroup^t}; \policy^t, \testCaseDistribution^t)$
as feedback given $\tuple{\policy^t, \testCaseDistribution^t}$ chosen by the antagonist.
If the antagonist always plays a best response to the ODP algorithm, that is, the tuning-policy--target-distribution pair that maximizes the loss of
$\hat{\testCaseDistribution}^t_{\testCaseGroup^t}$
on each round $t \in \set{1, \ldots, T}$, $T \ge 1$, then the \emph{no-regret} property ensures that at least one of the tests in the sequence $\tuple{
    \tuple{\testCaseGroup^t, \hat{\testCaseDistribution}_{\testCaseGroup^t}^t}
}_{t = 1}^T$ is at most $\bigO(\nicefrac{\maxGrad}{\sqrt{T}})$ away from the minimax value, where $\maxGrad > 0$ is the maximum magnitude of the loss gradient (see \citet{ED,exploitabilityDescentArxiv} and Appendix \cref{prop:best-optimality} for more details).

\emph{Regret matching$^+$}~\citep{cfrPlus,solvingHulhe} is a no-regret algorithm for simplex decision sets, \eg/, the $\mTestCasesToSelect$ dimensional test case weight space $\simplex^{\mTestCasesToSelect}$,
that selects
$\hat{\testCaseDistribution}^t_{\testCaseGroup^t}
   = \frac{q^{1:t - 1}}{\ones^{\top} q^{1:t - 1}}$
using \emph{pseudoregrets}
$q^{1:t} = \ramp{q^{1:t - 1} + \regret^t}$, $q^{1:0} = \zeros$,
where $\regret^t = \cfv^t - \subex{\cfv^t}^{\top} \hat{\testCaseDistribution}^t_{\testCaseGroup^t}$ is the \emph{instantaneous regret} vector ($\hat{\testCaseDistribution}^t_{\testCaseGroup^t} = \frac{1}{d}\ones$ if none of the pseudoregrets are positive).

\section{RPOSST}\label{section:rposst}
Our approach, \emph{robust population optimization for a small set of test cases} (\emph{RPOSST})
begins by evaluating each tuning policy $\policy \in \PolicySet_{\tuningLabel}$ on each test case $\testCase \in \testCaseSet$, yielding a
$\abs{\testCaseSet} \times \abs{\PolicySet_{\tuningLabel}}$
result matrix $A$ of test case results.
As an optimization approach, RPOSST aims to minimize prediction errors, as measured by a convex function
$\differenceFn: \reals \times \reals \to \reals$,
\eg/, the absolute difference
$\differenceFn(\hat{x}, x) = \abs{\hat{x} - x}$.
RPOSST robustly optimizes for a small set of test cases and a weighting over them according to how well it reproduces test scores admitted by $A$ as measured by a loss function
\[
    \ell:
        \hat{\testCaseDistribution}; \policy_j, \testCaseDistribution
        \mapsto
            \differenceFn(
                \underbrace{
                    E_{i \sim \hat{\testCaseDistribution}}\subblock*{
                        A_{i, j}
                    }
                }_{\hat{\testCaseDistribution}\text{'s test score for } \policy_j.},
                \underbrace{
                    E_{i \sim \testCaseDistribution}\subblock*{
                        A_{i, j}
                    }
                }_{\testCaseDistribution\text{'s test score for } \policy_j.}
            ),
\]
on test case distribution
$\hat{\testCaseDistribution} \in \simplex^{\abs{\testCaseSet}}$
compared to
$\testCaseDistribution \in \simplex^{\abs{\testCaseSet}}$
with respect to test results from the $j$\textsuperscript{th} tuning policy $\policy_j$.
Since $\hat{\testCaseDistribution}$ is being used to produce test scores that approximate those under $\testCaseDistribution$, we call $\testCaseDistribution$ a \emph{target distribution} in this context.
Our goal is to select a small number of test cases, so we constrain RPOSST to output weights
$\hat{\testCaseDistribution}_{\testCaseGroup} \in \simplex^{\mTestCasesToSelect}$
for groups of test cases $\testCaseGroup \subset \testCaseSet$ of size $\mTestCasesToSelect$.

Though $\testCaseSet$ is large, the cost of computing $A$ is balanced by the savings of using fewer test cases for future policies.
RPOSST is robust to any distribution over $\PolicySet_{\tuningLabel}$, so as long as this set covers the space of $\PolicySet_{\deploymentLabel}$ (\ie/, all $\pi \in \PolicySet_{\deploymentLabel}$ are convex mixtures of $\PolicySet_{\tuningLabel}$), this robustness imparts a minimum test accuracy guarantee even on deployment candidates.
Intuitively, this means the quality of RPOSST's tests will tend to improve with more diverse tuning policies.
Accordingly, it should be beneficial for a tuning policy to represent an extreme point in a reasonable region of policy space, or at least for it to be generated with a method similar to that which will generate deployment candidates (\eg/, sampled from checkpoints of RL training runs).
That way, the tuning policies include a diverse collection of skilled and unskilled policies with random variations, while retaining architectural and algorithmic similarities to future deployment candidates.

Following the earlier discussion of $k$-of-$N$ robustness, we frame the optimization in RPOSST as a zero-sum game.
By adversarially choosing policies to test, the antagonist forces RPOSST to compose tests that are better at accurately testing the more difficult-to-assess policies in the tuning set, providing a degree of robustness to the distribution of future deployment candidates.
Similarly, by adversarially choosing the target distribution, the antagonist also forces RPOSST to be robust along this dimension.
The steps of each round $t = 1, \ldots, T$ of our optimization game follows.
\begin{enumerate}
    \item The protagonist must choose an $\mTestCasesToSelect$-tuple of test cases $\testCaseGroup^t \subset \testCaseSet$ and weights
        $\hat{\testCaseDistribution}^t_{\testCaseGroup^t} \in \simplex^{\mTestCasesToSelect}$.
    \item $N$ policies to test and target distributions,
        $\tuple*{ \tuple*{ \policy_{j_i}, \testCaseDistribution_i} }_{i = 1}^N$,
        are sampled from uncertainty distribution
        $\jointTestCaseDistributionPolicyUncertainty$.
    \item The antagonist chooses the $k$ worst policies and target distributions, \ie/, those that maximize
        $\ell\subex*{
            \hat{\testCaseDistribution}^t_{\testCaseGroup^t};
            \policy_{j_i},
            \testCaseDistribution_i
        }.$
    \item One of the $k$ worst configurations is sampled uniformly, leading to the end of the round, at which point the protagonist receives the payoff
        $\cfv^t_{\testCaseGroup^t, (i)} = -\ell\subex*{
            \hat{\testCaseDistribution}^t_{\testCaseGroup^t};
            \policy_{j_{(i)}},
            \testCaseDistribution_{(i)}
        },$
        where the subscript $(i)$ denotes the $i$\textsuperscript{th} element of a sorted list in descending order (the $i$\textsuperscript{th} worst for the protagonist).
\end{enumerate}
The protagonist is allowed to update their strategy at the end of each round based on the expected payoff, $\E_{i \sim \Unif\subex*{\set{1, \ldots, k}}}\subblock*{ \cfv^t_{\testCaseGroup, (i)} }$, for each $\testCaseGroup \in \testCaseSet^{\mTestCasesToSelect}$ they could have chosen.
The more rounds of the game that are run (the larger $T$ is), the closer RPOSST gets to returning a minimax strategy, and consequently, a robust optimal selection of test cases and weights.
Thus, in application, $T$ can be set as large as is convenient under computational and time constraints.
\Cref{thm:rposst-seq} gives a precise rate for RPOSST's improvement, with high probability, as a function of $T$.
Although the protagonist must consider an exponential (in $\mTestCasesToSelect$) number of test case combinations, the premise of RPOSST is that we want a small set of test cases, so $\mTestCasesToSelect$ will be small.
To decrease computational requirements, RPOSST can be run in a loop to select test cases iteratively until $\mTestCasesToSelect$ have been selected, at a potential cost to test accuracy compared to optimizing for the entire $\mTestCasesToSelect$-tuple at once.

\subsection{Antagonist Information Models}

We consider two RPOSST algorithm variants that utilize different models of the information that the antagonist in our optimization game has before they make their choice.
These models correspond to two policy testing use cases.
The first, ``simultaneous move'' model is less pessimistic, but has impractical aspects, which are addressed by the subsequent ``sequential move'' model.

\textbf{Simultaneous move.}
The simultaneous move model is a na\"ive application of the original $k$-of-$N$ game by \citep{chen2012tractable}.
In this model, the antagonist does not observe which $\mTestCasesToSelect$-tuple of test cases, $\testCaseGroup^t$, is selected by the protagonist on each round $t$.
Instead, it is randomized with a distribution $\hat{\testCaseDistribution}^t_{\testCaseGroupDecisionLabel} \in \simplex^{\abs{\testCaseSet}^{\mTestCasesToSelect}}$.
This model corresponds to the policy testing use case where a new $\mTestCasesToSelect$-tuple of test cases is sampled independently for each test that is performed.
Every test only evaluates $\mTestCasesToSelect$ cases, as desired from a computational efficiency perspective, however, the particular test cases used in each test could be different, making results incomparable across tests.
See Appendix \cref{sec:sim-move-model} for additional details.

\begin{algorithm2e}[tb]
  \caption{RPOSST$_{\seqLabel}$ with regret matching$^+$ and Successive Rejects}\label{alg:rposst-seq}
  \DontPrintSemicolon
  \textbf{\textit{Inputs:}} $\langle k, N, T_1, \mTestCasesToSelect, \jointTestCaseDistributionPolicyUncertainty, \testCaseGroup^0, \ell, T_2 \rangle$
\vspace{0.2em} \hrule \vspace{0.2em}

  $q^{1:0}_{\testCaseGroup} \gets \zeros \in \reals^{\mTestCasesToSelect + \abs{\testCaseGroup^0}}$ \textbf{for} $\testCaseGroup \in \testCaseSet^{\mTestCasesToSelect}$

  $T' \sim \Unif(\set{1, \ldots, T_1})$

  \For{$t \gets 1, \ldots, T'$}{
      \For{$\testCaseGroup \in \testCaseSet^{\mTestCasesToSelect}$}{
          $z^t \gets \ones^{\top} q^{1:t - 1}_{\testCaseGroup}$

          $\hat{\testCaseDistribution}^t_{\testCaseGroup} \gets q^{1:t - 1}_{\testCaseGroup} / z^t$ \textbf{if} $z^t > 0$ \textbf{else} $\ones / \mTestCasesToSelect$

          \tcp{Add zeros to ensure $\hat{\testCaseDistribution}^t_{\testCaseGroup} \in \simplex^{\abs{\testCaseSet}}$.}
          $\hat{\testCaseDistribution}^t_{\testCaseGroup}(x) \gets 0$ \textbf{for} $x \in \testCaseSet \setminus (\testCaseGroup \cup \testCaseGroup^0)$

          $\subblock{ \ell_{\testCaseGroup, (i)} }_{i = 1}^k \gets
            \WorstKOfNLossesFn\subex*{\hat{\testCaseDistribution}^t_{\testCaseGroup}, \tuple{k, N}, \jointTestCaseDistributionPolicyUncertainty, \ell}$

          $\cfv^t_{\testCaseGroup} \gets
            \dfrac{-1}{k} \sum_{i = 1}^k \dfrac{\partial \ell_{\testCaseGroup, (i)}, \policy_{j_{(i)}}}{\partial \hat{\testCaseDistribution}^t_{\testCaseGroup}}$

          \tcp{Update regret matching$^+$.}
          $\regret^t_{\testCaseGroup} \gets \cfv^t_{\testCaseGroup} - (\hat{\testCaseDistribution}^t_{\testCaseGroup})^{\top} \cfv^t_{\testCaseGroup}$

          $q^{1:t}_{\testCaseGroup} \gets \ramp{q^{1:t - 1}_{\testCaseGroup} + \regret^t_{\testCaseGroup}}$
      }
  }
  $\testCaseGroup^* \gets
    \SuccessiveRejects\subex*{
      \testCaseGroup
        \mapsto
          \dfrac{1}{2 k \maxLoss} \ones^{\top}
            \WorstKOfNLossesFn\subex*{
              \hat{\testCaseDistribution}^{T'}_{\testCaseGroup},
              \tuple{k, N},
              \jointTestCaseDistributionPolicyUncertainty,
              \ell
            },
      T_2
    }$

  \Return $\testCaseGroup^*, \hat{\testCaseDistribution}^{T'}_{\testCaseGroup^*}$

  \vspace{0.5em} \hrule \vspace{0.2em}
  \setcounter{AlgoLine}{0}
  \SetKwProg{Subroutine}{Procedure}{}{}
  \Subroutine{$\WorstKOfNLossesFn$ \quad \textbf{Inputs:} $\langle \hat{\testCaseDistribution}, k, N, \jointTestCaseDistributionPolicyUncertainty, \ell \rangle$}{

\vspace{0.2em} \hrule \vspace{0.2em}

    \For{$i = 1 \ldots N$}{
      \tcp{Sample antagonist actions.}
      $\policy_{j_i}, \testCaseDistribution_i \sim \jointTestCaseDistributionPolicyUncertainty$

      \tcp{Evaluate $\hat{\testCaseDistribution}$.}
      $\ell_i \gets \ell(\hat{\testCaseDistribution}; \policy_{j_i},\testCaseDistribution_i)$
    }

    \tcp{Sort to identify the worst $k$.}
    $\SortFn\subex*{
      \subblock*{ \ell_i }_{i = 1}^N
    }$

    \Return $\subblock{ \ell_{(i)} }_{i = 1}^k$
  }
\end{algorithm2e}
 
\textbf{Sequential move.}
In the sequential move model, the antagonist observes $\testCaseGroup^t$ before acting.
The antagonist is thus able to tailor their choice of
$\tuple*{\tuple*{
    \policy_{j_{(i)}},
    \testCaseDistribution_{(i)}
}}_{i = 1}^k$
to whichever $\testCaseGroup^t$ is selected, and randomizing over the $\mTestCasesToSelect$-tuple of test cases has no benefit to the protagonist.
Since the antagonist observes $\testCaseGroup^t$, the protagonist must update all the weights that they would apply to each test case tuple $\testCaseGroup$ as if $\testCaseGroup^t = \testCaseGroup$.
Thus, the selection of $\testCaseGroup^t$ does not impact the protagonist's updates and we need not explicitly select an $\mTestCasesToSelect$-tuple until the very end of the algorithm, after $T' \sim \Unif\subex{\set{1, \ldots, T_1}}$ rounds.\footnote{RPOSST is run for $T'$ rather than $T_1$ rounds because we cannot guarantee a decrease in worst-case loss after every round. See the proof of \cref{thm:rposst-seq} for more details.}

Since the set of $N$ losses observed on each round are generally random, we cannot reuse them to identify which $\mTestCasesToSelect$-tuple leads to the lowest loss using the the test case weights computed after running for $T'$ rounds, $\tuple{\hat{\testCaseDistribution}^{T'}_{\testCaseGroup}}_{\testCaseGroup \in \testCaseSet^{\mTestCasesToSelect}}$.
In addition, we cannot access expected $k$-of-$N$ losses directly; we must estimate them by sampling from $\jointTestCaseDistributionPolicyUncertainty$.
Therefore, the selection of a single $\testCaseGroup$ is a ``best arm identification'' problem, where $\testCaseSet^{\mTestCasesToSelect}$ is the set of arms.
The Successive Rejects (SR)~\citep{audibert2010BestArmIdentification} algorithm is an exploration-only bandit algorithm that can be used to solve this problem with a worst-case guarantee on the probability that it identifies the best arm.
The more SR iterations that are run, the more likely it is to select the best arm.
\Cref{alg:rposst-seq} shows how to implement RPOSST for the sequential move model using regret matching$^+$ for tuning the test case weights and SR for the final selection of an $\mTestCasesToSelect$-tuple.

In specific applications, an example of which we will see in \cref{sec:deterministic-rposst} and our experiments, we can construct our optimization game so that it is deterministic, and consequently, we can replace SR with a simple argmax.

The RPOSST$_{\seqLabel}$ objective is the percentile performance loss
\begin{align}
    \min_{
        \substack{
            \testCaseGroup \in \testCaseSet^{\mTestCasesToSelect}\\
      \hat{\testCaseDistribution}_{\testCaseGroup} \in \simplex^{\mTestCasesToSelect}
        }
    }
    \inf_{\integrableFn \in \IntegrableFnSet}
        \hspace{-1em}
        \underset{
            \substack{
                \percentile \in [0, 1]\\
                \Prob \subblock*{
                    \ell(\hat{\testCaseDistribution}_{\testCaseGroup}; \policy_j, \testCaseDistribution) \le \integrableFn(\eta)
                } \ge \percentile
            }
        }{\int}
            \hspace{-3em}
            \integrableFn(\eta)
            \probMeasure_{\kOfN}(d\eta),
    \label{eq:rposst-seq-objective}
\end{align}
where
$\tuple{\policy_j, \testCaseDistribution} \sim \jointTestCaseDistributionPolicyUncertainty$.

The sequential move model represents the policy testing use case where we select and fix $\mTestCasesToSelect$ test cases and test case weights for all future test policies.
Test scores are easily reproducible and comparable across test applications since the test cases never change.

\begin{theorem}
    \label{thm:rposst-seq}
    After $T' \sim \Unif(\set{1, \ldots, T_1})$, $T_1 > 0$, rounds of its optimization game, \cref{alg:rposst-seq} selects an $\mTestCasesToSelect$-tuple of test cases, $\testCaseGroup^*$ and weights
$\hat{\testCaseDistribution}^{T'}_{\testCaseGroup^*} \in \simplex^{\mTestCasesToSelect}$
that, with probability $(1 - p)(1 - q)(1 - \alpha)$, $p, q, \alpha > 0$, are $\frac{\gap}{q}$-optimal for \cref{eq:rposst-seq-objective}, where
$\gap = \bigO\subex*{ \sqrt{\frac{1}{T_1} \mTestCasesToSelect} + \sqrt{\frac{1}{T_1} \log\subex*{ \nicefrac{1}{p} }} }$
and
$\alpha = \bigO\subex*{\e^{-T_2}}$. \end{theorem}
\textbf{All proofs deferred to the Appendix.}
In the extreme case where $\PolicySet_{\tuningLabel}$ covers $\PolicySet$, then this optimality result, (in terms of an upper bounded percentile loss integral), extends to all deployment candidates $\PolicySet_{\deploymentLabel}$.

\subsection{Deterministic CVaR RPOSST}\label{sec:deterministic-rposst}

While in general, an RPOSST algorithm has a randomized procedure and a non-deterministic optimality guarantee, we can actually select hyperparameters so that RPOSST is deterministic, making the procedure simpler and more reliable.
If we fix the ratio $\nicefrac{k}{N}$ and allow $N \to \infty$, the $k$-of-$N$ robustness measure converges toward the CVaR measure at the $\nicefrac{k}{N}$ fractile.
A $k$-of-$N$ algorithm where $N \to \infty$ cannot be implemented with the usual sampling procedure, but it can be implemented if the distribution characterizing our uncertainty, $\jointTestCaseDistributionPolicyUncertainty$, has finite support.

Sampling $\jointTestCaseDistributionPolicyUncertainty$ infinitely would result in sampling all tuning-policy--target-distribution pairs in its support exactly in proportion to their probabilities.
Rather than selecting $k$ tuning-policy--target-distribution pairs, the antagonist must select pairs until their cumulative probability sums to $\nicefrac{k}{N}$.
Effectively, the antagonist assigns weights
\[\alpha_{(i)} = \min \set*{
    \jointTestCaseDistributionPolicyUncertainty\subex*{
    \tuple*{\policy_{j_{(i)}}, \testCaseDistribution_{(i)}}
    },
    \nicefrac{k}{N} - \sum_{h = 1}^{i - 1} \alpha_{(h)}
}\]
to each tuning-policy--target-distribution pair in $\jointTestCaseDistributionPolicyUncertainty$'s support, where the ordering between pairs is determined by the loss each induces for the protagonist.
Finally, these tuning-policy--target-distribution pairs are sampled according to the normalized weights $\frac{\alpha_{(i)} N}{k}$.

The robustness guarantees become deterministic because the entire RPOSST algorithm, denoted as CVaR($\percentile$) RPOSST for the $\percentile = \nicefrac{k}{N}$ fractile, can be run using exact expectations (excluding randomness in $A$, which is taken as given in RPOSST).
Determinism in RPOSST$_{\seqLabel}$ allows us to directly check the exact expected loss of each test case distribution on each round, letting us track the lowest loss test case distribution across all rounds.
This tracking, in turn, allows us to avoid both sampling $T'$ and running the $\SuccessiveRejects$ algorithm to do the final selection.
Instead, we can simply return the lowest loss test case distribution across all $T$ rounds.

If there are $d$ tuning-policy--target-distribution pairs in $\jointTestCaseDistributionPolicyUncertainty$'s support, then the expected CVaR($\percentile$) loss of the protagonist on round $t$ is
$L^t = \min_{\testCaseGroup \in \testCaseSet^{\mTestCasesToSelect}}
    \sum_{i = 1}^d
        \frac{\alpha_{(i)}}{\percentile}
            \ell(\hat{\testCaseDistribution}^t_{\testCaseGroup}; \policy_{j_{(i)}, \testCaseDistribution_{(i)}})$.
The round with the lowest expected loss is
$t^* = \argmin_{t \in \set{1, \ldots, T}} L^t$,
and this definition allows us to state the following corollary.
\begin{corollary}
    \label{cor:deterministic-rposst-seq}
    Assume that
$\jointTestCaseDistributionPolicyUncertainty \in \simplex^d$
for some finite $d \ge 1$.
After $T$ rounds of the CVaR($\percentile$) RPOSST$_{\seqLabel}$ optimization game, where the protagonist chooses $\mTestCasesToSelect$-size tests according to regret matching$^+$ against a best response antagonist,
$\testCaseGroup^*$ and $\testCaseDistribution^{t^*}_{\testCaseGroup^*}$
are $\gap$-optimal for \cref{eq:rposst-seq-objective} under the $\percentile$-fractile CVaR robustness measure, where
$\gap = \bigO\subex*{ \sqrt{\frac{1}{T} \mTestCasesToSelect}}$. \end{corollary}
Pseudocode for CVaR($\percentile$) RPOSST$_{\seqLabel}$ is presented in Appendix \cref{alg:cvar-rposst-seq}.

\def\targetDistributionLabelText{TTD}
\def\tuningLabelText{TNP}

In addition, we can construct a series of ablations of CVaR RPOSST$_{\seqLabel}$ to act as baselines for experiments, and to make a connection to the test-construction literature.

CVaR RPOSST$_{\seqLabel}$ generalizes an intuitive algorithm: find the $\mTestCasesToSelect$-tuple of test cases that minimizes the maximum error assuming a uniform distribution over the tuple. This \emph{minimax uniform} algorithm is implemented by executing only the initialization and selection steps of CVaR($0$) RPOSST$_{\seqLabel}$ ($T = 0$).
Further simplifying,
\emph{minimax(\targetDistributionLabelText) uniform} performs the antagonist maximization only over target distributions and assumes a uniform distribution over tuning policies.
\emph{Minimax(\tuningLabelText) uniform} performs the antagonist maximization only over tuning policies and assumes a uniform target distribution.
\emph{Miniaverage uniform} assumes both a uniform distribution over tuning policies and for the target distribution.

Additionally, we could select test cases one at a time to minimize the maximum error, echoing greedy algorithms from the test-construction literature (Chapter 4 of ~\citet{vanderlinden2005linear}).
This \emph{iterative minimax} algorithm is almost the same as running the initialization and return steps of CVaR($0$) RPOSST$_{\seqLabel}$ to select a single test case in a loop.
The sole difference being that iterative minimax could select the same test case multiple times within its loop to adjust the test case weighting away from uniform.

\section{Experiments}\label{section:experiments}

\begin{figure*}[t]
    \begin{minipage}[t]{0.325\linewidth}
        \includegraphics[width=\linewidth]{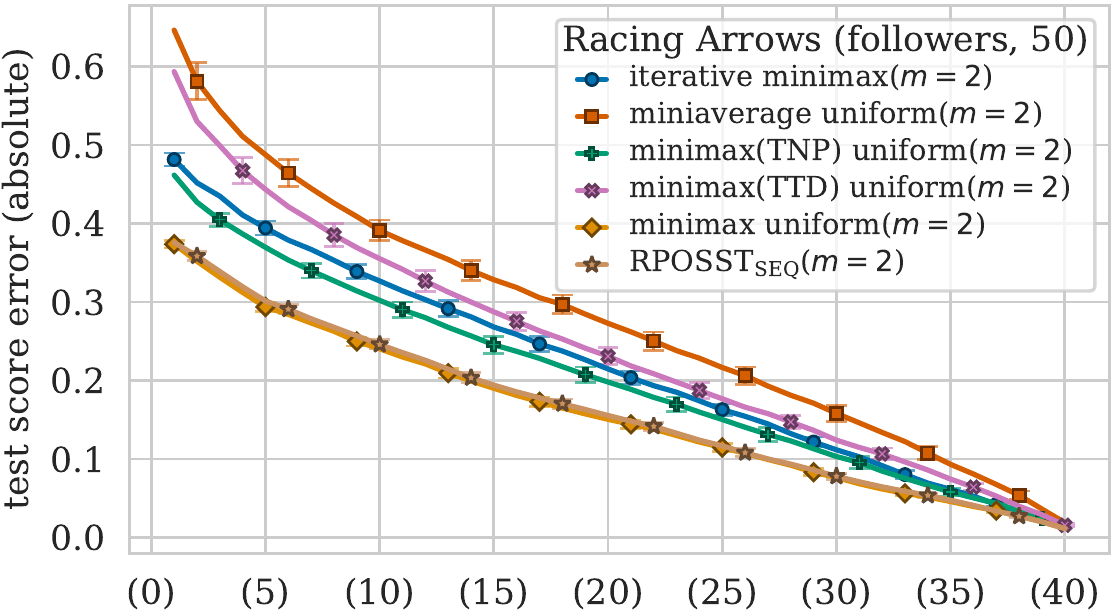}
    \end{minipage}\hfill \begin{minipage}[t]{0.325\linewidth}
        \includegraphics[width=\linewidth]{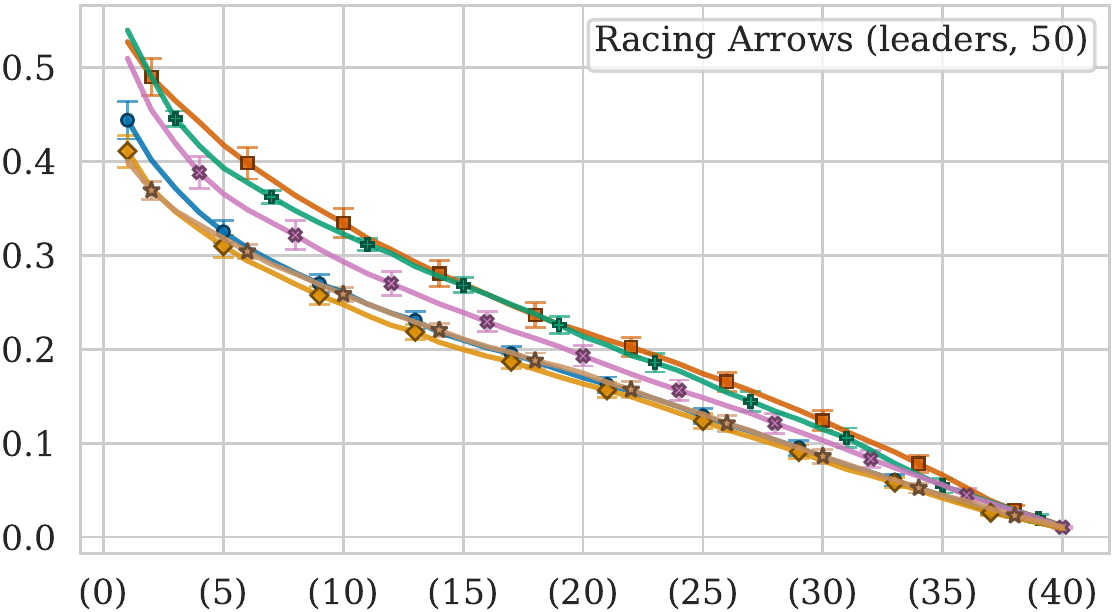}
    \end{minipage}\hfill \begin{minipage}[t]{0.325\linewidth}
        \includegraphics[width=\linewidth]{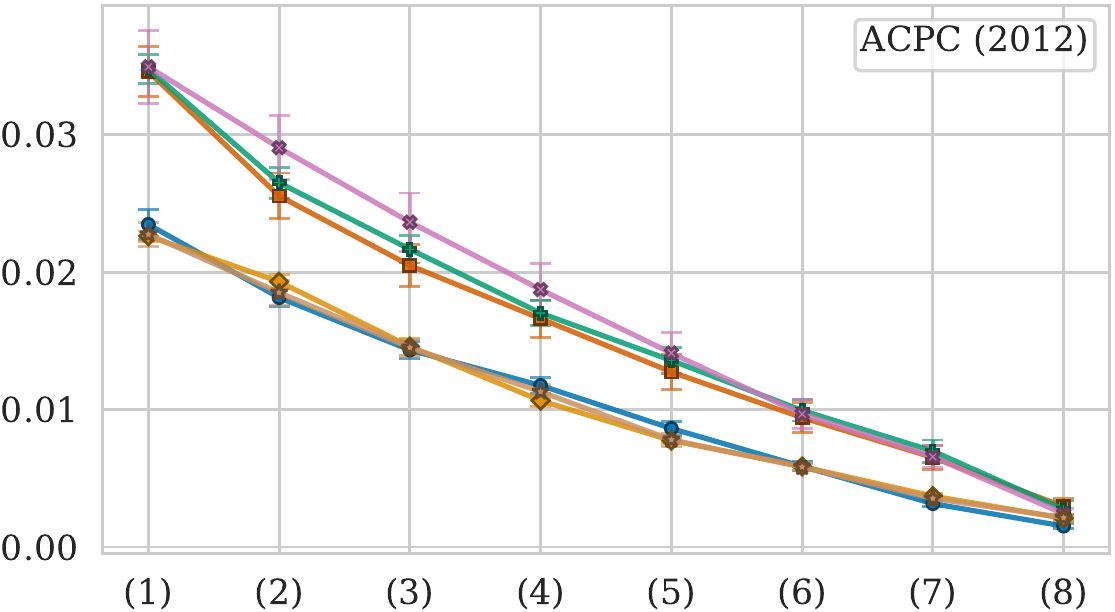}
    \end{minipage}\vspace{0.5em}
    \begin{minipage}[t]{0.325\linewidth}
        \includegraphics[width=\linewidth]{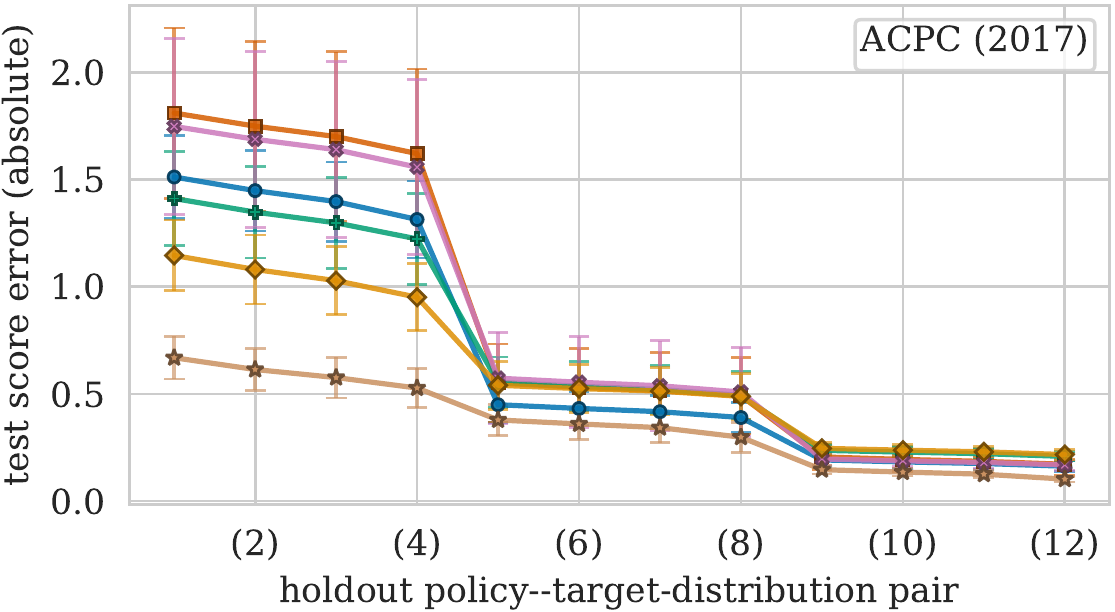}
    \end{minipage}\hfill \begin{minipage}[t]{0.325\linewidth}
        \includegraphics[width=\linewidth]{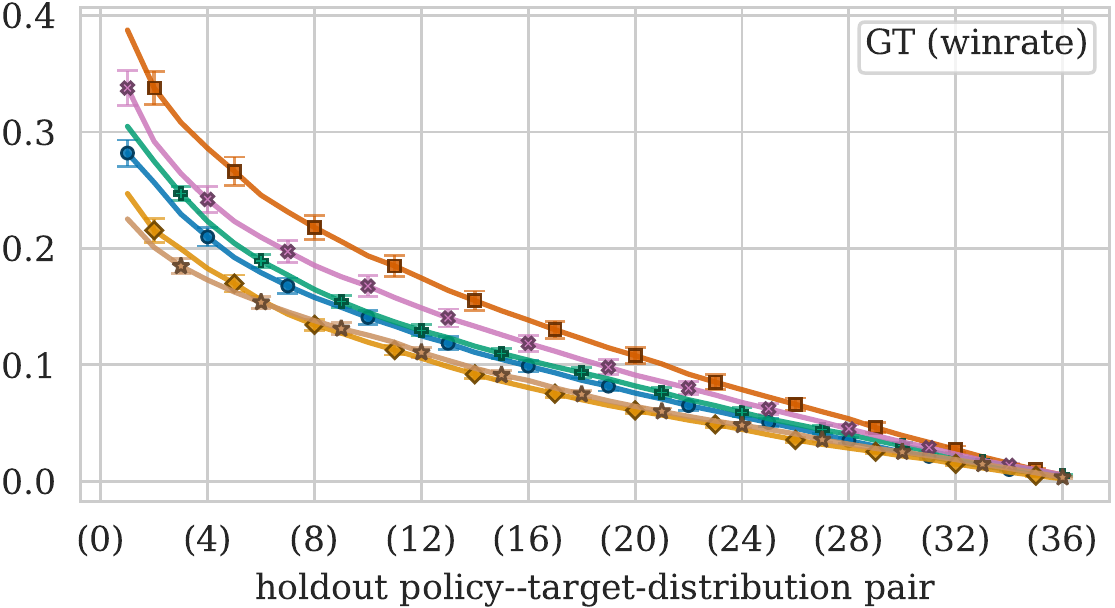}
    \end{minipage}\hfill \begin{minipage}[t]{0.325\linewidth}
        \includegraphics[width=\linewidth]{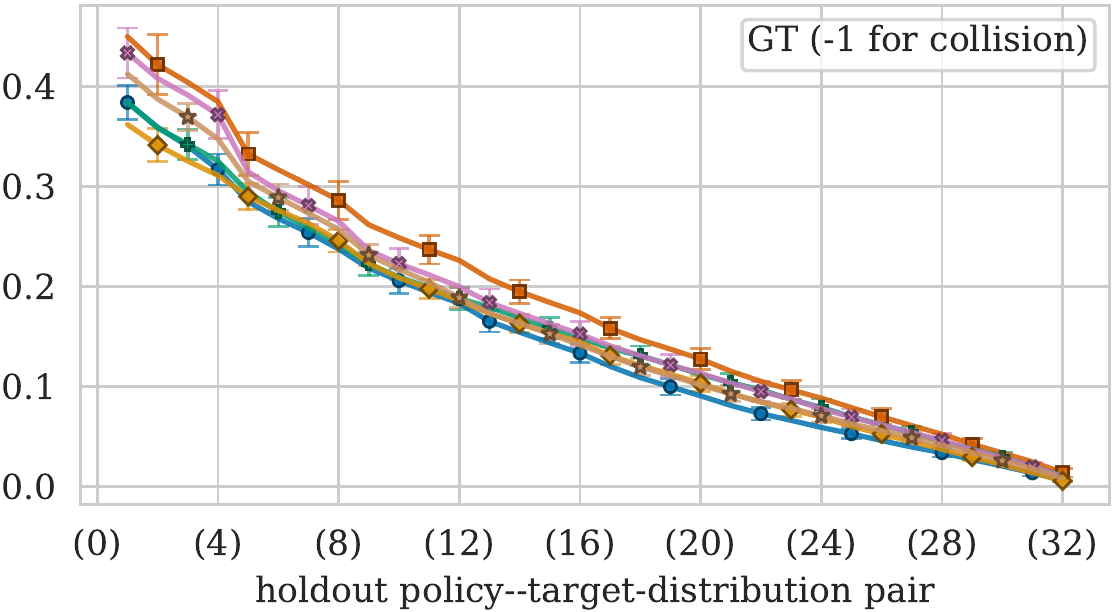}
    \end{minipage}\caption{Expected test score error (absolute difference) across holdout-policy--target-distribution pairs on (top left and middle) Racing Arrows, (top right) the 2012 two-player, limit competition of the ACPC, (bottom left) the 2017 two-player, no-limit competition of the ACPC, (bottom middle and right) \GranTurismo7/ races, between CVaR($\cvarPercentile$) RPOSST$_{\seqLabel}$ and baseline tests on $\numHoldoutReplicas$ randomly sampled sets of holdout policies ($20\%$ of the full set of policies; $80\%$ used as tuning policies). Holdout-policy--target-distribution pairs are sorted according to test score error. Each RPOSST$_{\seqLabel}$ instance was run for $500$ rounds ($T = 500$). Errorbars represent $95\%$ t-distribution confidence intervals.}
    \label{fig:test-score-error}
\end{figure*}
\begin{figure*}[t]
    \begin{subfigure}[t]{0.32\linewidth}
        \includegraphics[width=\linewidth]{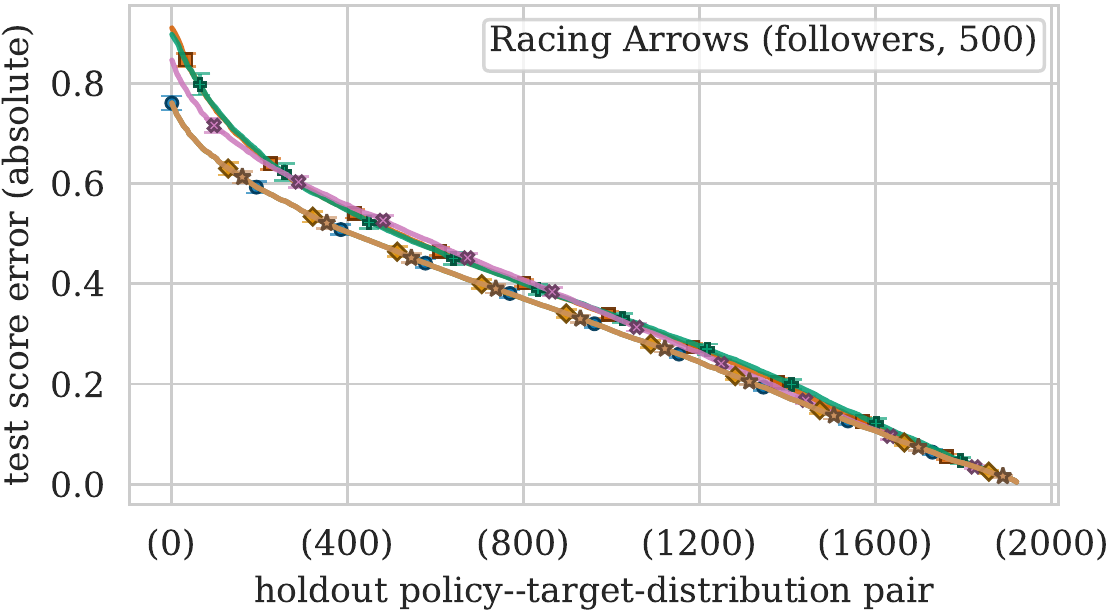}
        \caption{$\mTestCasesToSelect = 1$}
    \end{subfigure}\hfill \begin{subfigure}[t]{0.32\linewidth}
        \includegraphics[width=\linewidth]{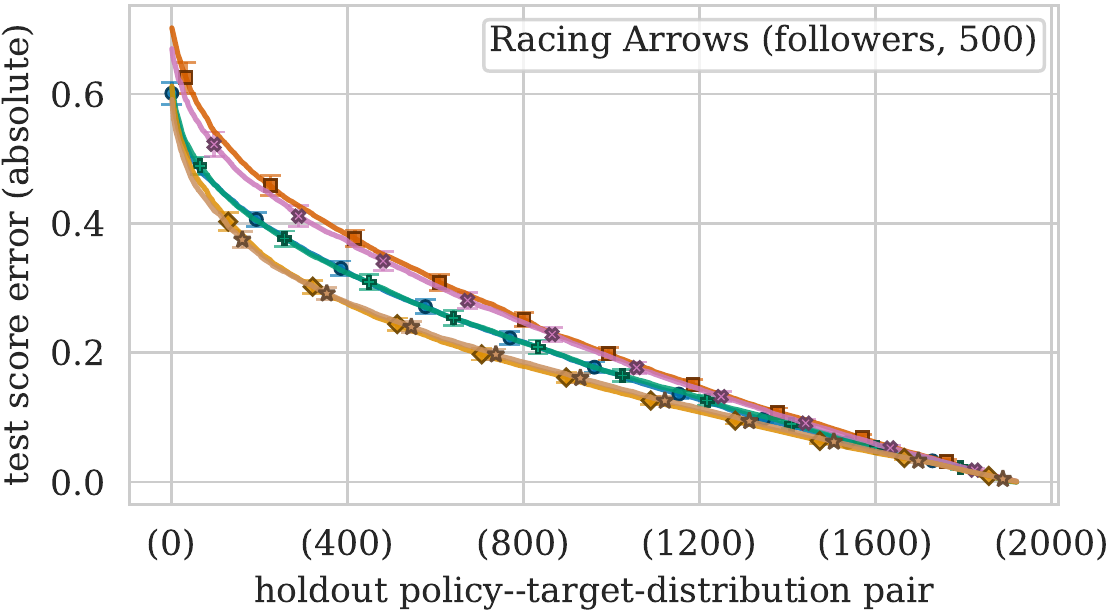}
        \caption{$\mTestCasesToSelect = 2$}
    \end{subfigure}\hfill \begin{subfigure}[t]{0.32\linewidth}
        \includegraphics[width=\linewidth]{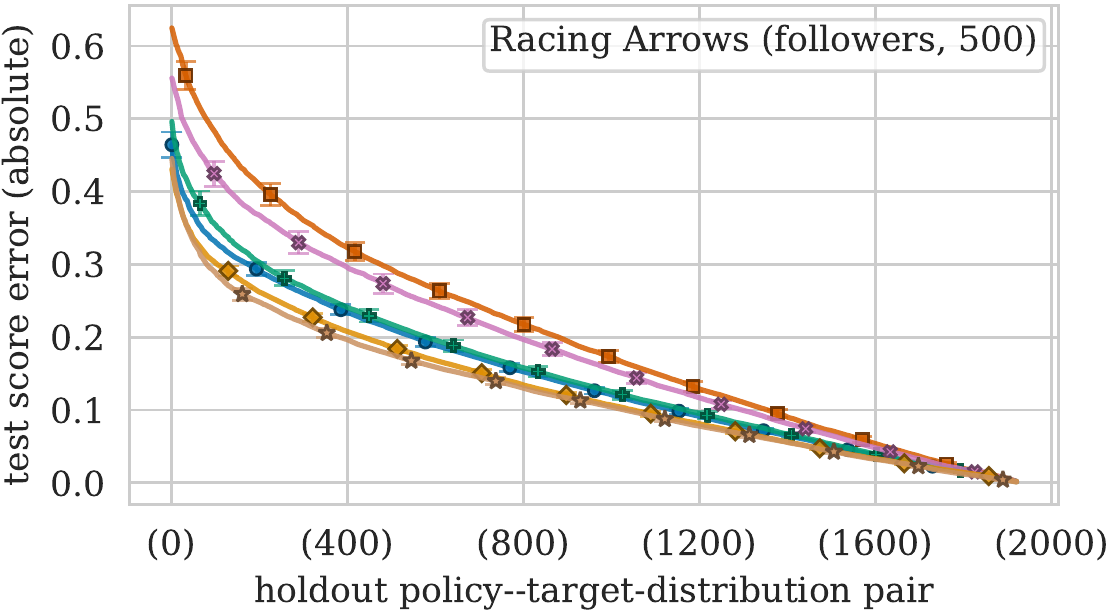}
        \caption{$\mTestCasesToSelect = 3$}
    \end{subfigure}
    \caption{Expected test score error across holdout-policy--target-distribution pairs on Racing Arrows where test cases are follower policies. Here, 500 Racing Arrows policies were sampled for both the follower and leader role and then $96\%$ of policies of both roles were held out before running RPOSST and each baseline. Each column uses a different test size $\mTestCasesToSelect$. $\numHoldoutReplicas$ sets of holdout policies were sampled and each RPOSST$_{\seqLabel}$ instance was run for $500$ rounds ($T = 500$).}
    \label{fig:racing-arrows-f-500}
\end{figure*}

We explore CVaR RPOSST$_{\seqLabel}$'s performance in three two-player games spanning the range of complexity from a toy one-shot game to a high-fidelity racing simulator, in comparison with minimax and miniaverage baselines.
We show that robustness does tend to decrease test score errors on holdout policies and that RPOSST specifically either outperforms or performs about as well as each baseline in each domain.

\subsection{Experimental Setup}

In each domain, we start with data from playing out every pairing of $n > 0$ policies, yielding a matrix of scores for the column policy.
Each policy along the rows of this matrix is then treated as a test case, making the score at row $i$ and column $j$ the result of evaluating policy $j$ on test case $i$.

To emulate unknown deployment candidate policies to be tested, we hold out $h > 0$ columns of this matrix and call the policy associated with a holdout column a \emph{holdout policy}.
The remaining columns represent the test case results for the set of tuning policies.
The resulting $n \times (n - h)$ matrix is shifted and rescaled so that all entries are between zero and one, and then it is set as the test result matrix $A$ that our methods take as input. Note, although $h$ test cases are generated by holdout policies, as test cases they cannot provide any special information about what tests would be effective on the holdout policies.
To simulate scenarios where the set of tuning policies covers the set of future candidate deployment policies to varying degrees, we run experiments with three different values of $h$: $0.2n$, $0.4n$, and $0.6n$.
$\numHoldoutReplicas$ different holdout sets are randomly sampled for each value of $h$ and in each domain.

Given results for $n$ test cases, the goal is to produce a distribution over $\mTestCasesToSelect < n$ test cases that provides accurate test results on the set of holdout policies, according to a set of target distributions.
For our experiments, we use $\mTestCasesToSelect \in \set{1, 2, 3}$ and the set of target distributions generated from the softmax function applied to the negative average test case result under four different scales, specifically,
$\exp\subex*{ \frac{-\beta}{n} A \ones }
    / \ones^{\top} \exp\subex*{ \frac{-\beta}{n} A \ones }$
for $\beta \in \set{0, 1, 2, 4}$,
so that the distributions put varying degrees of emphasis on test cases that are more difficult on average across the tuning policies.
We set the RPOSST uncertainty distribution, $\jointTestCaseDistributionPolicyUncertainty$, to be uniform over each tuning-policy--target-distribution pair.
We set the CVaR percentile to $\cvarPercentile$ so that it is nearly optimizing for the worst-case, but is slightly less pessimistic, to add an additional distinguishing factor to RPOSST compared to the minimax and minaverage baselines.
We use the absolute difference loss for both optimization and evaluation.

\subsection{Domains}\label{sec:domains}

We test RPOSST on the following three domains of varying complexity.
Each domain has two variants arising from asymmetry, multiple datasets, or alternative scoring rules.
Appendix \cref{sec:experimental-details} provides further details on each domain.

\textbf{Racing Arrows.} Racing Arrows is a two-player, zero-sum, one-shot, continuous action game invented for our experiments to replicate aspects of a passing scenario in a race featuring a ``leader'' player and faster ``follower'' player. The follower tries to pass the leader while the latter tries to block.
Scores are recorded as $0$ or $+1$ for a loss or win, respectively, for the column player, which is either the leader or the follower, depending on the configuration.
We run RPOSST on both configurations.
For our experiments, we sample 50 or 500 different leader and follower policies evenly spread through the valid policy space, angles in $\subblock{0, \pi}$, by taking 50 or 500 evenly spaced angles between $\subblock{0.05\pi, (1 - 0.05)\pi}$ and then shifting them independently with uniform samples in $\subblock{-0.05 \pi, 0.05 \pi}$.

\textbf{Annual Computer Poker Competition.} We take two open datasets from the Annual Computer Poker Competition (ACPC) \citep{bard2013annual} containing pairwise match data for poker agents submitted to the 2017 two-player, no-limit competition and the 2012 two-player, limit competition.
These competitions contain different agent populations since they are separated by five years and are in different game formats (limit and no-limit).
The 2017 competition consists of 15 agents and the 2012 competition consists of 12 agents.
Scores are recorded as chip differentials of duplicate matches (two sets of hands where players play with the same set of shuffled decks in both seats).

\textbf{Gran Turismo\textsuperscript{\texttrademark} one-on-one races.} \GranTurismo7/ (GT)\footnote{\url{https://www.gran-turismo.com/us/}} is a high fidelity racing simulator on the PlayStation\textsuperscript{\texttrademark} platform.
Previous versions of GT served as benchmarks for training RL policies \citep{fuchs2021super,song2021autonomous} including policies that outraced the best human competitors~\citep{wurman2022outracing} in four-on-four racing.
We consider a simpler one-on-one racing scenario (see Appendix \cref{sec:gt7} for details).
We carry out two experiments, one where test case results are average winrates, and another where policies receive $0$ for a loss, $+1$ for a win, and $-1$ if there was a collision, making the game non-zero-sum.
The test case pool is comprised of $43$ trained RL policies and $3$ built-in ``AI'' policies.

\subsection{Results and Analysis}\label{sec:results_and_analysis}

The results of running CVaR($\cvarPercentile$) RPOSST$_{\seqLabel}$ on each domain, with $\mTestCasesToSelect = 2$ and $20\%$ of policies marked as holdout policies, are shown in \cref{fig:test-score-error}.
The same set of figures with $\mTestCasesToSelect = 1$ and $\mTestCasesToSelect = 3$, as well as 40\% and 60\% holdout policies, are qualitatively similar, except that the differences between the algorithms are typically smaller, and are provided in Appendix \cref{sec:supplemental-experimental-results}.

\begin{figure*}
  \begin{center}
    \includegraphics[width=\linewidth, keepaspectratio]{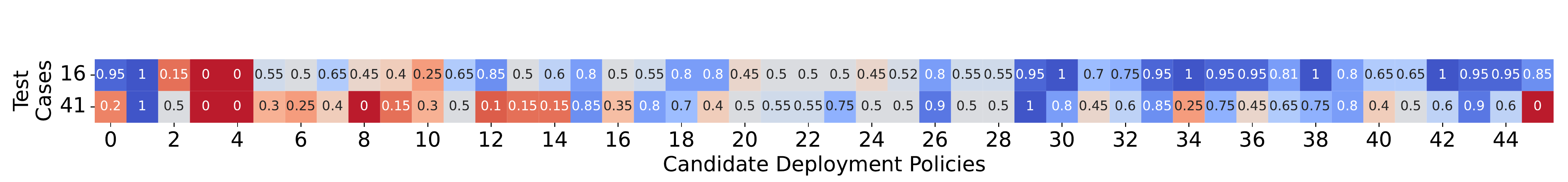}
  \end{center}
  \caption{The GT test results of candidate deployment policies against the test case pair most favoured by RPOSST. Blue and red indicates positive and negative winrates respectively for the candidate deployment policy.}\label{fig:pair_chosen_by_rposst}
\end{figure*}

Looking across each domain and variant, we can see that RPOSST$_{\seqLabel}$ performs nearly as well or better than all of the minimax and miniaverage baselines, particularly in terms of maximum error across holdout-policy--target-distribution pairs.
Interestingly, RPOSST$_{\seqLabel}$ has noticeably lower error in ACPC 2017 and GT (winrate) on the four most difficult holdout-policy--target-distribution pairs to accurately evaluate.
The improvement over the next best method is substantial in ACPC 2017 because RPOSST is the only method with an unlimited ability to optimize with a non-uniform test case weighting.\footnote{Iterative minimax can change its test case distribution away from uniform, but only indirectly by selecting a test case it already selected on a previous iteration before it fills its test-case quota.}
On the other variant in each domain, RPOSST$_{\seqLabel}$ is within the group of the lowest error methods.
In the two Racing Arrows domains, RPOSST$_{\seqLabel}$ and minimax uniform substantially outperform the other methods, at least on the most difficult holdout-policy--target distribution pairs.
This result shows that robustness is indeed beneficial here, but the uniform distribution over the selected two opponents happens to be quite effective.
The GT variant where $-1$ is assigned to a collision appears to be more difficult than the winrate variant, as all the methods cluster together in this variant at higher errors than in the winrate variant.

These results illustrate the utility of incorporating robustness generally, as all of the robust methods tend to outperform miniaverage uniform.
Minimax uniform and iterative minimax are the only baselines that minimize their maximum error over both tuning policy and target distribution uncertainty, and they are usually the next best methods after RPOSST$_{\seqLabel}$.
Minimax(TNP) uniform typically outperforms minimax(TTD) uniform, showing that it is more important to be robust to the tuning policy than the target distribution, in these domains.
When the target distributions are the same in the optimization and holdout evaluation phases, robustness should directly improve the minimum performance across holdout realizations.
Since no effort was made to enforce any relationship between the tuning and holdout policies, this result suggests that robustness to the tuning policy can yield large error reductions when $\PolicySet_{\tuningLabel}$ are even somewhat similar to the holdout policies.

As an example of RPOSST's capabilities, consider the pairs of opponent policies chosen as test cases in GT (winrate) over 100 experiment seeds (Appendix \cref{tab:top_picks_holdout_20}).
RPOSST$_{\seqLabel}$ is both more accurate (\cref{fig:test-score-error}) and very consistent, choosing the same pair 90\% of the time.
\Cref{fig:pair_chosen_by_rposst} illustrates the portion of the result matrix for just the two test cases most frequently chosen by RPOSST$_{\seqLabel}$ (test races against opponents 16 and 41).  The race against policy 41 (bottom row) is chosen because that policy wins/loses about half the time, providing a 50/50 information split.
Policy 16 is a weaker policy in many ways (more blue in the top row) but it serves to differentiate the worst policies (darker red squares in the left side of the matrix) from the rest of the policies, and to highlight the strongest policies.
Specifically, the best performing policies almost always win against policy 16, which provides a strong complementary signal to the noisier but more competitive policy 41 test case.
Overall, the two test cases indicate policies 1, 29, and 43 (darkest blue columns) are the strongest for deployment.
Policy 1 is a built-in AI in an overpowered car but 29 and 43 are very strong RL policies.  Looking at the overall winrate matrix (Appendix \cref{figure:gt_winrate_matrix}) we see that the same conclusion (the three darkest blue columns overall) would have been chosen using all 46 test cases.
Compressing from 46 test cases to two presents a massive saving in test time for future policies, and shows RPOSST$_{\seqLabel}$ can construct small tests to select deployment policies in a real and complex video game.

The results in \cref{fig:racing-arrows-f-500} repeat the previous analysis in Racing Arrows but with ten times the number of policies.
Only the results where follower policies are treated as test cases are shown, but the corresponding results where leader policies are test cases appear similar and are shown in Appendix \cref{sec:supplemental-experimental-results}.
$96\%$ of policies are held out, including those used as test cases, so there are only 20 test cases and tuning policies for RPOSST and the other algorithms to utilize.
This experiment emulates a scenario where an efficient test is constructed once with a relatively small number of tuning policies and then reused for many future deployment candidates.
As in the previous experiments, RPOSST is almost always one of the best methods.

\section{Related Work}

The bulk of the work on policy selection in RL focuses on selecting opponents for \emph{training} with self-play algorithms~\citep{hernandez2021comparison}. In that case, diversity is key for training additional policies to cooperate~\citep{rahman2022towards} or compete~\citep{liu2021towards,mcaleer2022anytime} with pre-existing policies.  However, the selection of policies as training opponents is often guided by aggregate performance metrics across entire populations~\citep{li2019generalized,lanctot2017unified,omidshafiei2019alpha,balduzzi2018re} and thus do not reduce the number of opponent pairings (test cases) required for assessments.

On the testing side, researchers in complex domains develop procedures for testing skill competency using hand-calibrated~\citep{wurman2022outracing} or randomly generated tests with complex percentile-scoring functions~\citep{team2021open}.
Our work seeks to automate and target test construction in such scenarios.
Complementary work~\citep{rowland2019multiagent} treats the computation of a result matrix as a multi-armed bandit problem, each entry represented by one arm.
While this method can greatly reduce sampling costs in the presence of low-variance outcomes,
it does not generalize to policies outside its input population, with the testing of a new policy requiring adding extra arms to be estimated from scratch. However, this method could be used in tandem with RPOSST to reduce the samples required to compute $A$.

Learning to rank methods~\citep{oosterhuis2021robust,bruch2021alternative,hu2018reinforcement} aim to find a function that ranks a set of items (\eg/, documents) based on the relevance of a given query, with hopes to generalize to future queries. Indeed, \citet{akiyama2016learning} use learning to rank to evaluate action sequences. However, predicting unseen policy performances under this model requires the tuning policies to be the queries, which would produce a ranking of the test cases themselves. The scores from such tests would therefore be incomparable across policies, violating one of our main objectives.

Test construction in educational modeling~\citep{vanderlinden2005linear} starts from an
item bank and a statistical model (\eg/, Item Response Theory~\citep{embretson2000item}) predicting the probability of answering each item correctly given a student's (unobserved) skill level.  That model yields an \emph{information matrix} and then automatic test construction methods, including linear optimization or greedy heuristics, can then build a finite-sized test.
By contrast, we do not assume a model of the response variance or a univariate skill measurement, so a closed-form calculation of information is often infeasible.  However,  we do empirically compare our optimization approach to the greedy heuristic.

\section{Conclusion and Future Work}
RPOSST is, to the best of our knowledge, the first algorithm to directly address test construction for reinforcement learning policies.  By leveraging the $k$-of-$N$ framework, RPOSST provides bounds on the approximation error of the resulting test despite uncertainty over the exact policies that will be evaluated and the desired test case weighting in the future.  Thus, RPOSST provides a much needed tool for policy selection in real-world deployment scenarios.  An interesting direction for future work is generating the test cases themselves~\citep{marris2021multi,pugh2016quality}, which is challenging on its own~\citep{balduzzi2019open}.

\begin{acknowledgements}
    Thanks to Francesco Riccio for reviewing this work. Thanks to the whole Sony AI team for experiment infrastructure.
\end{acknowledgements}
 \bibliography{morrill_257}

\newpage
\appendix
\onecolumn

\section{Appendix}

\section{Glossary}

\begin{description}
    \item[Policy.] A policy to solve a control problem or play a game, potentially generated by an RL algorithm.
    \item[Deployment policy.] A policy used in production, \eg/, deployed to end users, used in a competition, or integrated into a technology demonstration.
    \item[Deployment candidate.] A policy in consideration for deployment.
    \item[Test.] The aggregate result of test cases applied to a policy.
    \item[Test case.] An atomic unit of a test that reveals a particular skill or emulates a specific deployment scenario. RPOSST selects a small number of test cases and a distribution over them so that we can avoid executing all conceivable test cases on every deployment candidate every time we want to deploy a policy.
    \item[Test case result.] The numerical result of evaluating a policy on a test case.
        This number should be a good estimate of the policy's expected performance in the test case scenario, but it maybe noisy if the test case is stochastic, \eg/, the average test case result observed from Monte Carlo rollouts.
    \item[Test score.] The final score produced by a test, \ie/, the average test case result across test cases, perhaps weighted by the relative importance of each test case.
    \item[Tuning policy.] A policy used at the start of the RPOSST procedure to gather information about test cases. Each tuning policy is evaluated on each test case to construct the test case result matrix that forms the basis of RPOSST's loss function.
\end{description}

\section{Theory Background}

We make use of six basic results, which are restated here for completeness.
\begin{proposition}[Azuma-Hoeffding inequality.]
  \label{prop:azumaHoeffdingInequality}
  For constants $\tuple{c^t \in \reals}_{t = 1}^T$, martingale difference sequence $\tuple{Y^t \in \reals}_{t = 1}^T$ where $\abs{Y^t} \le c^t$ for each $t$, and $\tau \ge 0$,
  \[
    \Prob\subblock*{\abs*{\sum_{t = 1}^T Y^t} \ge \tau}
      \le 2 \exp\subex*{\frac{-\tau^2}{2 \sum_{t = 1}^T (c^t)^2}}.
  \]
\end{proposition}
For proof, see that of Theorem 3.14 by \citet{azumaHoeffdingInequality}.
\begin{proposition}[Regret matching$^+$ regret bound]
    \label{prop:rmp-regret-bound}
    Consider an online decision process with $\mTestCasesToSelect$ actions and the set of bounded, linear loss functions,
    $\LossSet = [0, \maxLoss]^{\mTestCasesToSelect}$.
    Regret matching$^+$ accumulates pseudoregrets
    $q^{1:t} = \ramp{q^{1:t - 1} + \regret^t}$, $q^{1:0} = \zeros$,
    where
    $\regret^t = (\ell^t)^{\top} \testCaseDistribution^t - \ell^t$
    is the instantaneous regret on round $t$ under loss function $\ell^t \in \LossSet$, and
    $\testCaseDistribution^t = q^{1:t - 1} / (\ones^{\top} q^{1:t - 1})$
    if
    $\ones^{\top} q^{1:t - 1} > 0$
    or
    $\testCaseDistribution^t = \frac{1}{\mTestCasesToSelect} \ones$ otherwise,
    is regret matching$^+$'s action distribution on round $t$.
    After $T$ rounds, regret matching$^+$'s cumulative regret is bounded as
    $\sum_{t = 1}^T \regret^t
        \le \maxLoss \sqrt{T \mTestCasesToSelect}$.
\end{proposition}
For proof, see \citet{solvingHulhe}.
\begin{proposition}[The linearization trick]
    \label{prop:linearization-trick}
    Consider an online decision process with convex decision set $\DecisionSet \subseteq \reals^{\mTestCasesToSelect}$ and a set of bounded, convex loss functions
    $\LossSet \subseteq \set*{\ell \where \ell: \DecisionSet \to [0, \maxLoss]}$,
    where each loss function $\ell \in \LossSet$ has subgradients with bounded maximum magnitude, \ie/,
    $\norm{\grad \ell(\odpDecision)}_{\infty} \le \maxGrad$,
    for all
    $\odpDecision \in \DecisionSet$.
    The instantaneous regret under loss function $\ell \in \LossSet$ is upper bounded by the instantaneous regret under the loss function subgradient $\grad \ell(\odpDecision)$ given decision $\odpDecision \in \DecisionSet$, \ie/,
    \[
        \ell(\odpDecision) - \ell(\odpDecision')
            \le
                \subex*{\grad \ell(\odpDecision)}^{\top} \odpDecision - \subex*{\grad \ell(\odpDecision)}^{\top} \odpDecision'.
    \]
\end{proposition}
\begin{proof}
    From the convexity of $\ell$, its first-order Taylor expansion lower bounds $\ell$, \ie/,
    $\ell(\odpDecision')
        \ge
            \ell(\odpDecision)
            + \subex*{\grad \ell(\odpDecision)}^{\top}
            \subex*{ \odpDecision' - \odpDecision}$,
    for all $\odpDecision, \odpDecision' \in \DecisionSet$.
    Therefore,
    \begin{align*}
        \ell(\odpDecision) - \ell(\odpDecision')
            &\le
                \ell(\odpDecision)
                - \subex*{
                    \ell(\odpDecision)
                    + \subex*{\grad \ell(\odpDecision)}^{\top} \subex*{ \odpDecision' - \odpDecision }
                }\\
            &=
                \subex*{\grad \ell(\odpDecision)}^{\top} \odpDecision
                - \subex*{\grad \ell(\odpDecision)}^{\top} \odpDecision',
    \end{align*}
    as required.
\end{proof}
\begin{proposition}[Lemma 2 of \citet{ED,exploitabilityDescentArxiv}]
    \label{prop:best-optimality}
    Assume that on each round $t$ of an online decision process with decision set $\DecisionSet \subseteq \reals^{\mTestCasesToSelect}$ and bounded loss functions from
    $\LossSet \subseteq \set*{\ell \where \ell: \DecisionSet \to [0, \maxLoss]}$,
    the loss function $\ell^t$ maximizes the loss of $\odpDecision^t \in \DecisionSet$ chosen by the decision-maker, \ie/,
    $\ell^t
        \in \argmax_{\ell \in \LossSet}
            \ell(\odpDecision^t)$.
    On the round $t^*$ where the minimum loss was observed,
    $t^* \in \argmin_{t \in \set{1, \ldots, T}}
        \ell^t(\odpDecision^t)
    $,
    the decision $\odpDecision^{t^*}$ has a maximum loss that is no more than $\frac{1}{T} \regret^{1:T}(\odpDecision)$ larger than that of any alternative decision $\odpDecision \in \DecisionSet$, \ie/,
    $\ell^{t^*}(\odpDecision^{t^*}) - \ell^{\odpDecision}(\odpDecision) \le \frac{1}{T} \regret^{1:T}(\odpDecision)$,
    where $\ell^{\odpDecision}(\odpDecision) \in \LossSet$ is a loss function that maximizes the loss on $\odpDecision$.
\end{proposition}
\begin{proof}
    Since the loss function on each round is chosen to maximize loss, the average regret for not choosing $\odpDecision \in \DecisionSet$ is lower bounded as
    \begin{align*}
        \dfrac{1}{T} \regret^{1:T}
            &\ge
                \dfrac{1}{T} \min_{t \in \set{1, \ldots, T}} T \ell^t(\odpDecision^t) - \dfrac{1}{T} \sum_{t = 1}^T \ell^t(\odpDecision)\\
            &\ge
                \ell^{t^*}(\odpDecision^{t^*}) - \ell^{\odpDecision}(\odpDecision),
    \end{align*}
    as required.
\end{proof}
\begin{proposition}[Theorem 4 of \citet{cfrbr}]
    \label{prop:sample-optimality}
    Assume that on each round $t$ of an online decision process with decision set $\DecisionSet \subseteq \reals^{\mTestCasesToSelect}$ and bounded (possibly random) loss functions from
    $\LossSet \subseteq \set*{\ell \where \ell: \DecisionSet \to [0, \maxLoss]}$,
    the loss function $\ell^t$ maximizes the loss of $\odpDecision^t \in \DecisionSet$ chosen by the decision-maker, \ie/,
    $\ell^t
        \in \argmax_{\ell \in \LossSet}
            \ell(\odpDecision^t)$.
    The loss function that the decision-maker observes on each round $t$ may be a random loss function $\hat{\ell}^t$ where $\E\subblock*{\hat{\ell}^t} = \ell^t$.
    On round
    $T' \sim \Unif(\set{1, \ldots, T})$
    after $T$ rounds of the online decision process, the decision $\odpDecision^{T'}$ has a maximum loss that is no more than
    $\frac{1}{qT} \regret^{1:T}(\odpDecision)$
    larger than that of any alternative decision $\odpDecision \in \DecisionSet$ with probability $1 - q$, $q \in (0, 1]$, \ie/,
    $\ell^{T'}(\odpDecision^{T'}) - \ell^{\odpDecision}(\odpDecision) \le \frac{1}{qT} \regret^{1:T}(\odpDecision)$
    holds with probability $1 - q$,
    where $\ell^{\odpDecision}(\odpDecision) \in \LossSet$ is a loss function that maximizes the loss on $\odpDecision$
    and the cumulative regret $\regret^{1:T}$ is with respect to the expected loss functions, $\tuple{\ell^t}_{t = 1}^T$.
\end{proposition}
See \citet{cfrbr} for proof.
\begin{proposition}[Successive Rejects error probability]
    \label{prop:sr-error-prob}
    Consider a best action identification task with $\mTestCasesToSelect$ actions from set $\Actions$.
    Each time an action $a \in \Actions$ is selected, a random sample of that action's loss, $\ell(a) \in [-0.5, 0.5]$, under a fixed but random loss function $\ell$, is observed.
    The goal is to identify an action $a^* \in \Actions^* \subset \Actions$ with the lowest expected loss, $\E\subblock*{\ell(a^*)}$, after $T$ samples.
    The probability that the action returned by the Successive Rejects algorithm is in $\Actions^*$ is at least
    \[
        1 - \dfrac{\mTestCasesToSelect\subex{ \mTestCasesToSelect - 1}}{2}
            \exp\subex*{ - \dfrac{T - \mTestCasesToSelect}{\overline{\log}(\mTestCasesToSelect)H_2} },
    \]
    where
    $\overline{\log}(\mTestCasesToSelect) = \frac{1}{2} + \sum_{i = 2}^{\mTestCasesToSelect} \frac{1}{i}$,
    $H_2 = \max_{i \in \set{1, \ldots, \abs{\Actions \setminus \Actions^*}}}
        \frac{i}{\subex*{\E\subblock*{\ell(a_{(i)})} - \E\subblock*{\ell(a^*)}}^2}$,
    and $a_{(i)}$ is the action that achieves the $i^{\text{th}}$ smallest loss (with ties broken arbitrarily) among the suboptimal actions.
\end{proposition}
See \citet{audibert2010BestArmIdentification} for proof.

\section{Sequential-Move Model Theory}

\begin{lemma}
  \label{lem:regret_matching_k_of_n_hp_regret_bound}
  Consider a $k$-of-$N$ game with $\mTestCasesToSelect$ actions and the set of bounded, convex loss functions
  $\LossSet = \set*{\ell \where \ell: \simplex^{\mTestCasesToSelect} \to [0, \maxLoss]}$,
  where each loss function $\ell \in \LossSet$ has subgradients with bounded maximum magnitude, \ie/,
  $\norm{\grad \ell(\testCaseDistribution)}_{\infty} \le \maxGrad$,
  for all
  $\testCaseDistribution \in \simplex^{\mTestCasesToSelect}$.
  Let the $k$-worst loss functions from $N$ of those sampled from the given uncertainty distribution $\jointTestCaseDistributionPolicyUncertainty$ on round $t$ be
  $\tuple{ \ell^t_{(i)} \in \LossSet }_{i = 1}^k$.
  The randomly sampled $k$-of-$N$ loss function on round $t$ is then the average
  $\bar{\ell}^t
      = \frac{1}{k} \sum_{i = 1}^k \ell_{(i)}$.
  After $T$ rounds, regret matching$^+$ on the random loss gradients $\grad \bar{\ell}^t(\testCaseDistribution^t)$ has no more than
  $2 \maxGrad \sqrt{T \mTestCasesToSelect} + 2 \maxLoss \sqrt{2 T \log{\nicefrac{1}{p}}}$
  cumulative regret on the expected $k$-of-$N$ losses, $\tuple*{ \E\subblock{ \bar{\ell} }^t }_{t = 1}^T$, with probability $1 - p$, $p > 0$.
\end{lemma}
\begin{proof}
  Since regret matching$^+$ observes and learns directly from $\grad \bar{\ell}^t$, its regret for not always choosing $\testCaseDistribution \in \simplex^{\mTestCasesToSelect}$, under the sampled loss functions, is deterministically upper bounded as
  \[
      \Regret^{1:T}
          = \sum_{t = 1}^T
              \underbrace{
                  \bar{\ell}^t(\testCaseDistribution^t)
                  - \bar{\ell}^t(\testCaseDistribution)
              }_{\as \Regret^t}
          \le 2 \maxGrad \sqrt{T \mTestCasesToSelect},
  \]
  where $\tuple{\testCaseDistribution^t \in \simplex^{\mTestCasesToSelect}}_{t = 1}^T$ are the decisions made by regret matching$^+$.
  This bound comes from regret matching$^+$'s regret bound on linear losses (\cref{prop:rmp-regret-bound}) and the linearization trick (\cref{prop:linearization-trick}), which states that the regret on loss gradients upper bounds that of the loss itself, \ie/,
  $\Regret^{1:T}
      \le
          \sum_{t = 1}^T
              \subex*{\grad \bar{\ell}^t(\testCaseDistribution^t)}^{\top} \testCaseDistribution^t
              - \subex*{\grad \bar{\ell}^t(\testCaseDistribution)}^{\top} \testCaseDistribution$.

  The rest of the proof largely follows the proof of \citet{farina2020stochasticRegretMin}'s Proposition 1.
  The sequence of differences,
  $\tuple*{
      \E\subblock*{\Regret^t} - \Regret^t \le 2\maxLoss
  }_{t = 1}^T,$
  is a bounded martingale difference sequence.

  The probability that the expected cumulative regret,
  $\E[\Regret^{1:T}]$,
  is bounded by the cumulative sampled regret plus slack
  $\lambda \ge 0$ is bounded according to the Azuma-Hoeffding inequality (\cref{prop:azumaHoeffdingInequality}) as
  \begin{align}
  &\Prob\subblock*{\E[\Regret^{1:T}] \le \Regret^{1:T} + \lambda}\\
  &\le
      \Prob\subblock*{\sum_{t = 1}^T \E[\Regret^t] - \Regret^t \le \lambda}\\
  &=
      1 - \Prob\subblock*{
      \sum_{t = 1}^T
          \E[\Regret^t] - \Regret^t
      \ge \lambda}\\
  &\le
      1 - \exp\subex*{\dfrac{2\lambda^2}{4T \subex*{2\maxLoss}^2}}.
  \end{align}
  Setting $\lambda = 2\maxLoss \sqrt{2T\log(\nicefrac{1}{p})}$ ensures that
  \[
  \E[\Regret^{1:T}]
      \le \Regret^{1:T} + 2\maxLoss \sqrt{2 T \log{\nicefrac{1}{p}}}
  \]
  with probability $1 - p$.
  Since $\Regret^{1:T} \le 2 \maxLoss \sqrt{T \mTestCasesToSelect}$,
  \[
  \E[\Regret^{1:T}]
      \le 2 \maxGrad \sqrt{T \mTestCasesToSelect} + 2\maxLoss \sqrt{2 T \log{\nicefrac{1}{p}}}
  \]
  with probability $1 - p$, as required.
\end{proof} 
\begin{theorem}
    \label{thm:rposst-seq}
    After $T' \sim \Unif(\set{1, \ldots, T_1})$, $T_1 > 0$, rounds of its optimization game, \cref{alg:rposst-seq} selects an $\mTestCasesToSelect$-tuple of test cases, $\testCaseGroup^*$ and weights
$\hat{\testCaseDistribution}^{T'}_{\testCaseGroup^*} \in \simplex^{\mTestCasesToSelect}$
that, with probability $(1 - p)(1 - q)(1 - \alpha)$, $p, q, \alpha > 0$, are $\frac{\gap}{q}$-optimal for \cref{eq:rposst-seq-objective}, where
$\gap = \bigO\subex*{ \sqrt{\frac{1}{T_1} \mTestCasesToSelect} + \sqrt{\frac{1}{T_1} \log\subex*{ \nicefrac{1}{p} }} }$
and
$\alpha = \bigO\subex*{\e^{-T_2}}$. \end{theorem}
\begin{proof}
    \def\kOfNLossFn{L_{\probMeasure_{\kOfN}, \jointTestCaseDistributionPolicyUncertainty}}
    Recall that the $k$-of-$N$ loss $\bar{\ell}^t$ that RPOSST$_{\seqLabel}$ updates from on each round $t = 1, \ldots, T_1$ is a Monte Carlo estimate of the $k$-of-$N$ percentile loss,
    \begin{align}
        \kOfNLossFn(\hat{\testCaseDistribution}^t_{\testCaseGroup})
            = \inf_{\integrableFn \in \IntegrableFnSet}
        \hspace{-1em}
        \underset{
            \substack{
                \percentile \in [0, 1]\\
                \Prob_{\policy_j, \testCaseDistribution} \subblock*{
					\ell(\hat{\testCaseDistribution}^t_{\testCaseGroup}; \policy_j, \testCaseDistribution) \le \integrableFn(\eta)
                } \ge \percentile
            }
        }{\int}
            \hspace{-3em}
            \integrableFn(\eta)
            \probMeasure_{\kOfN}(d\eta)
            = \E_{\policy_j, \testCaseDistribution}\subblock{ \bar{\ell}^t },
        \label{eq:loss-given-test-case-group}
    \end{align}
    where $(\policy_j, \testCaseDistribution) \sim \jointTestCaseDistributionPolicyUncertainty$.
    The sequence of test case weights,
    $\tuple{ \testCaseDistribution_{\testCaseGroup}^t }_{t = 1}^{T_1}$,
    for each $\mTestCasesToSelect$-tuple of test cases $\testCaseGroup \subset \testCaseSet$ is therefore random.
    All of the following probabilities and expectations are with respect to these random variables.

    \Cref{lem:regret_matching_k_of_n_hp_regret_bound} guarantees that RPOSST$_{\seqLabel}$, in generating the test case weight sequence $\tuple{ \testCaseDistribution_{\testCaseGroup}^t }_{t = 1}^{T_1}$ has no more than
    $C = 2 \maxGrad \sqrt{T_1 \mTestCasesToSelect} + 2\maxLoss \sqrt{2 T_1 \log{\nicefrac{1}{p}}}$
    cumulative regret on the $k$-of-$N$ percentile losses,
    \[\regret^{1:T_1}_{\testCaseDistribution_{\testCaseGroup}}
        = \sum_{t = 1}^T
            \kOfNLossFn(\hat{\testCaseDistribution}^t_{\testCaseGroup})
            - \kOfNLossFn(\testCaseDistribution_{\testCaseGroup}),
    \]
    for not always selecting test case weights $\testCaseDistribution_{\testCaseGroup}$,
    with probability $1 - p$.
    That is, $1 - p = \Prob\subblock*{ \regret^{1:T_1}_{\testCaseDistribution_{\testCaseGroup}} \le C}$.

    \cref{prop:sample-optimality} guarantees that, on round $T' \sim \Unif(1, \ldots, T_1)$, the weights for each $\mTestCasesToSelect$-tuple are
    $\frac{1}{qT_1} \regret^{1:T_1}_{\testCaseDistribution_{\testCaseGroup}}$
    close to optimal for \cref{eq:loss-given-test-case-group}, with probability $1 - q$.
    That is,
    $1 - q
        = \Prob\subblock*{
            \kOfNLossFn(\hat{\testCaseDistribution}^{T'}_{\testCaseGroup})
            - \kOfNLossFn(\testCaseDistribution_{\testCaseGroup})
                \le \frac{\regret^{1:T_1}_{\testCaseDistribution_{\testCaseGroup}}}{qT_1}
        }$,
    and this holds regardless of the value of $\regret^{1:T_1}_{\testCaseDistribution_{\testCaseGroup}}$, \ie/,
    $\Prob\subblock*{
        \kOfNLossFn(\hat{\testCaseDistribution}^{T'}_{\testCaseGroup})
        - \kOfNLossFn(\testCaseDistribution_{\testCaseGroup})
            \le \frac{\regret^{1:T_1}_{\testCaseDistribution_{\testCaseGroup}}}{qT_1}
    } = \Prob\subblock*{
        \kOfNLossFn(\hat{\testCaseDistribution}^{T'}_{\testCaseGroup})
        - \kOfNLossFn(\testCaseDistribution_{\testCaseGroup})
            \le \frac{\regret^{1:T_1}_{\testCaseDistribution_{\testCaseGroup}}}{qT_1}
        \given
        \regret^{1:T_1}_{\testCaseDistribution_{\testCaseGroup}} \le C'
    }$
    for all $C' \in \reals$.

    Combining these two results, we see that the probability that $\hat{\testCaseDistribution}^{T'}_{\testCaseGroup}$ has at most $\frac{C}{qT_1}$ excess $k$-of-$N$ percentile loss is
    \begin{align}
        \Prob\subblock*{
            \kOfNLossFn(\hat{\testCaseDistribution}^{T'}_{\testCaseGroup})
            - \kOfNLossFn(\testCaseDistribution_{\testCaseGroup})
                \le \frac{C}{qT_1}
        }
            &= \Prob\subblock*{
                \kOfNLossFn(\hat{\testCaseDistribution}^{T'}_{\testCaseGroup})
                - \kOfNLossFn(\testCaseDistribution_{\testCaseGroup})
                    \le \frac{\regret^{1:T_1}_{\testCaseDistribution_{\testCaseGroup}}}{qT_1},
                \regret^{1:T_1}_{\testCaseDistribution_{\testCaseGroup}} \le C
            }\\
            &= \Prob\subblock*{
                \kOfNLossFn(\hat{\testCaseDistribution}^{T'}_{\testCaseGroup})
                - \kOfNLossFn(\testCaseDistribution_{\testCaseGroup})
                    \le \frac{\regret^{1:T_1}_{\testCaseDistribution_{\testCaseGroup}}}{qT_1}
                \given
                \regret^{1:T_1}_{\testCaseDistribution_{\testCaseGroup}} \le C
            } \Prob\subblock*{
                \regret^{1:T_1}_{\testCaseDistribution_{\testCaseGroup}} \le C
            }\\
            &= \Prob\subblock*{
                \kOfNLossFn(\hat{\testCaseDistribution}^{T'}_{\testCaseGroup})
                - \kOfNLossFn(\testCaseDistribution_{\testCaseGroup})
                    \le \frac{\regret^{1:T_1}_{\testCaseDistribution_{\testCaseGroup}}}{qT_1}
            } \Prob\subblock*{
                \regret^{1:T_1}_{\testCaseDistribution_{\testCaseGroup}} \le C
            }\\
            &= (1 - p)(1 - q).
    \end{align}

    The last remaining step is to complete the outer minimization in \cref{eq:rposst-seq-objective} to select a single $\mTestCasesToSelect$-tuple of test cases.
    Since the $k$-of-$N$ loss observed on each round is random, we cannot compute a simple argmin using the test case weights on round $T'$, and are instead faced with a best arm identification problem.
    For this, we run the Successive Rejects algorithm, which we know from \cref{prop:sr-error-prob} identifies a minimum loss $\mTestCasesToSelect$-tuple of test cases with probability at least
    \[
        \alpha = 1 - \dfrac{\mTestCasesToSelect\subex{ \mTestCasesToSelect - 1}}{2}
            \exp\subex*{ - \dfrac{T_2 - \mTestCasesToSelect}{\overline{\log}(\mTestCasesToSelect)H_2} }.
    \]
    The probability of selecting the best $\mTestCasesToSelect$-tuple using the test case weights on round $T'$ is independent of whether or not the regret bound $C$ was actually achieved or if the test case weights on $T'$ are actually nearly optimal for any given $\mTestCasesToSelect$-tuple, the probability of which we previously characterized as $(1 - p)(1 - q)$.
    Therefore, the probability of achieving $\frac{C}{qT_1}$-optimality given each $\mTestCasesToSelect$-tuple and selecting the best $\mTestCasesToSelect$-tuple is the product $(1 - p)(1 - q)(1 - \alpha)$, as required.
\end{proof}

The $\sqrt{\abs{\testCaseSet}}$ dependence in \cref{thm:rposst-seq} could be improved to $\sqrt{\log\subex*{\abs{\testCaseSet}}}$ if regret matching$^+$ (within or without CFR, respectively) was replaced with an algorithm like Hedge~\citep{hedge}, but this tends to lead to worse performance in practice (see, \eg/, \citet{solvingHulhe,burch2017time}).

In the deterministic CVaR RPOSST case, we get the following corollary.
\begin{corollary}
    \label{cor:deterministic-rposst-seq}
    Assume that
$\jointTestCaseDistributionPolicyUncertainty \in \simplex^d$
for some finite $d \ge 1$.
After $T$ rounds of the CVaR($\percentile$) RPOSST$_{\seqLabel}$ optimization game, where the protagonist chooses $\mTestCasesToSelect$-size tests according to regret matching$^+$ against a best response antagonist,
$\testCaseGroup^*$ and $\testCaseDistribution^{t^*}_{\testCaseGroup^*}$
are $\gap$-optimal for \cref{eq:rposst-seq-objective} under the $\percentile$-fractile CVaR robustness measure, where
$\gap = \bigO\subex*{ \sqrt{\frac{1}{T} \mTestCasesToSelect}}$. \end{corollary}
\begin{proof}
    \Cref{prop:rmp-regret-bound} and \cref{prop:best-optimality} ensures that there is a round $t^*_{\testCaseGroup} \le T$ where
    $\testCaseDistribution_{\testCaseGroup}^{t^*_{\testCaseGroup}}$
    is
    $2 \maxGrad \sqrt{\mTestCasesToSelect \dfrac{1}{T}}$-optimal on the deterministic $k$-of-$N$ losses.
    Since the $k$-of-$N$ loss function observed on each round is deterministic, we can perform a simple minimization across $\set{1, \ldots, T}$ and the $\mTestCasesToSelect$-tuple of test cases to find the minimizers $t^*$ and $\testCaseGroup^*$, leading to the stated optimality guarantee.
\end{proof}

\section{Deterministic CVaR$(\percentile)$ RPOSST$_{\seqLabel}$ Pseudocode}
\begin{algorithm2e}[tb]
    \caption{Deterministic CVaR$(\percentile)$ RPOSST$_{\seqLabel}$ with regret matching$^+$}\label{alg:cvar-rposst-seq}
    \DontPrintSemicolon
    \textbf{\textit{Inputs:}} $\langle \percentile, T, \mTestCasesToSelect, \jointTestCaseDistributionPolicyUncertainty, \testCaseGroup^0, \ell \rangle$
    \vspace{0.2em} \hrule \vspace{0.2em}

    $q^{1:0}_{\testCaseGroup} \gets \zeros \in \reals^{\mTestCasesToSelect + \abs{\testCaseGroup^0}}$ \textbf{for} $\testCaseGroup \in \testCaseSet^{\mTestCasesToSelect}$

    $t^* \gets 1$

    $\bar{\cfv}^{t^*} \gets -\infty$

    \For{$t \gets 1, \ldots, T$}{
        \For{$\testCaseGroup \in \testCaseSet^{\mTestCasesToSelect}$}{
            $z^t \gets \ones^{\top} q^{1:t - 1}_{\testCaseGroup}$

            $\hat{\testCaseDistribution}^t_{\testCaseGroup} \gets q^{1:t - 1}_{\testCaseGroup} / z^t$ \textbf{if} $z^t > 0$ \textbf{else} $\ones / \mTestCasesToSelect$

            \tcp{Fill in zeros so that $\hat{\testCaseDistribution}^t_{\testCaseGroup} \in \simplex^{\abs{\testCaseSet}}$.}
            $\hat{\testCaseDistribution}^t_{\testCaseGroup}(x) \gets 0$ \textbf{for} $x \in \testCaseSet \setminus (\testCaseGroup \cup \testCaseGroup^0)$

$\subblock*{ \ell_{\testCaseGroup, (i)} }_{i = 1}^d \gets
              \WorstFactileLossesFn\subex*{\hat{\testCaseDistribution}^t_{\testCaseGroup}, \percentile, \jointTestCaseDistributionPolicyUncertainty, \ell}$

            $\cfv^t_{\testCaseGroup} \gets
              -\sum_{i = 1}^d \dfrac{\partial \ell_{\testCaseGroup, (i)}, \policy_{j_{(i)}}}{\partial \hat{\testCaseDistribution}^t_{\testCaseGroup}}$

            \tcp{Update regret matching$^+$.}
            $\bar{\cfv}_{\testCaseGroup}^t \gets (\hat{\testCaseDistribution}^t_{\testCaseGroup})^{\top} \cfv^t_{\testCaseGroup}$

            $\regret^t_{\testCaseGroup} \gets \cfv^t_{\testCaseGroup} - \bar{\cfv}_{\testCaseGroup}^t$

            $q^{1:t}_{\testCaseGroup} \gets \ramp{q^{1:t - 1}_{\testCaseGroup} + \regret^t_{\testCaseGroup}}$

            \tcp{Update the best round.}
            \If{$\bar{\cfv}_{\testCaseGroup}^t > \bar{\cfv}^{t^*}$}{
              $t^* \gets t$

              $\bar{\cfv}^{t^*} \gets \bar{\cfv}_{\testCaseGroup}^t$
            }
        }
    }

    \Return $\testCaseGroup^{t^*}, \hat{\testCaseDistribution}^{t^*}_{\testCaseGroup^*}$

    \vspace{0.5em} \hrule \vspace{0.2em}
    \setcounter{AlgoLine}{0}
    \SetKwProg{Subroutine}{Procedure}{}{}
    \Subroutine{$\WorstFactileLossesFn$ \quad \textbf{Inputs:} $\langle \hat{\testCaseDistribution}, \percentile, \jointTestCaseDistributionPolicyUncertainty, \ell \rangle$}{
      \vspace{0.2em} \hrule \vspace{0.2em}

      \tcp{The support of $\jointTestCaseDistributionPolicyUncertainty$, $\supportFn(\jointTestCaseDistributionPolicyUncertainty)$, is assumed to be a finite number $d = \abs{\supportFn(\jointTestCaseDistributionPolicyUncertainty)}$.}

      \For{$\policy_{j_i}, \testCaseDistribution_i \in \supportFn(\jointTestCaseDistributionPolicyUncertainty)$}{
        \tcp{Evaluate $\hat{\testCaseDistribution}$.}
        $\ell_i \gets \ell(\hat{\testCaseDistribution}; \policy_{j_i},\testCaseDistribution_i)$
      }

      $\SortFn\subex*{
        \set*{ i \where \ell_i }_{i = 1}^d
      }$

      \tcp{Assign weights to each loss function.}
      \tcp{Iterate over $\jointTestCaseDistributionPolicyUncertainty$'s support sorted accoding to descending loss value from the previous step.}
      $\beta \gets 0$

      \For{$\policy_{j_{(i)}}, \testCaseDistribution_{(i)} \in \supportFn(\jointTestCaseDistributionPolicyUncertainty)$}{
        $\alpha_{(i)} = \min \set*{
          \jointTestCaseDistributionPolicyUncertainty\subex*{
            \tuple*{\policy_{j_{(i)}}, \testCaseDistribution_{(i)}}
          },
          \percentile - \beta
        }$

        $\beta \gets \beta + \alpha_{(i)}$
      }

      \Return $\subblock*{ \frac{\alpha_{(i)}}{\percentile} \ell_{(i)} }_{i = 1}^d$
    }
  \end{algorithm2e}
 Pseudocode for CVaR($\percentile$) RPOSST$_{\seqLabel}$ is presented in \cref{alg:cvar-rposst-seq}.

\section{Simultaneous-Move Model}\label{sec:sim-move-model}
We present a more in-depth description of the simultaneous move antagonist model which describes the $\text{RPOSST}_{SIM}$ as introduced in Section~\ref{section:rposst}. This description is complemented by pseudocode describing its workings in \cref{alg:rposst-sim}.

In this model, the antagonist does not observe which $\mTestCasesToSelect$-tuple of test cases, $\testCaseGroup$, is sampled from the protagonist's $\hat{\testCaseDistribution}^t_{\testCaseGroupDecisionLabel} \in \simplex^{\abs{\testCaseSet}^{\mTestCasesToSelect}}$ distribution, making the antagonist role more difficult.
The simultaneous move model corresponds to the policy testing use case where a new $\mTestCasesToSelect$-tuple of test cases is sampled independently for each test that is performed.
Effectively, the protagonist and antagonist choose $\testCaseGroup$ and
$\tuple*{\tuple*{
    \policy_{j_{(i)}},
    \testCaseDistribution_{(i)}
}}_{i = 1}^k$
respectively in a simultaneous fashion. In this model, the antagonist must choose a single list of tuples
$\tuple*{\tuple*{
    \policy_{j_{(i)}},
    \testCaseDistribution_{(i)}
}}_{i = 1}^k$
that will lead to a large loss across all of the $\mTestCasesToSelect$-tuples of test cases that the protagonist might choose, thereby preventing the antagonist from exploiting the lacking aspects of each individual $\mTestCasesToSelect$-tuple.

The protagonist in the simultaneous move model must carefully choose $\hat{\testCaseDistribution}^t_{\testCaseGroupDecisionLabel}$ and each $\mTestCasesToSelect$-tuple distribution,
$\subblock*{\hat{\testCaseDistribution}^t_{\testCaseGroup}}_{\testCaseGroup \in \testCaseSet^{\mTestCasesToSelect}}$,
to thwart the antagonist.
We organize the protagonist's actions into two sequential decisions: first choosing the $\mTestCasesToSelect$-tuple $\testCaseGroup$ and then choosing $\hat{\testCaseDistribution}^t_{\testCaseGroup}$ given $\testCaseGroup$.
We then use CFR$^+$ to refine both $\hat{\testCaseDistribution}^t_{\testCaseGroupDecisionLabel}$ and each $\hat{\testCaseDistribution}^t_{\testCaseGroup}$ after each round.

\begin{algorithm2e}
  \DontPrintSemicolon
  \textbf{\textit{Inputs:}} $\tuple*{k, N, T, \mTestCasesToSelect, \jointTestCaseDistributionPolicyUncertainty, \testCaseGroup^0, \ell}$
  \vspace{0.2em} \hrule \vspace{0.2em}

  \tcp{Initialize pseudoregrets.}
  $q^{1:0}_{\testCaseGroupDecisionLabel} \gets \zeros \in \reals^{\abs{\testCaseSet^{\mTestCasesToSelect}}}$

  $q^{1:0}_{\testCaseGroup} \gets \zeros \in \reals^{\mTestCasesToSelect + \abs{\testCaseGroup^0}}$ \textbf{for} $\testCaseGroup \in \testCaseSet^{\mTestCasesToSelect}$

  \tcp{Initialize average distributions.}
  $\hat{\testCaseDistribution}^{1:0}_{\testCaseGroupDecisionLabel} \gets
      \zeros \in \reals^{\abs{\testCaseSet^{\mTestCasesToSelect}}}$

  $\hat{\testCaseDistribution}^{1:0}_{\testCaseGroup} \gets
      \zeros \in \reals^{\mTestCasesToSelect + \abs{\testCaseGroup^0}}$ \textbf{for} $\testCaseGroup \in \testCaseSet^{\mTestCasesToSelect}$

  \For{$t \gets 1, \ldots, T$}{

\tcp{Sample antagonist actions.}
      $\policy_{j_i}, \testCaseDistribution_i \sim \jointTestCaseDistributionPolicyUncertainty$ \textbf{for} $i = 1 \ldots N$

      \tcp{Generate test case distributions.}
      $z^t_{\testCaseGroupDecisionLabel} \gets \ones^{\top} q^{1:t - 1}_{\testCaseGroupDecisionLabel}$

      $\hat{\testCaseDistribution}^t_{\testCaseGroupDecisionLabel} \gets q^{1:t - 1}_{\testCaseGroupDecisionLabel} / z^t_{\testCaseGroupDecisionLabel}$ \textbf{if} $z^t_{\testCaseGroupDecisionLabel} > 0$ \textbf{else} $\ones / \abs{\testCaseSet^{\mTestCasesToSelect}}$

      \For{$\testCaseGroup \in \testCaseSet^{\mTestCasesToSelect}$}{
        $z^t_{\testCaseGroup} \gets \ones^{\top} q^{1:t - 1}_{\testCaseGroup}$

        $\hat{\testCaseDistribution}^t_{\testCaseGroup} \gets q^{1:t - 1}_{\testCaseGroup} / z^t_{\testCaseGroup}$ \textbf{if} $z^t_{\testCaseGroup} > 0$ \textbf{else} $\ones / \mTestCasesToSelect$

        \tcp{Fill in zeros so that $\hat{\testCaseDistribution}^t_{\testCaseGroup} \in \simplex^{\abs{\testCaseSet}}$.}
        $\hat{\testCaseDistribution}^t_{\testCaseGroup}(x) \gets 0$ \textbf{for} $x \in \testCaseSet \setminus (\testCaseGroup \cup \testCaseGroup^0)$
      }

      \tcp{Evaluate the CFR$^+$ distributions.}
      $\ell_i \gets
          (\hat{\testCaseDistribution}^t_{\testCaseGroupDecisionLabel})^{\top}
          \subblock*{
              \ell(\hat{\testCaseDistribution}^t_{\testCaseGroup}; \policy_{j_i}, \testCaseDistribution_i)
          }_{\testCaseGroup \in \testCaseSet^{\mTestCasesToSelect}}
      $ \textbf{for} $i = 1, \ldots, N$

      \tcp{Sort to identify the worst $k$.}
      $\SortBy\subex*{
          \subblock*{ \tuple*{\testCaseDistribution_i, \policy_{j_i}} }_{i = 1}^N,
          \subblock*{ \ell_i }_{i = 1}^N
      }$

      \tcp{Update CFR$^+$.}
      \For{$\testCaseGroup \in \testCaseSet^{\mTestCasesToSelect}$}{
          $\ell_{\testCaseGroup, (i)} \gets \ell(\hat{\testCaseDistribution}^t_{\testCaseGroup}; \policy_{j_{(i)}}, \testCaseDistribution_{(i)})$
              \textbf{for} $i = 1, \ldots, k$

          $\cfv^t_{\testCaseGroup} \gets \frac{-1}{k} \sum_{i = 1}^k \frac{\partial \ell_{\testCaseGroup, (i)}}{\partial \hat{\testCaseDistribution}^t_{\testCaseGroup}}$

          $\regret^t_{\testCaseGroup} \gets \cfv^t_{\testCaseGroup} - (\hat{\testCaseDistribution}^t_{\testCaseGroup})^{\top} \cfv^t_{\testCaseGroup}$

          $q^{1:t}_{\testCaseGroup} \gets \ramp{q^{1:t - 1}_{\testCaseGroup} + \regret^t_{\testCaseGroup}}$
      }
      $\cfv^t_{\testCaseGroupDecisionLabel} \gets
          \frac{-1}{k} \sum_{i = 1}^k
              \subblock*{
                  \ell_{\testCaseGroup, (i)}
              }_{\testCaseGroup \in \testCaseSet^{\mTestCasesToSelect}}$

      $\regret^t_{\testCaseGroupDecisionLabel} \gets \cfv^t_{\testCaseGroupDecisionLabel} - (\hat{\testCaseDistribution}^t_{\testCaseGroupDecisionLabel})^{\top} \cfv^t_{\testCaseGroupDecisionLabel}$

      $q^{1:t}_{\testCaseGroupDecisionLabel} \gets \ramp{q^{1:t - 1}_{\testCaseGroupDecisionLabel} + \regret^t_{\testCaseGroupDecisionLabel}}$

      \tcp{Update average distributions.}
      $\hat{\testCaseDistribution}^{1:t}_{\testCaseGroupDecisionLabel} \gets
          \hat{\testCaseDistribution}^{1:t - 1}_{\testCaseGroupDecisionLabel}
          + t \hat{\testCaseDistribution}^t_{\testCaseGroupDecisionLabel}$

      $\hat{\testCaseDistribution}^{1:t}_{\testCaseGroup} \gets
          \hat{\testCaseDistribution}^{1:t - 1}_{\testCaseGroup}
          + t \hat{\testCaseDistribution}^t_{\testCaseGroupDecisionLabel}(\testCaseGroup) \hat{\testCaseDistribution}^t_{\testCaseGroup}
      $ \textbf{for} $\testCaseGroup \in \testCaseSet^{\mTestCasesToSelect}$
  }
  \Return
      $\tuple*{
          \dfrac{
              \hat{\testCaseDistribution}^{1:T}_{\testCaseGroupDecisionLabel}
          }{
              \ones^{\top} \hat{\testCaseDistribution}^{1:T}
          },
          \subblock*{
              \dfrac{
                  \hat{\testCaseDistribution}^{1:T}_{\testCaseGroup}
              }{
                  \ones^{\top} \hat{\testCaseDistribution}^{1:T}_{\testCaseGroup}
              }
          }_{\testCaseGroup \in \testCaseSet^{\mTestCasesToSelect}}
      }$
  \caption{RPOSST$_{\simLabel}$ (simultaneous model; CFR$^+$)}
  \label{alg:rposst-sim}
\end{algorithm2e}
 
Instantiating the percentile performance loss of \cref{eq:percentile-performance-loss} for the simultaneous move model, the RPOSST objective is,
\begin{align}
    \min_{
        \substack{
						\hat{\testCaseDistribution}_{\testCaseSet} \in \simplex^{\abs{\testCaseSet}^{\mTestCasesToSelect}}\\
            \subblock*{
								\hat{\testCaseDistribution}_{\testCaseGroup} \in \simplex^{\mTestCasesToSelect}
            }_{\testCaseGroup \in \testCaseSet^{\mTestCasesToSelect}}
        }
    }
    \inf_{\integrableFn \in \IntegrableFnSet}
        \hspace{-1.5em}
        \underset{
            \substack{
                \percentile \in [0, 1]\\
                \Prob \subblock*{
                    \E_{
												\testCaseGroup \sim \hat{\testCaseDistribution}_{\testCaseSet}
                    }\subblock*{
												\ell(\hat{\testCaseDistribution}_{\testCaseGroup}; \testCaseDistribution, \policy_j)
                    } \le \integrableFn(\eta)
                } \ge \percentile
            }
        }{\int}
            \hspace{-4.5em}
            \integrableFn(\eta)
            \probMeasure_{\kOfN}(d\eta),
    \label{eq:rposst-sim-objective}
\end{align}
where
$\testCaseDistribution, \policy_j \sim \jointTestCaseDistributionPolicyUncertainty$.

After (linearly) averaging the protagonist's choices of $\hat{\testCaseDistribution}^t_{\testCaseGroupDecisionLabel}$ and $\subblock*{\hat{\testCaseDistribution}^t_{\testCaseGroup}}_{\testCaseGroup \in \testCaseSet^{\mTestCasesToSelect}}$ across each round, \cref{alg:rposst-sim} returns the average distributions $\bar{\hat{\testCaseDistribution}}^T_{\testCaseGroupDecisionLabel}$
and
$\subblock*{\bar{\hat{\testCaseDistribution}}^T_{\testCaseGroup}}_{\testCaseGroup \in \testCaseSet^{\mTestCasesToSelect}}$.

The simultaneous-move model can be made deterministic using a CVaR measure in the same way as the sequential-move model.
If we fix the ratio $\nicefrac{k}{N}$ and allow $N \to \infty$, the $k$-of-$N$ robustness measure converges toward the CVaR measure at the $\nicefrac{k}{N}$ fractile.
Furthermore, if our the distribution characterizing our uncertainty, $\jointTestCaseDistributionPolicyUncertainty$,
is over a discrete set of manageable size, then we can run RPOSST on CVaR robustness measures.
In RPOSST$_{\simLabel}$, the lowest loss test case distributions across all rounds can also be tracked instead of averaging all of the distributions.

\section{Experimental Details}\label{sec:experimental-details}
\begin{table*}[t]
	\caption{Approximate amount of time required to run $T = 500$ rounds of CVaR($\cvarPercentile$) RPOSST$_{\seqLabel}$ in each domain. Runtimes are similar across both variants in each domain and across holdout policy set sizes.}
	\label{tab:experiment_runtimes}
	\centering
    \begin{tabular}[c]{l|l}
        \hline
        domain & runtime / seed\\
        \hline
        \hline
        Racing Arrows & $\sim 2$ minutes\\
        \hline
        ACPC & $\sim 10$ seconds\\
        \hline
        GT & $\sim 90$ seconds\\
        \hline
    \end{tabular}
\end{table*}

In this section we provide further details on some of the experimental setups used in \cref{section:experiments}.

All CVaR($\cvarPercentile$) RPOSST$_{\seqLabel}$ procedures were run on a 16 core AMD$^{\text{\textregistered}}$ Ryzen 7 5800h CPU with 30.7 GiB of memory.
See \cref{tab:experiment_runtimes} for the time required to run CVaR($\cvarPercentile$) RPOSST$_{\seqLabel}$ on each domain.

\subsection{Racing Arrows}
Racing Arrows is a two-player, zero-sum, one-shot, continuous action game that replicates simple aspects of a passing scenario in a race featuring a "leader" player and faster "follower" player.
The goal of the follower is to pass the leader while the goal of the leader
is to block the follower.

Both players privately choose an angle in the half-circle between 0 and pi for their arrow.
The speed of each player is represented as the length of their arrow. The leader and follower are assigned a speed
according to their roles, where the leader's speed of $0.8$ is slightly slower than
the follower's speed of $1$ to give the follower a chance to pass.
The distance a player travels is the height of their arrow, \ie/,
$\text{speed} \cdot \sin(\text{angle})$.

The follower is blocked and the leader wins if the difference between the two arrows is below $\nicefrac{\pi}{10}$, that is, the leader is close enough to block the follower.
If the follower is not blocked, then the player who traveled the farthest wins.
Players receive $+1$ for a win, $0$ for a loss, or $0.5$ if they travel exactly the same distance (these payoffs sum to the constant $+1$, which is isomorphic to true zero-sum payoffs).

\subsection{Annual Computer Poker Competition}

The Annual Computer Poker Competition (ACPC) was run to test autonomous poker playing agents from 2006 to 2017.
The logs of play are freely available online.\footnote{\url{http://www.computerpokercompetition.org/downloads/competitions}}
Typically, these competitions are Texas hold'em variants: two-player limit, two-player no-limit, and 3-player limit, where ``limit'' and ``no-limit'' indicates whether players are only allowed to bet in fixed increments or if they can bet any number of chips from their current stack, respectively.
Chip stacks reset to their initial sizes after every hand (Doyle's game) so that players can be evaluated on their average one-hand performance across deck shufflings and seat positions.

To reduce variance, hands are played in duplicate, which means that the same deck order is played out multiple times so that each player has a turn playing with the same hands.
For example, if Alice in seat 1 is dealt the ace and king of spades and Bob in seat 2 is dealt the 2 and 7 of hearts in one hand, then Alice and Bob will also play the same hand in opposite positions, where Bob is dealt the ace and king of spades in seat 1, and Alice is dealt the 2 and 7 of hearts.
Alice's duplicate score is then the number of chips she wins over what Bob won in the same position, averaged across both positions.

Our experiments use duplicate score data, \ie/, a test case result here is a duplicate score between two agents, from the 2012 two-player limit and the 2017 two-player no-limit events.

\subsection{\GranTurismo7/}\label{sec:gt7}

\begin{figure}[!h]
	\centering
	\begin{subfigure}{0.55\linewidth}
			\includegraphics[width=\columnwidth, keepaspectratio]{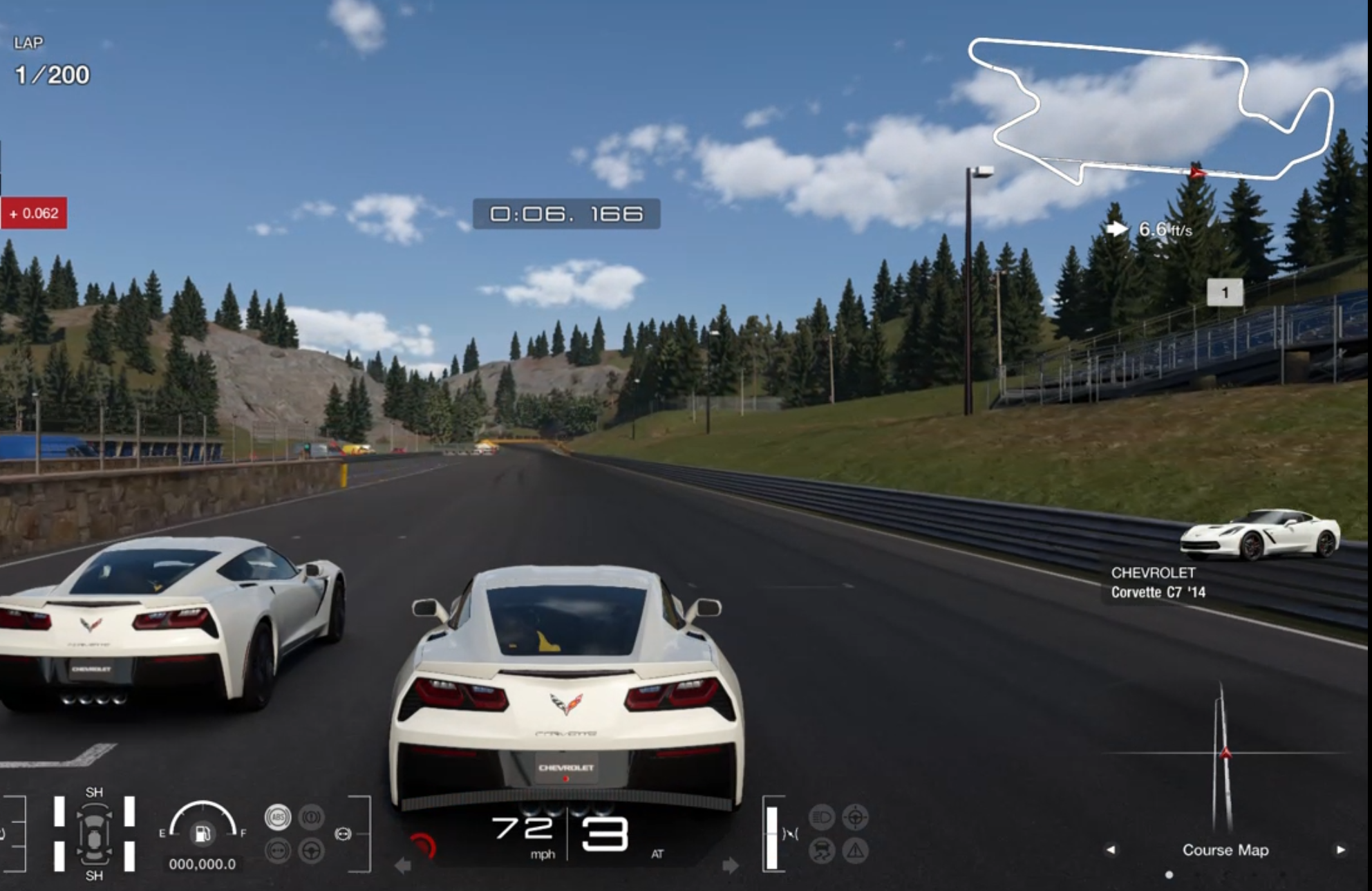}
			\caption{}
			\label{fig:trial-mountain}
	\end{subfigure}
	\hfill
	\begin{subfigure}{0.42\linewidth}
			\includegraphics[width=\linewidth, clip, trim=4em 20em 4em 10em, keepaspectratio]{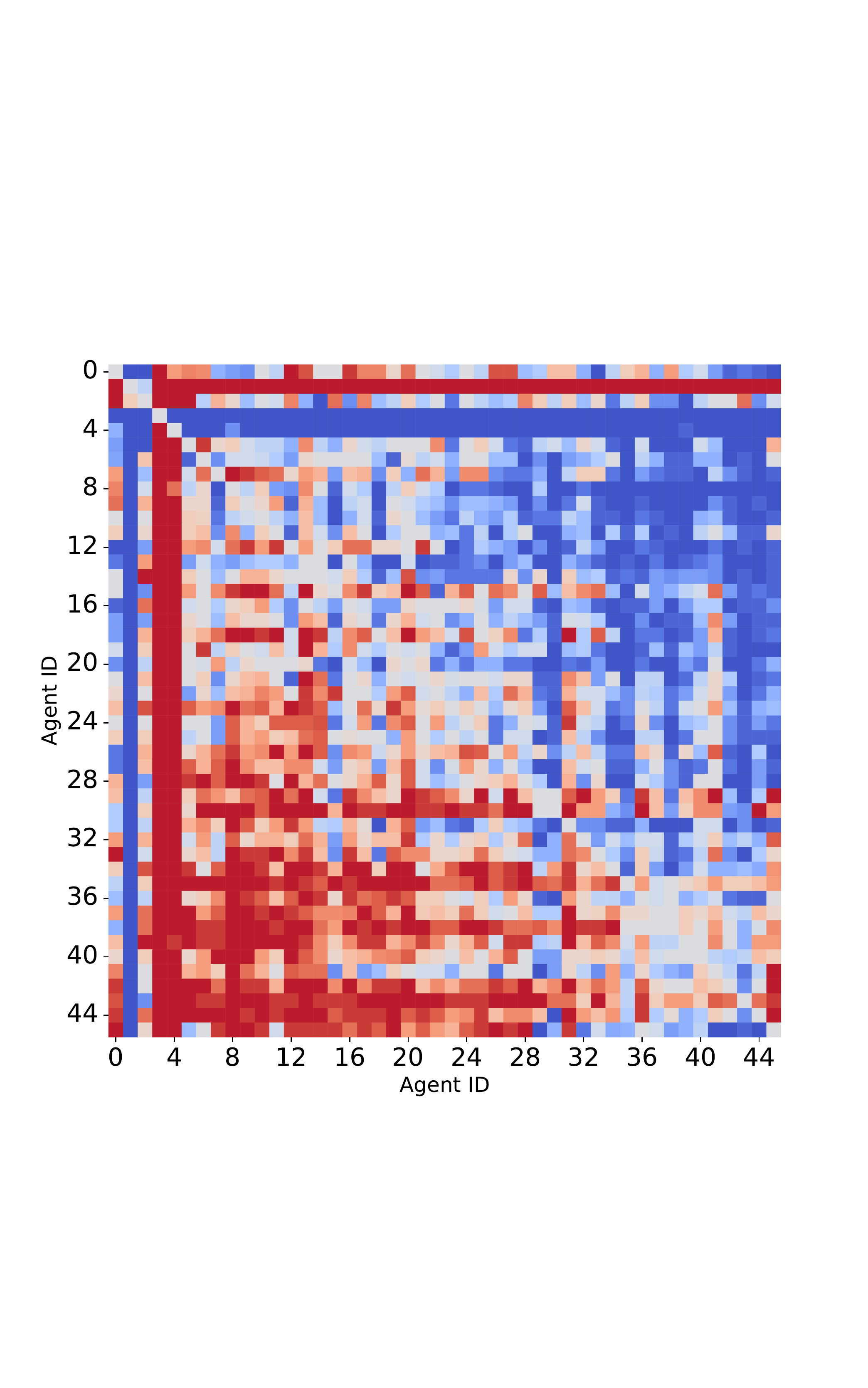}
			\caption{}
			\label{figure:gt_winrate_matrix}
	\end{subfigure}
 \caption{Figure~\ref{fig:trial-mountain} shows a screenshot of two RL agents racing at the Trial Mountain racetrack. The layout of the track can be seen in the top right. Figure~\ref{figure:gt_winrate_matrix} shows the result matrix for the zero-sum experiment. Blue / red colors indicate positive / negative winrates from the point of view of the column player. Agents 0-2 correspond to the different built-in AIs, with the remaining agents being the trained RL agents sorted according to skill. Diagonal values denote an agent playing against itself, which we artificially set to 50\%.}
\end{figure}

Our \GranTurismo7/ experiments were conducted using the \GranTurismo7/ racing simulator.  Previous versions of the \GranTurismo7/ franchise have been used to exhibit  reinforcement learning results~\citep{fuchs2021super,song2021autonomous} including outracing top human drivers~\citep{wurman2022outracing}. Note our focus was not on agent training but rather the problem of selecting the best policy for a deployment, so for training we used the same training parameters reported by by Wurman et al. except for changes to training scenarios to match the track and car combination chosen for this experiment, training only for one-on-one competition, and utilizing a version of self-play to simplify the training process.

The experiment was conducted at the Trial Mountain racetrack (see \cref{fig:trial-mountain}) with the RL policy (and any RL-trained opponent policies) driving a Chevrolet Corvette C7 Stingray '14 using Sport Hard tires. The track and car were chosen because the long straightaways and sharp turns at Trial Mountain led to competitive racing among various RL policies as there are many different areas of the track where passes can occur and the long straightaways allow the agent to use the slipstream of the other car to stay in touch with the car in front.

From a single one-on-one training run we evaluated checkpoints from epochs 5, 200, and then every 75 epochs between epoch 1000 and 4000 for a total of 43 checkpoints. We also evaluated 3 built-in AI agent using cars and tires that made them competitive with the RL agents. Overall we evaluated 46 policies, each of which was considered as a candidate deployment policy or an opponent in a test case.

To create the result matrix shown in \cref{figure:gt_winrate_matrix}, each race was run 20 times with a side-by-side standstill start with the candidate and opponent policies swapping sides half the time to enforce symmetry. An agent would obtain 1 or 0 for winning or losing the race respectively. The diagonal denoting a race between an agent against itself was filled in with $0.5$ entries. As a second experiment on \GranTurismo7/ for a non-zero sum game, using the sportsmanship rule mentioned in \cref{section:experiments} we recomputed the result matrix from \cref{figure:gt_winrate_matrix} so as to penalize trajectories where any car collisions had happened, giving both agents a payoff of $-1$. We remove the entries in the result matrix related to built-in AIs as they are highly collision averse and therefore the sportsmanship constraints would not change their test results, reducing the test case pool size to 43.

\subsection{Supplemental Experimental Results}\label{sec:supplemental-experimental-results}

\begin{figure*}[t]
    \begin{subfigure}[t]{0.32\linewidth}
        \includegraphics[width=\linewidth]{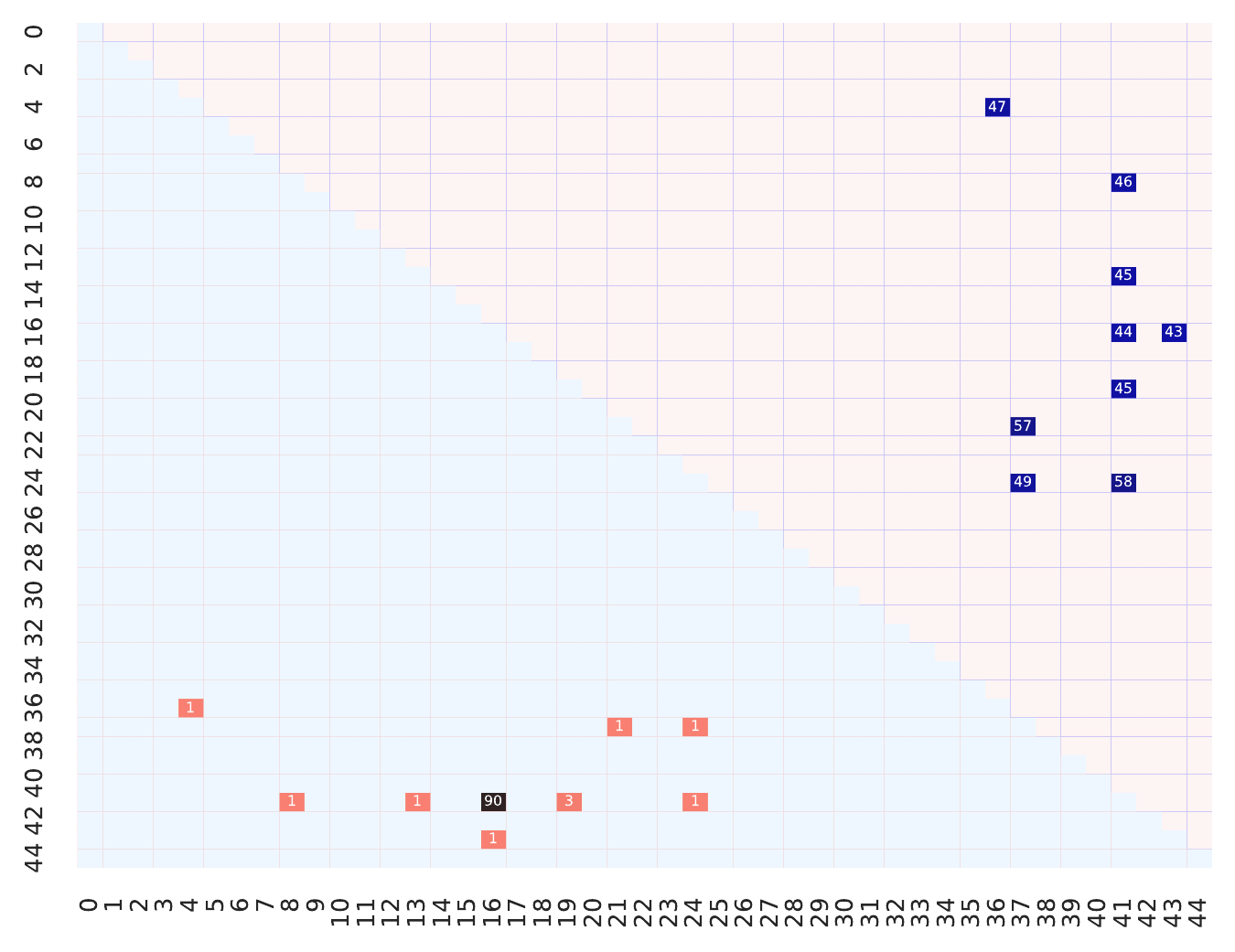}
        \caption{\kOfN}
        \label{fig:frequency_matrix_k_of_n}
    \end{subfigure}\hfill \begin{subfigure}[t]{0.32\linewidth}
        \includegraphics[width=\linewidth]{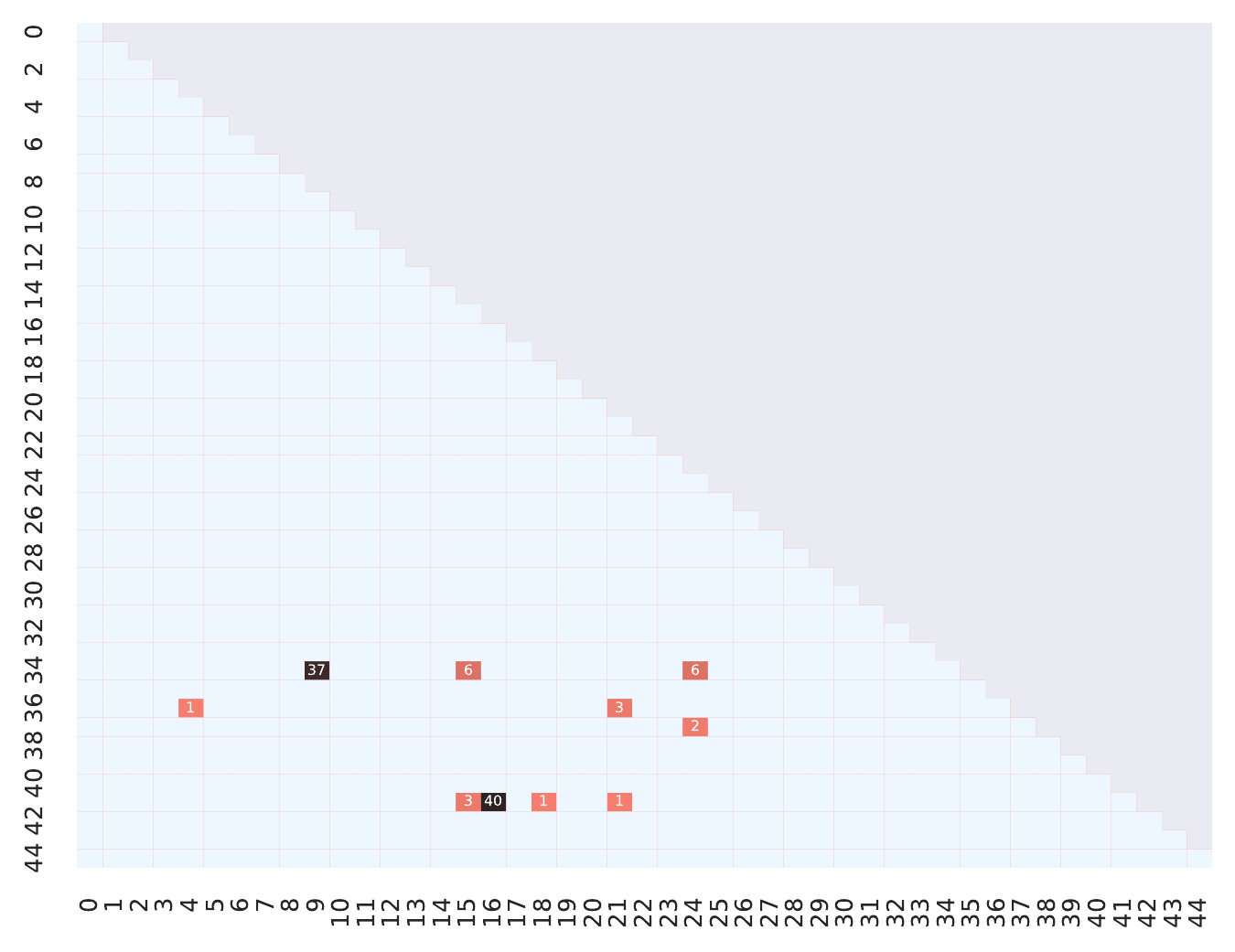}
        \caption{Minimax uniform}
    \end{subfigure}\hfill \begin{subfigure}[t]{0.32\linewidth}
        \includegraphics[width=\linewidth]{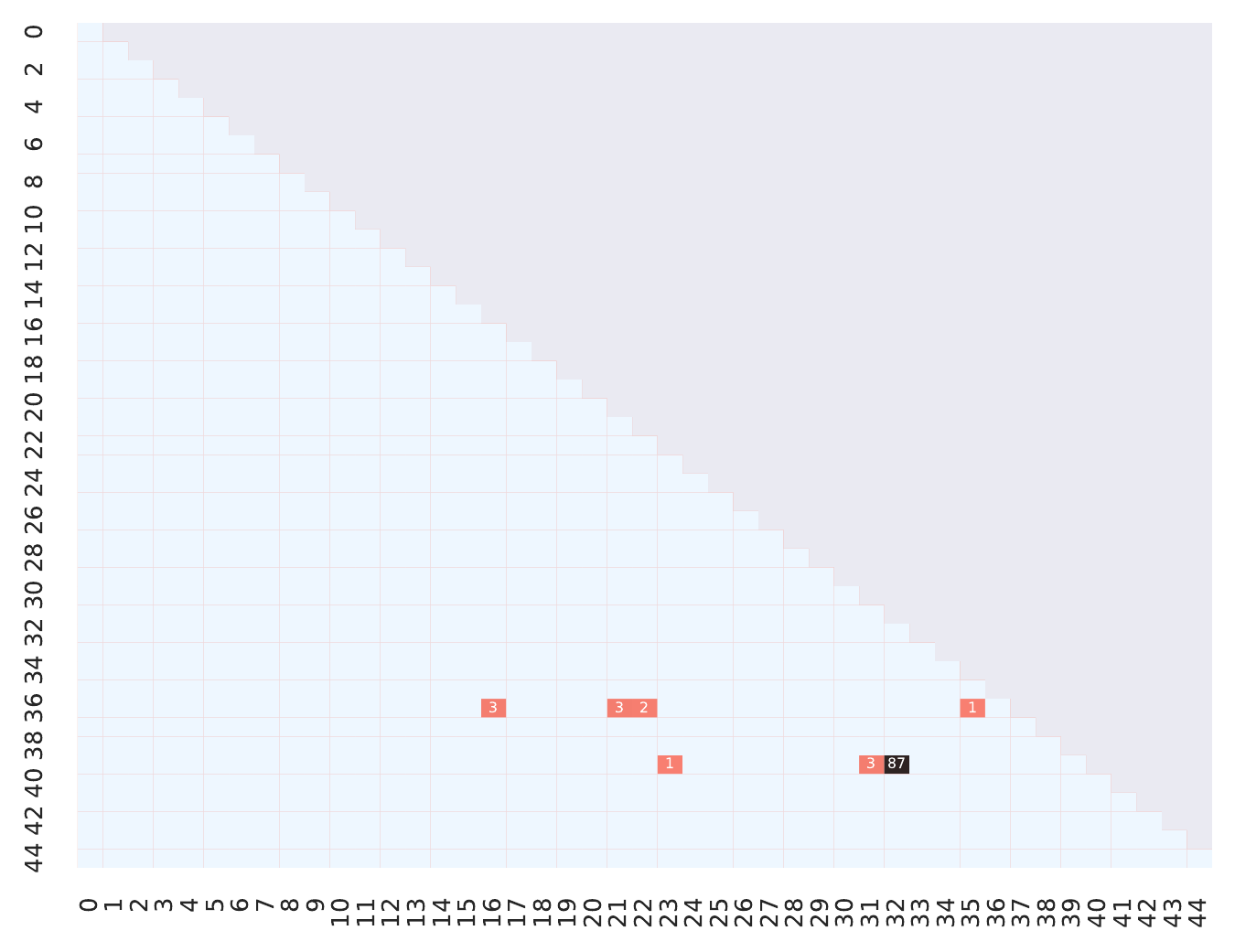}
        \caption{Iterative minimax}
    \end{subfigure}\vspace{0.5em}
    \begin{subfigure}[t]{0.32\linewidth}
        \includegraphics[width=\linewidth]{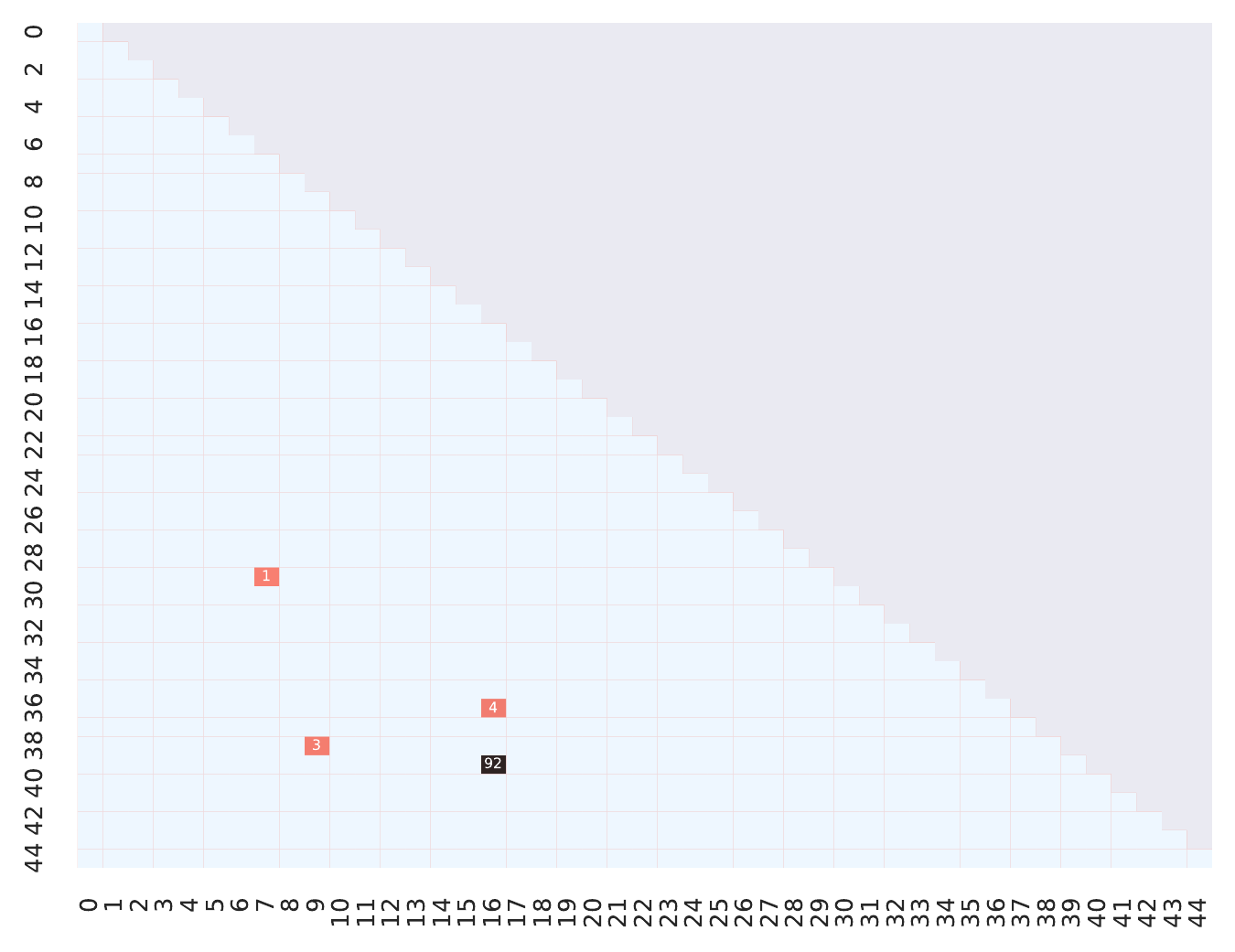}
        \caption{Minimax(TNP) uniform}
    \end{subfigure}\hfill \begin{subfigure}[t]{0.32\linewidth}
        \includegraphics[width=\linewidth]{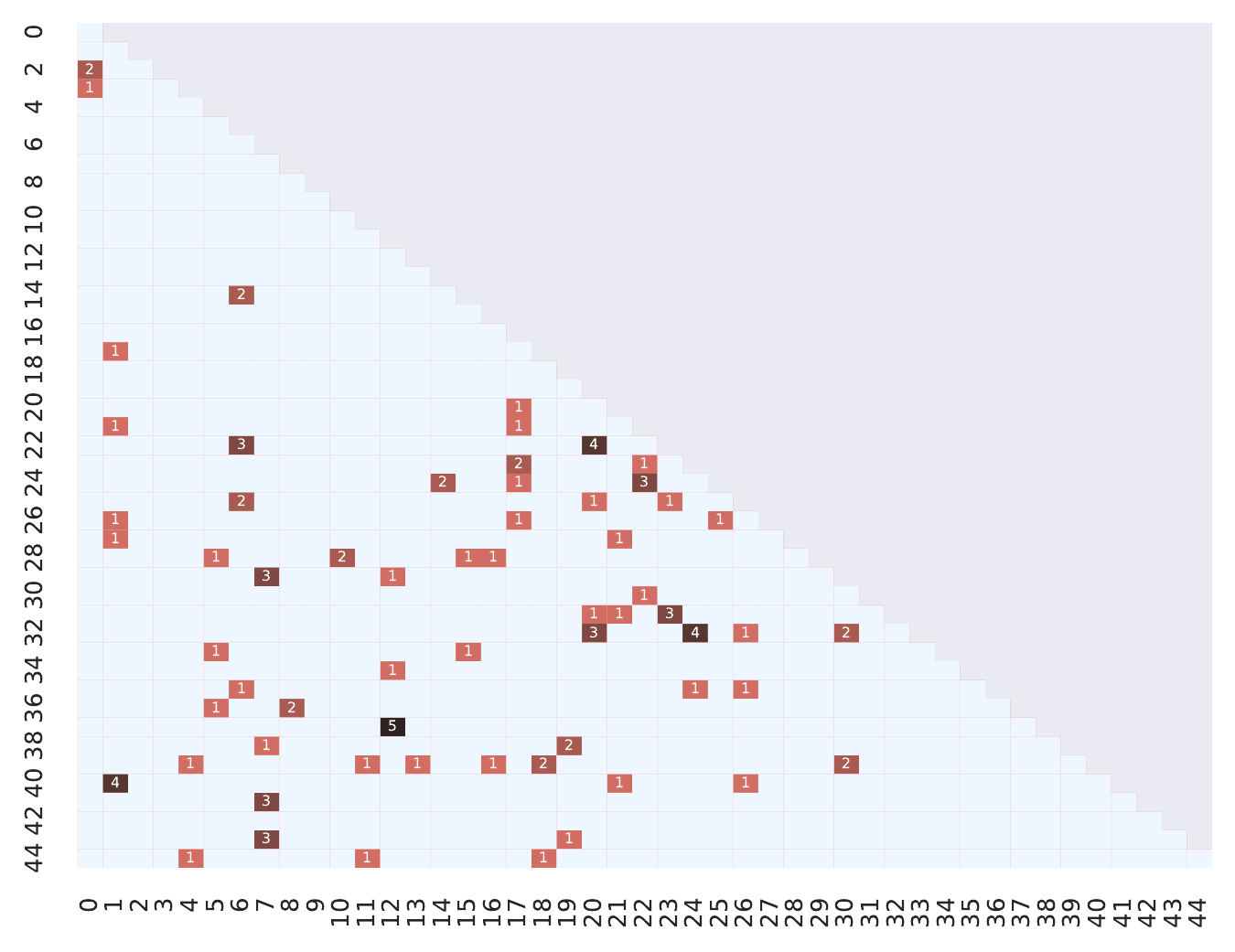}
        \caption{Miniaverage uniform}
    \end{subfigure}\hfill \begin{subfigure}[t]{0.32\linewidth}
        \includegraphics[width=\linewidth]{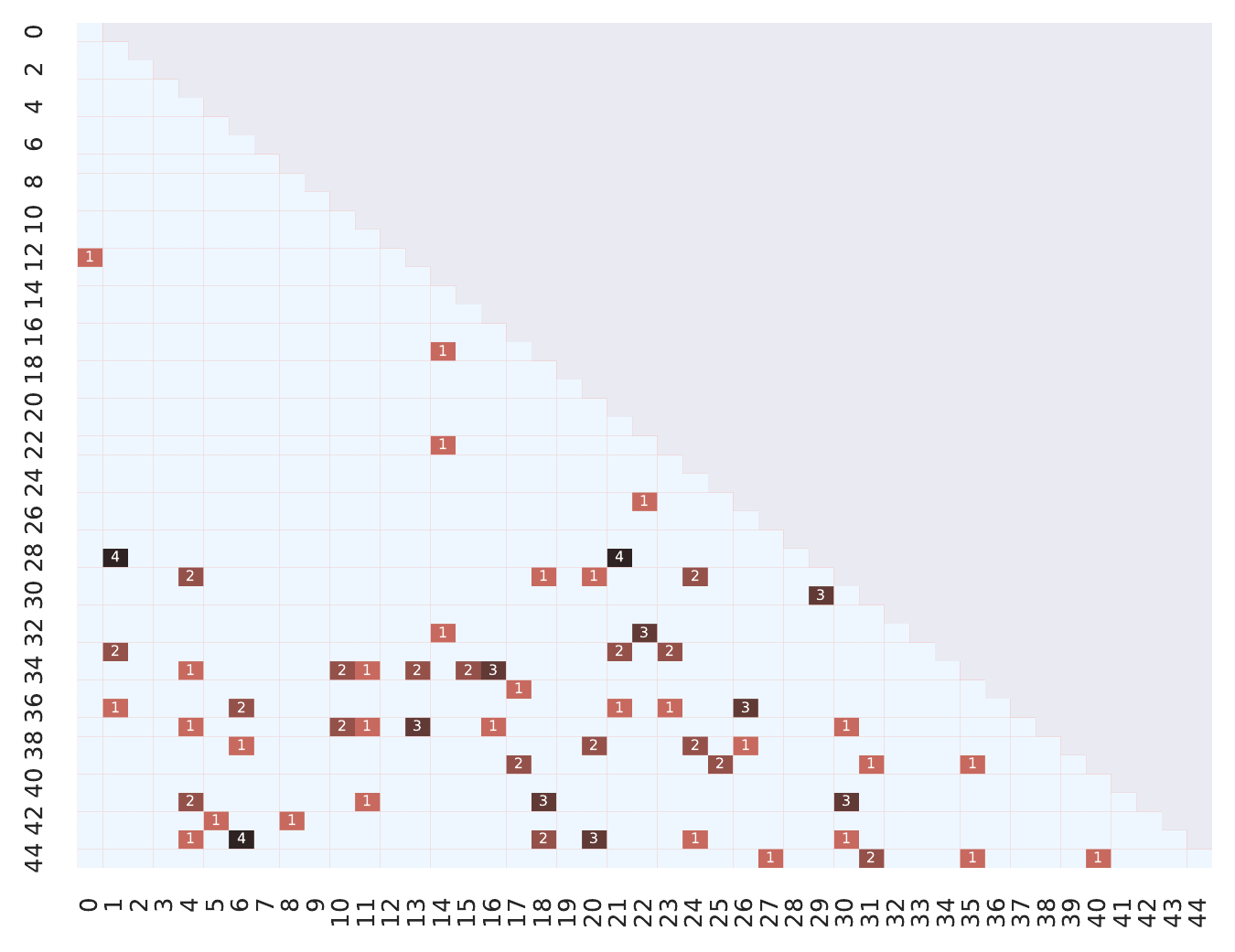}
        \caption{Minimax(TTD) uniform}
    \end{subfigure}\caption{Triangular matrices denoting frequencies of test case pairs choosen by all 6 algorithmic ablations over 100 runs of the winrate GT experiment on a holdout of size 20\%. This visualization is possible because only 2 test cases where chosen as output. The upper triangular matrix from \cref{fig:frequency_matrix_k_of_n} denotes average probability mass given to test case $i$. All other algorithms are limited to uniformly mixing over test cases so the upper triangular matrix is omitted for clarity.}
    \label{fig:frequency_matrices}
\end{figure*}
\begin{table*}[t]
    \caption{Top two test case pairs and corresponding selection frequencies chosen by each algorithm over the 100 seeds in the large GT experiment.}
    \label{tab:top_picks_holdout_20}
    \centering
    \begin{tabular}[c]{|l|cc|}
        \hline
        Algorithm & Pairs & Frequency \\
        \hline
\multirow{2}*{RPOSST$_{\seqLabel}$} & (41, 16) & 90 \\
                                            & (41, 19) & 3 \\
        \hline
        \multirow{2}*{Minimax uniform} & (41, 16) & 40 \\
                                        & (34, 9) & 37 \\
        \hline
\multirow{2}*{Iterative minimax} & (39, 32) & 87 \\
                                        & (39, 31) & 3 \\
        \hline
\multirow{2}*{Minimax(TNP) uniform} & (40, 16) & 92 \\
                                            & (36, 16) & 4 \\
        \hline
\multirow{2}*{Miniaverage uniform} & (37, 12) & 5 \\
                                            & (40, 1) & 4 \\
        \hline
\multirow{2}*{Minimax(TTD) uniform} & (43, 6) & 4 \\
                                            & (28, 1) & 4 \\
        \hline
    \end{tabular}
\end{table*}

\Cref{fig:test-score-error} in \cref{section:experiments} analyses the quantitative performance of RPOSST and its algorithmic ablations with respect to measuring test scores on a holdout set of unseen candidate deployment policies. We complement that analysis with a qualitative study of behaviors exhibited by the algorithms using the large GT experiment with holdout of size 20 as a representative example. We are interested in examining (1) how deterministic each algorithm's output is with respect to the selection of test case pairs and (2) whether different algorithms choose the same test-cases.

The lower triangular matrices from \cref{fig:frequency_matrices} show the frequency at which test case pairs were chosen over the 100 seeds. The top 2 most selected test case pairs for each algorithm are presented in \cref{tab:top_picks_holdout_20}. We observe that RPOSST, alongside Iterative minimax and Minimax(TNP) uniform are very deterministic algorithms, favouring the selection of the same test case pair over 90\%, 87\% and 92\% of the seeds respectively. We deem this a desirable property, as variance  in evaluation scenarios is undesirable because it can hamper interpretability and reproducibility. In contrast, Minimax uniform exhibits a bimodal choice. The remaining algorithms feature a very high variance in their choice of test case pairs, with their most chosen test case pair being selected 5\% of the time, spreading selection widely.

From \cref{tab:top_picks_holdout_20}, test case 16 is heavily favoured by half of the algorithms (RPOSST$_{\seqLabel}$, Minimax uniform and Minimax(TNP) uniform), followed to a lesser extent by test case 41. This indicates that all these algorithms find useful structure in such pairs of agents.

In \cref{fig:racing-arrows-f-50,fig:racing-arrows-l-50,fig:gt-large-matrix,fig:gt-non-zero-sum,fig:acpc2012,fig:acpc2017}, we show the performance of RPOSST$_{\seqLabel}$ and baselines in each domain across test sizes ($\mTestCasesToSelect \in \set{1, 2, 3}$) and holdout proportions ($20\%$, $40\%$, and $60\%$).
\Cref{fig:racing-arrows-l-500} shows the results for the 500 policy Racing Arrows experiment where the leader policies are treated as test cases.

\begin{figure*}[t]
\begin{minipage}[t]{0.32\linewidth}
        \includegraphics[width=\linewidth]{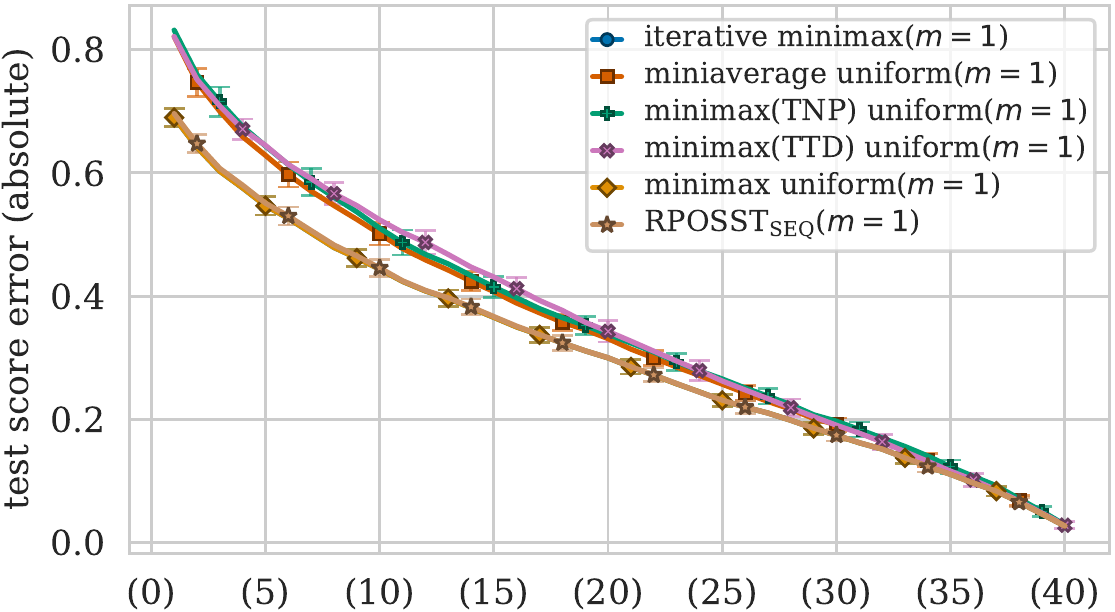}
    \end{minipage}\hfill \begin{minipage}[t]{0.32\linewidth}
        \includegraphics[width=\linewidth]{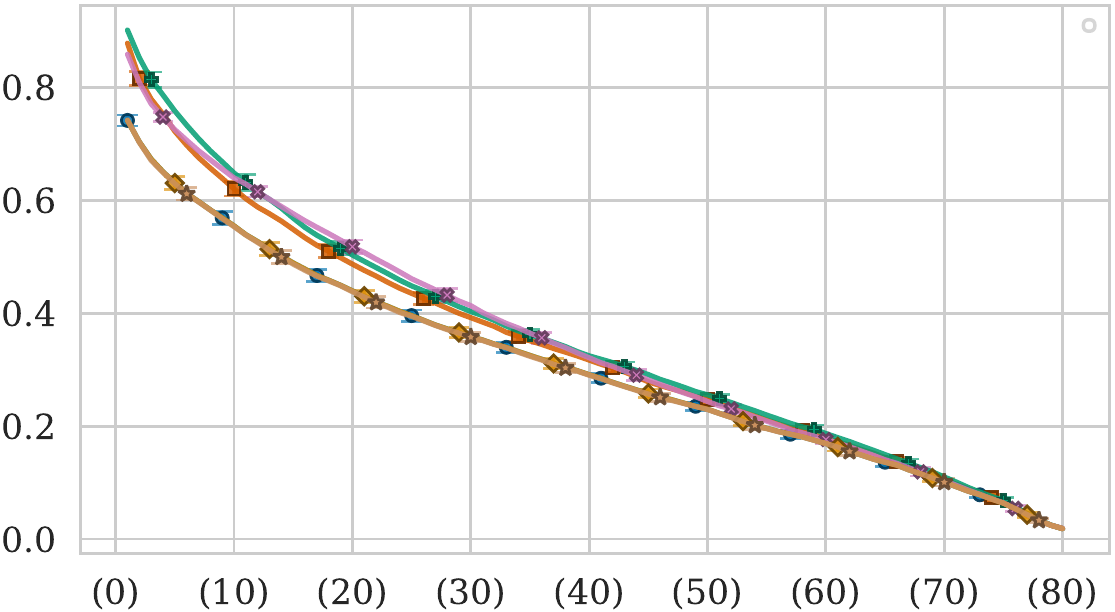}
    \end{minipage}\hfill \begin{minipage}[t]{0.32\linewidth}
        \includegraphics[width=\linewidth]{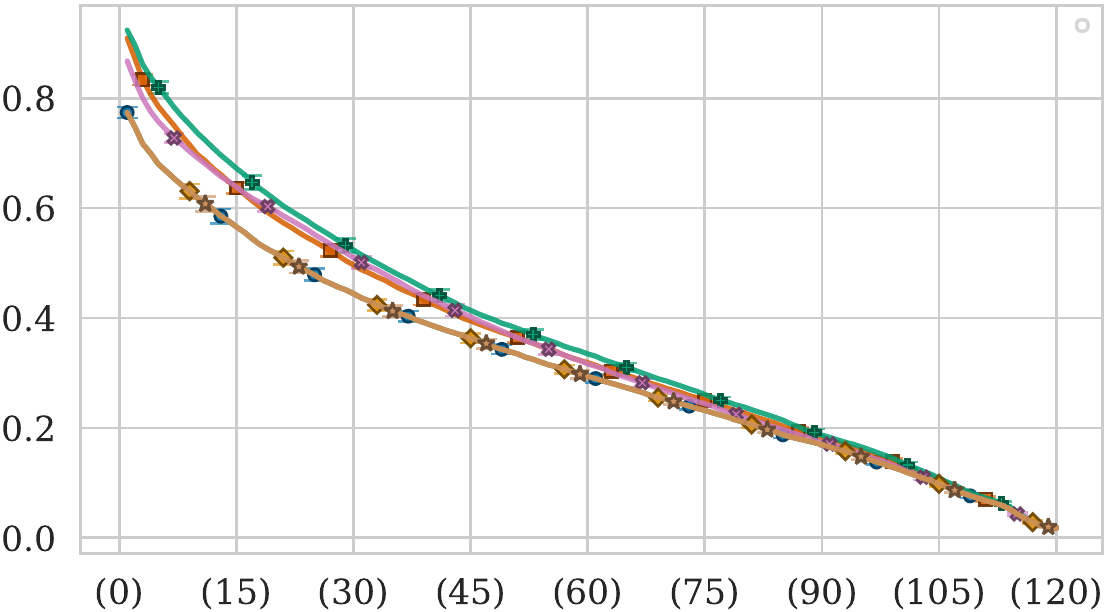}
    \end{minipage}\vspace{0.5em}

\begin{minipage}[t]{0.32\linewidth}
        \includegraphics[width=\linewidth]{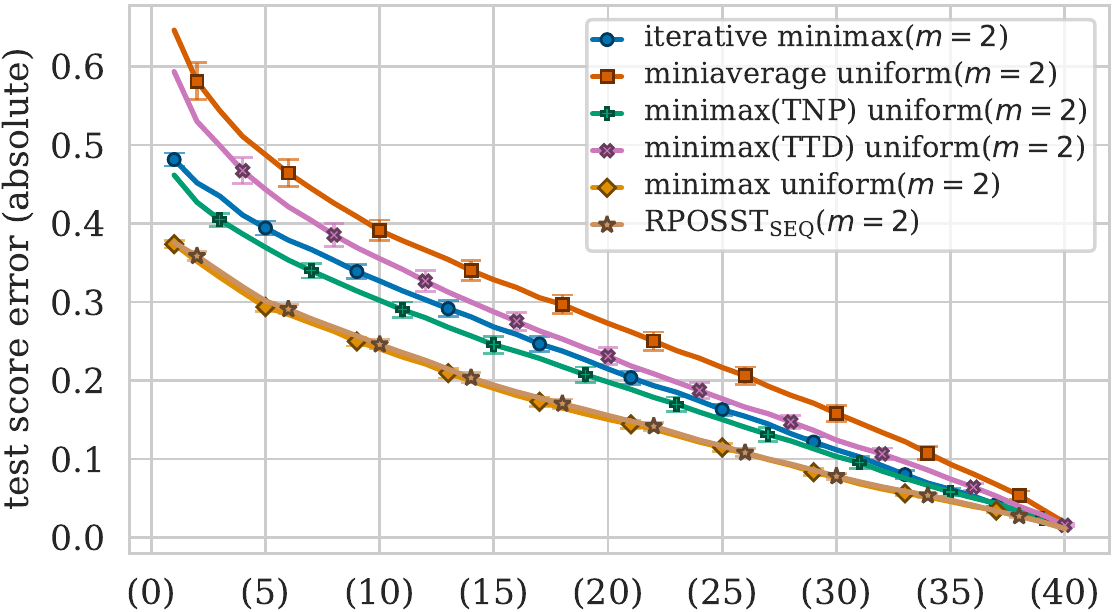}
    \end{minipage}\hfill \begin{minipage}[t]{0.32\linewidth}
        \includegraphics[width=\linewidth]{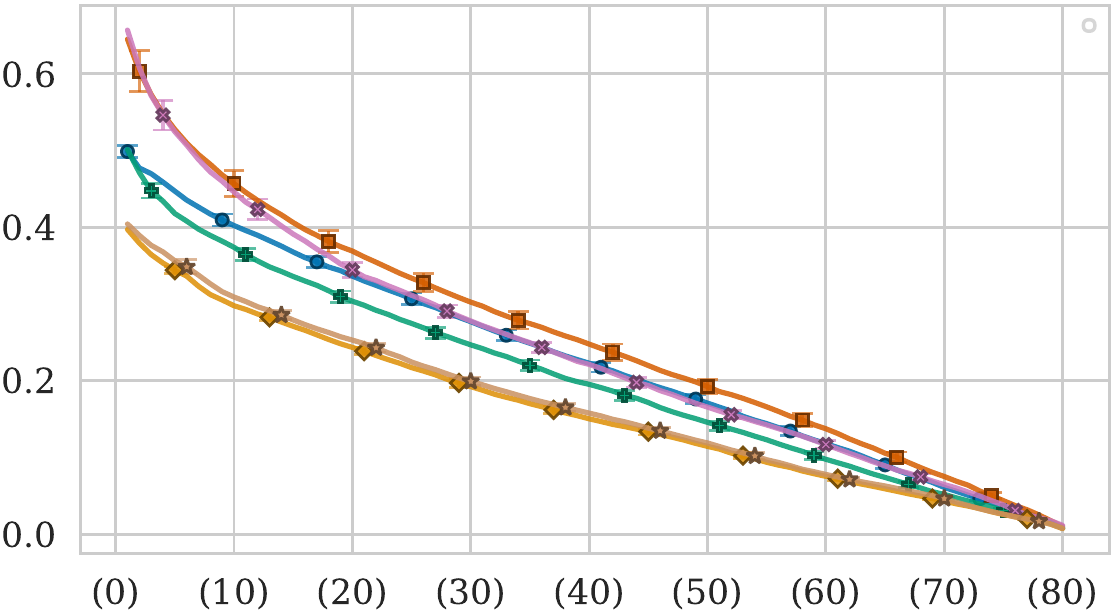}
    \end{minipage}\hfill \begin{minipage}[t]{0.32\linewidth}
        \includegraphics[width=\linewidth]{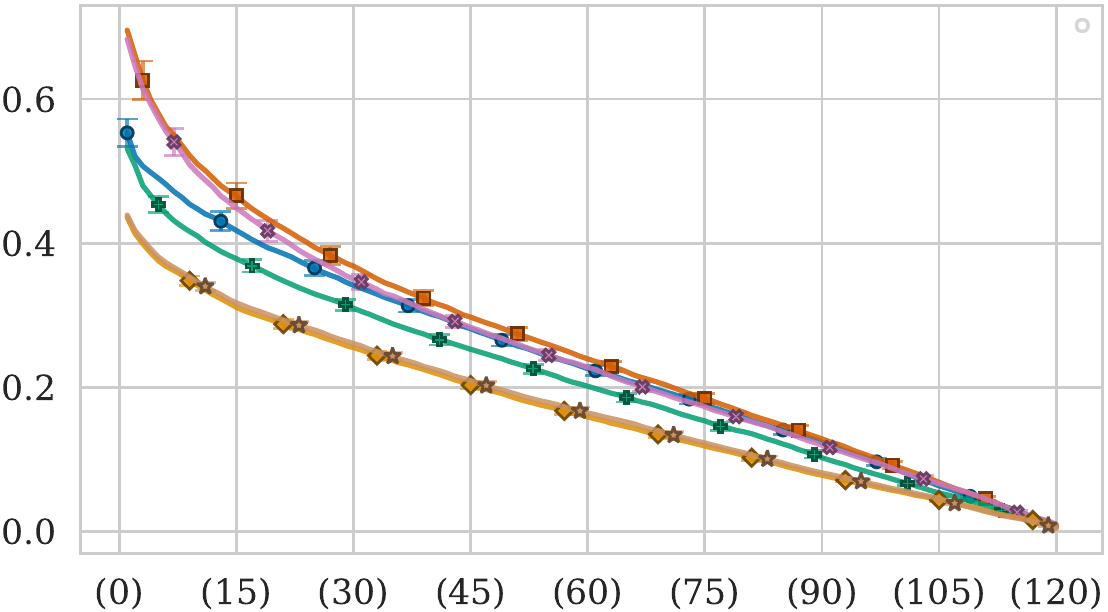}
    \end{minipage}\vspace{0.5em}

\begin{minipage}[t]{0.32\linewidth}
        \includegraphics[width=\linewidth]{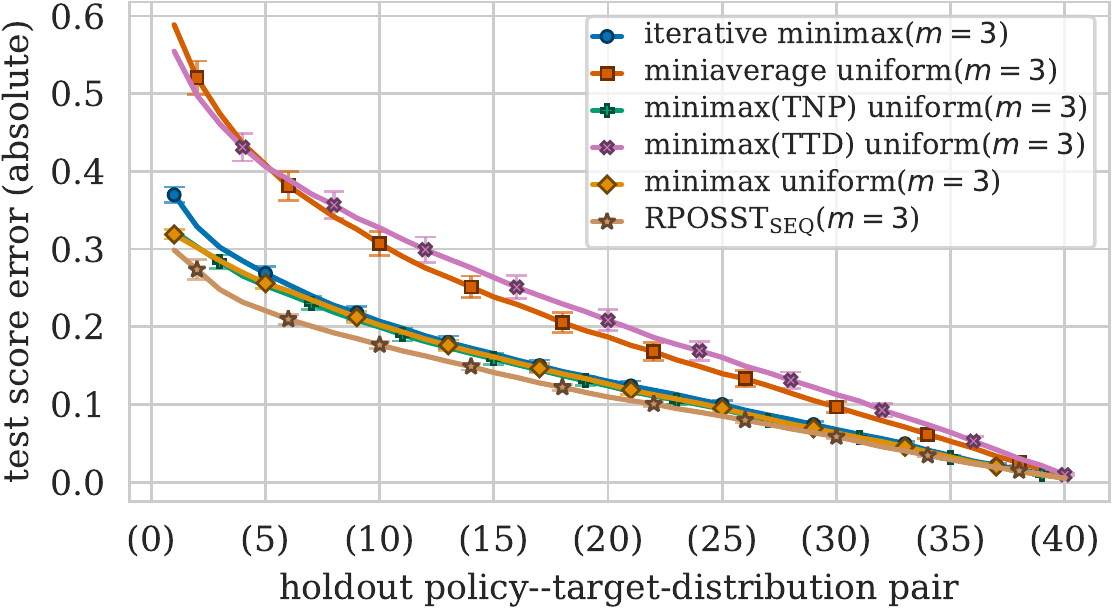}
    \end{minipage}\hfill \begin{minipage}[t]{0.32\linewidth}
        \includegraphics[width=\linewidth]{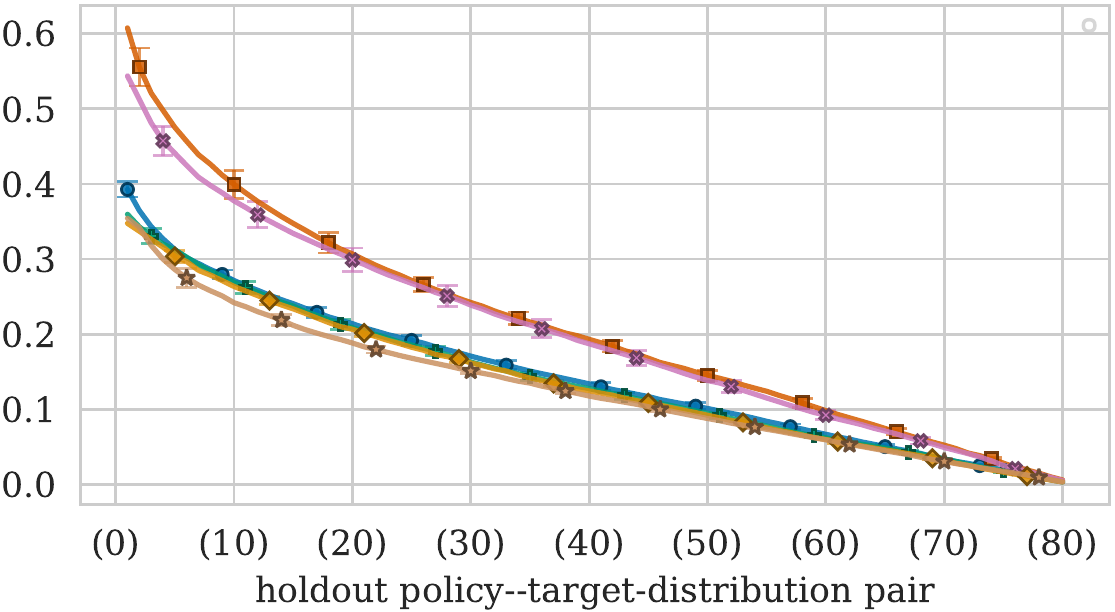}
    \end{minipage}\hfill \begin{minipage}[t]{0.32\linewidth}
        \includegraphics[width=\linewidth]{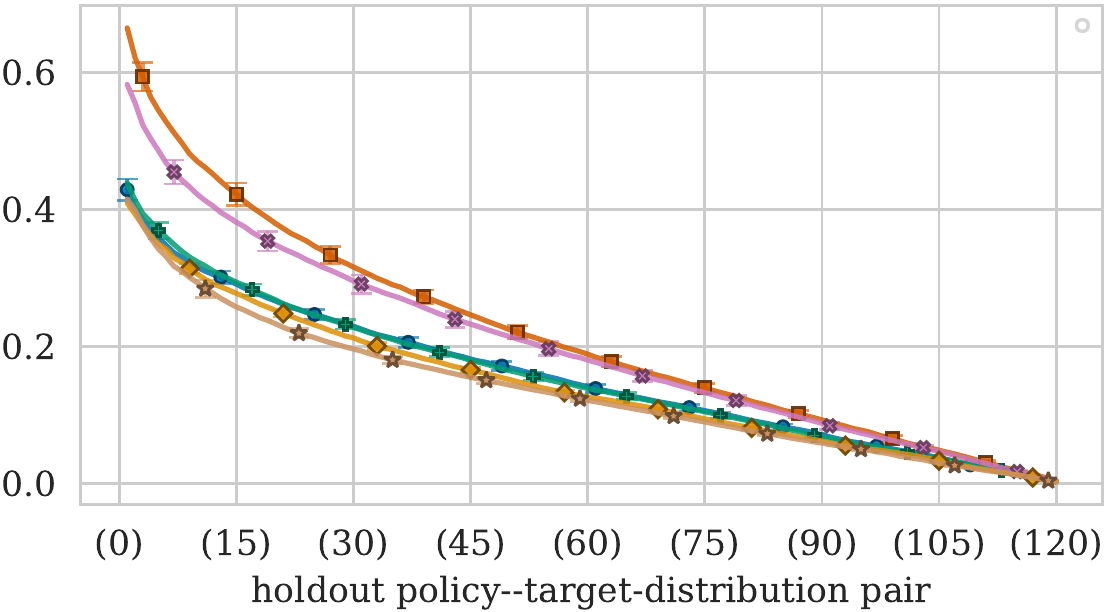}
    \end{minipage}\caption{Expected test score error (absolute difference) across holdout-policy--target-distribution pairs on Racing Arrows where test cases are 50 follower policies. Each row uses a different setting for the test size ($m = 1$ top, $m = 2$ middle, and $m = 3$ bottom) and each column uses a different holdout proportion ($20\%$ held out in the left column, $40\%$ middle, and $60\%$ right). $\numHoldoutReplicas$ sets of holdout policies were sampled. Holdout-policy--target-distribution pairs are sorted according to test score error. Each RPOSST$_{\seqLabel}$ instance was run for $500$ rounds ($T = 500$). Errorbars represent $95\%$ t-distribution confidence intervals.}
    \label{fig:racing-arrows-f-50}
\end{figure*}

\begin{figure*}[t]
\begin{minipage}[t]{0.32\linewidth}
        \includegraphics[width=\linewidth]{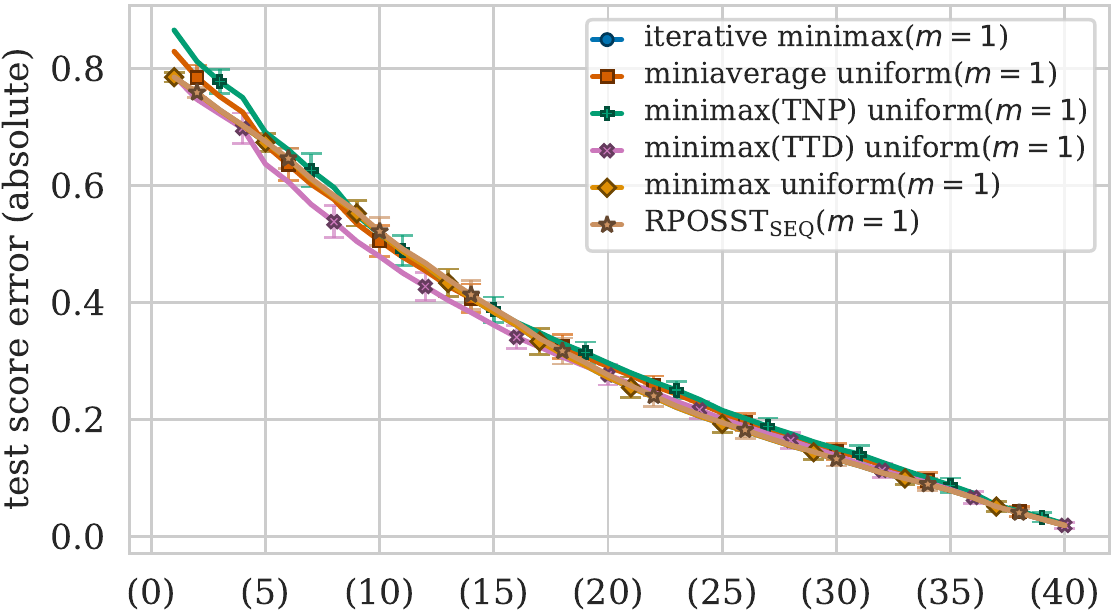}
    \end{minipage}\hfill \begin{minipage}[t]{0.32\linewidth}
        \includegraphics[width=\linewidth]{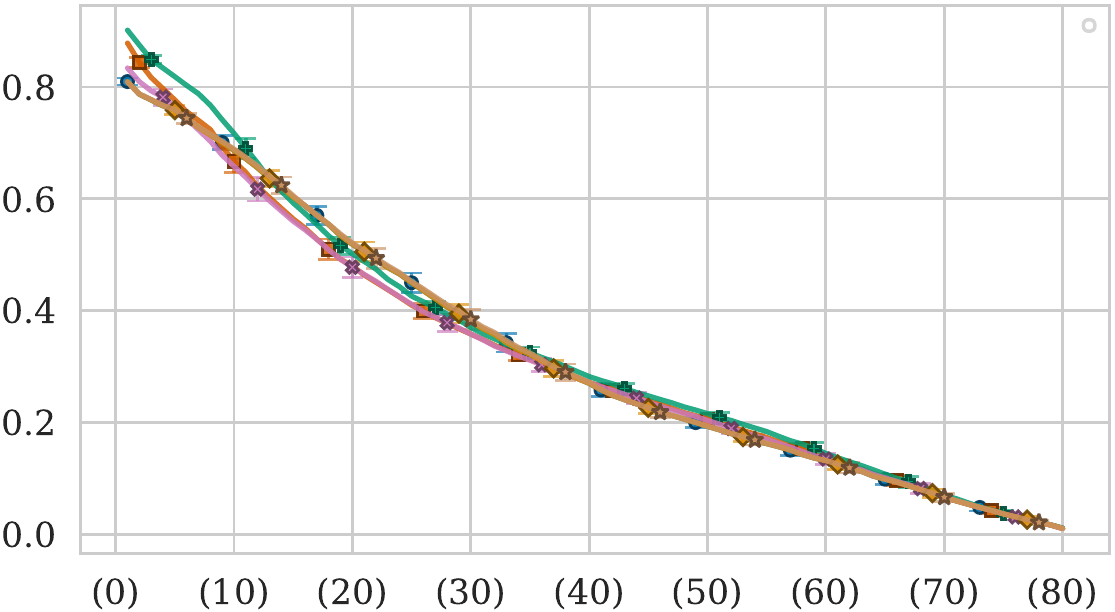}
    \end{minipage}\hfill \begin{minipage}[t]{0.32\linewidth}
        \includegraphics[width=\linewidth]{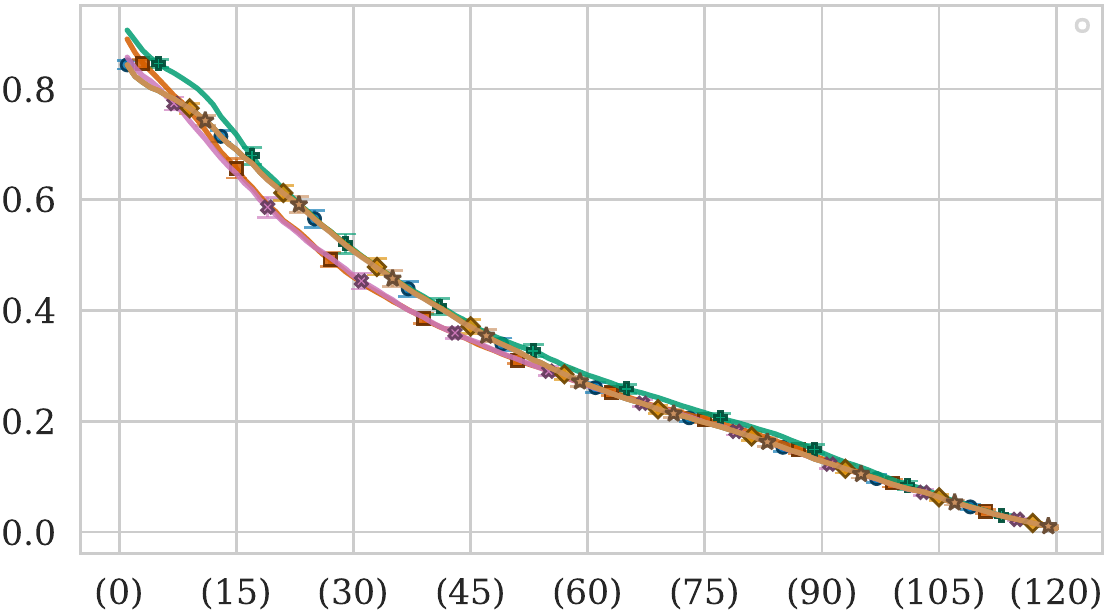}
    \end{minipage}\vspace{0.5em}

\begin{minipage}[t]{0.32\linewidth}
        \includegraphics[width=\linewidth]{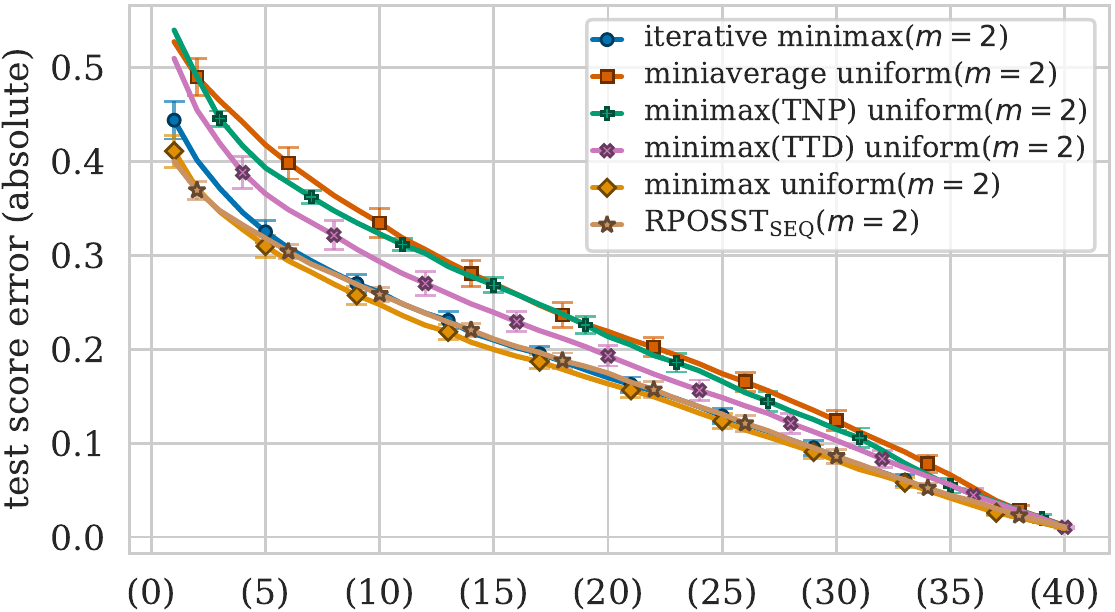}
    \end{minipage}\hfill \begin{minipage}[t]{0.32\linewidth}
        \includegraphics[width=\linewidth]{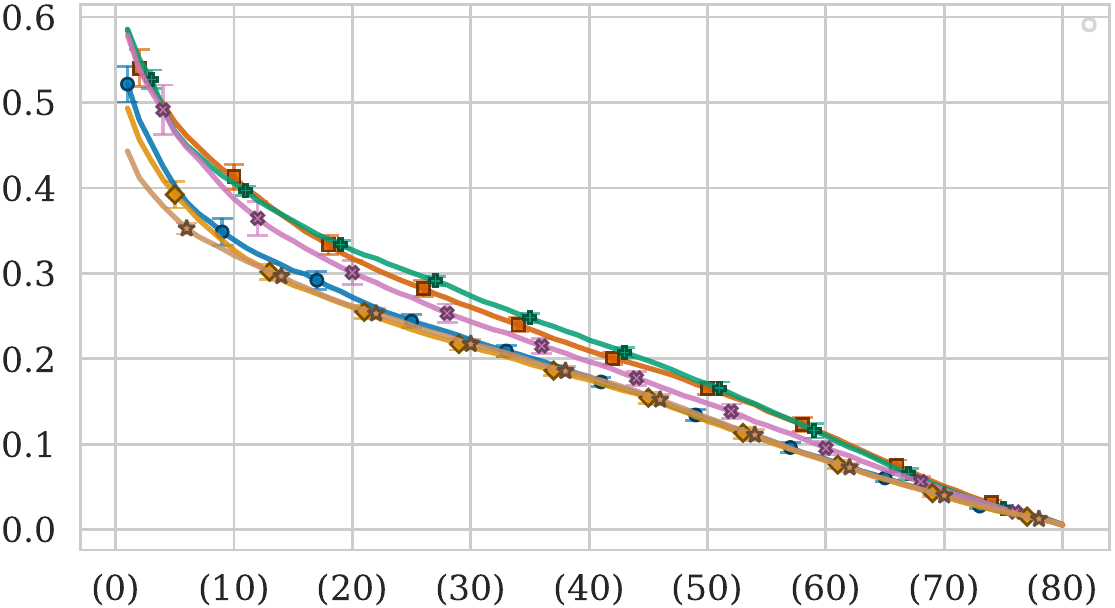}
    \end{minipage}\hfill \begin{minipage}[t]{0.32\linewidth}
        \includegraphics[width=\linewidth]{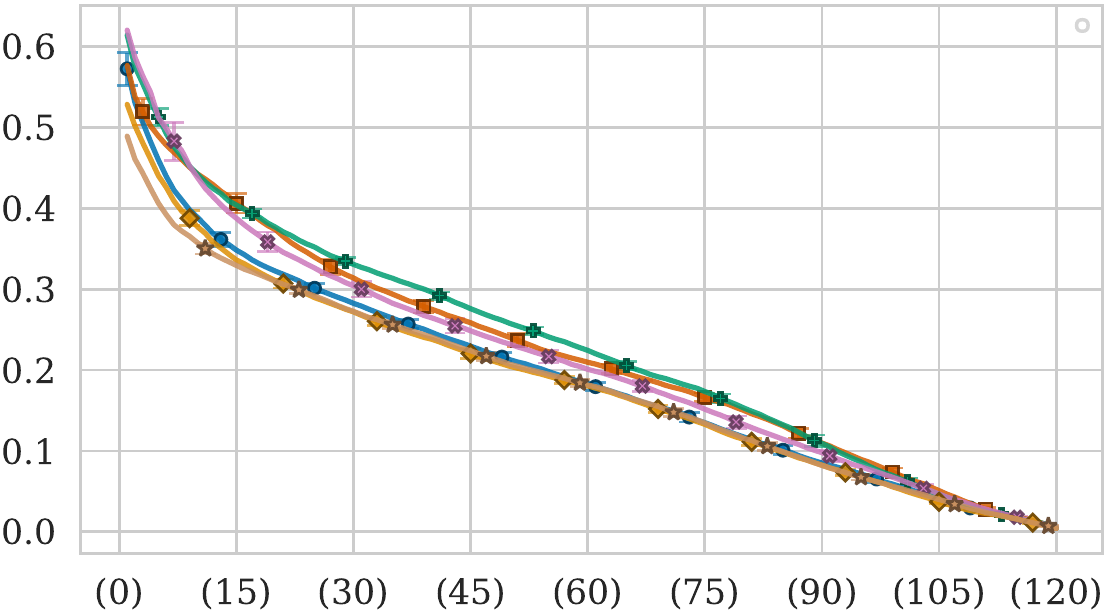}
    \end{minipage}\vspace{0.5em}

\begin{minipage}[t]{0.32\linewidth}
        \includegraphics[width=\linewidth]{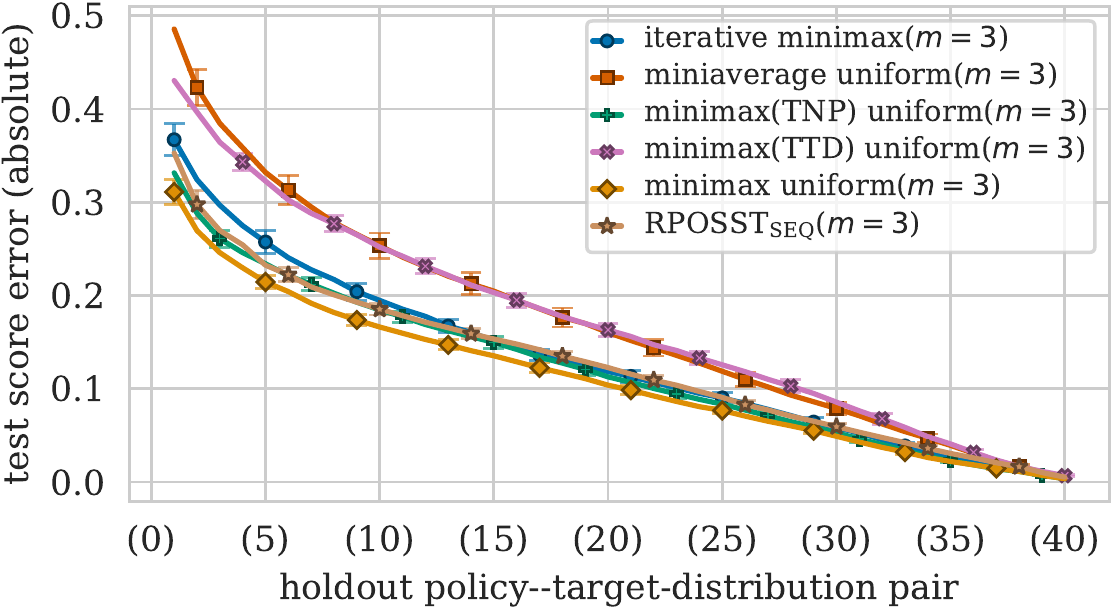}
    \end{minipage}\hfill \begin{minipage}[t]{0.32\linewidth}
        \includegraphics[width=\linewidth]{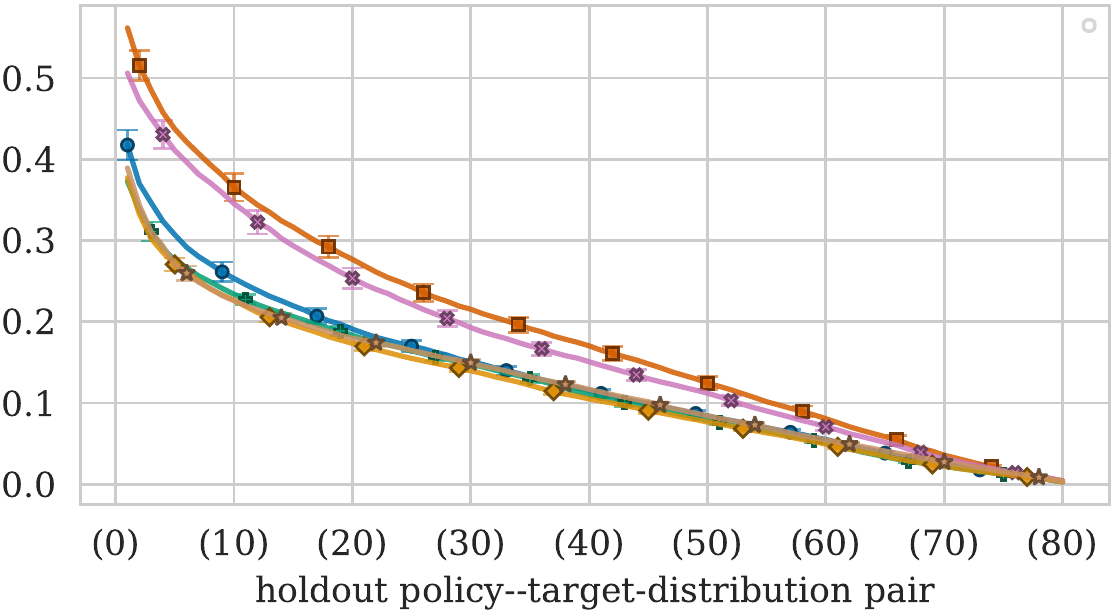}
    \end{minipage}\hfill \begin{minipage}[t]{0.32\linewidth}
        \includegraphics[width=\linewidth]{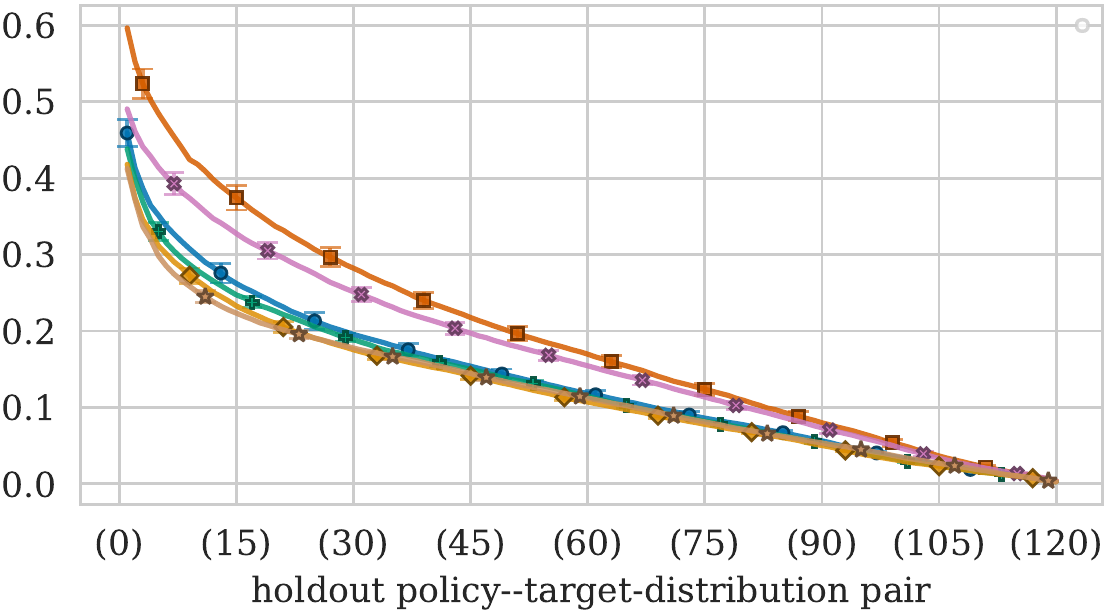}
    \end{minipage}\caption{Expected test score error (absolute difference) across holdout-policy--target-distribution pairs on Racing Arrows where test cases are 50 leader policies. Each row uses a different setting for the test size ($m = 1$ top, $m = 2$ middle, and $m = 3$ bottom) and each column uses a different holdout proportion ($20\%$ held out in the left column, $40\%$ middle, and $60\%$ right). $\numHoldoutReplicas$ sets of holdout policies were sampled. Holdout-policy--target-distribution pairs are sorted according to test score error. Each RPOSST$_{\seqLabel}$ instance was run for $500$ rounds ($T = 500$). Errorbars represent $95\%$ t-distribution confidence intervals.}
    \label{fig:racing-arrows-l-50}
\end{figure*}

\begin{figure*}[t]
\begin{minipage}[t]{0.32\linewidth}
        \includegraphics[width=\linewidth]{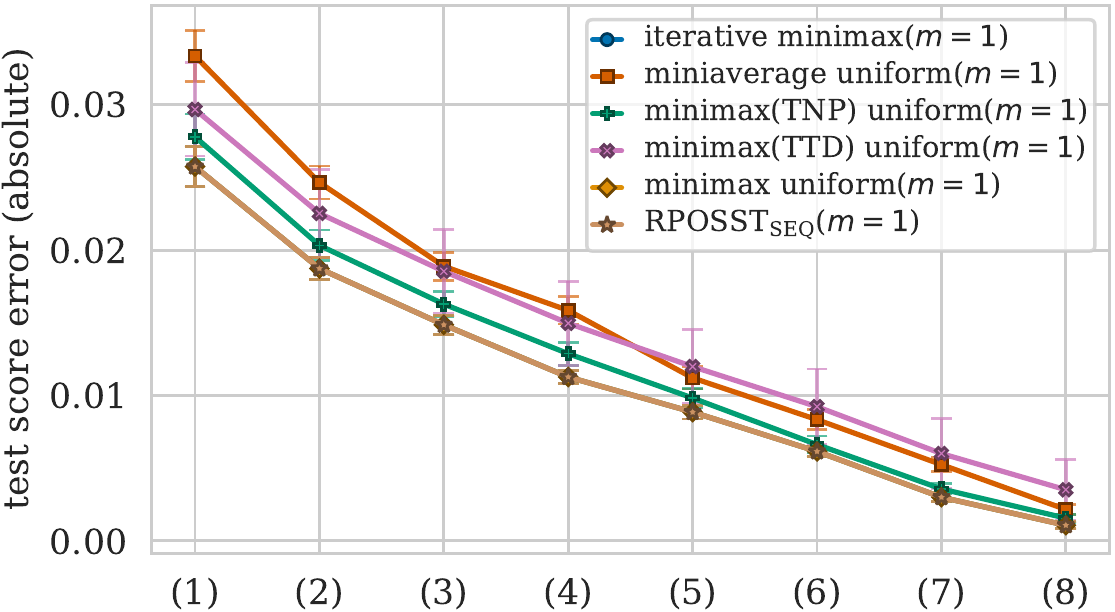}
    \end{minipage}\hfill \begin{minipage}[t]{0.32\linewidth}
        \includegraphics[width=\linewidth]{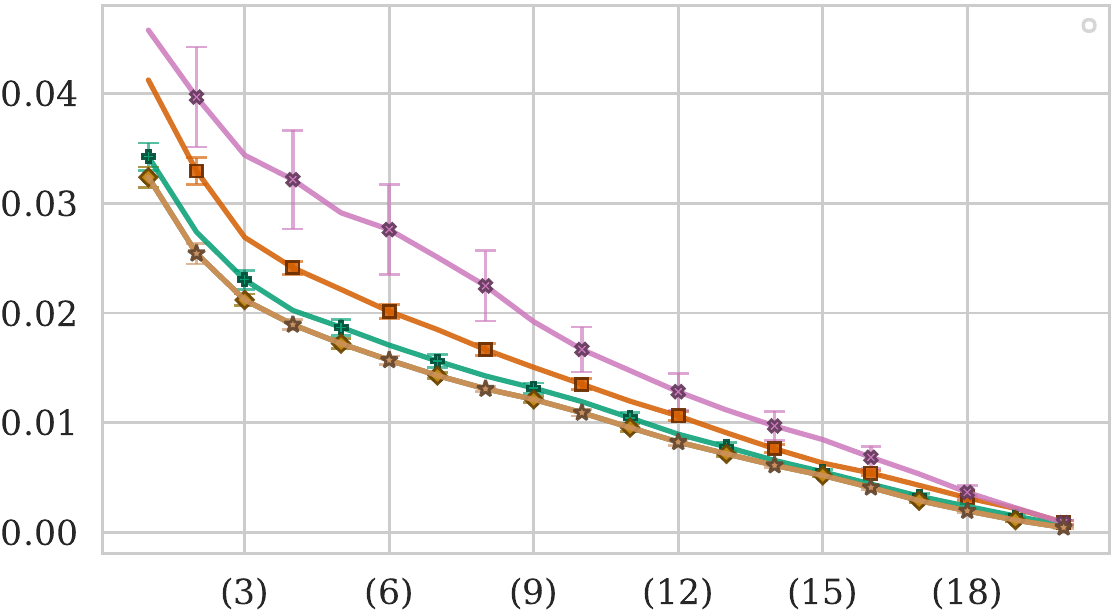}
    \end{minipage}\hfill \begin{minipage}[t]{0.32\linewidth}
        \includegraphics[width=\linewidth]{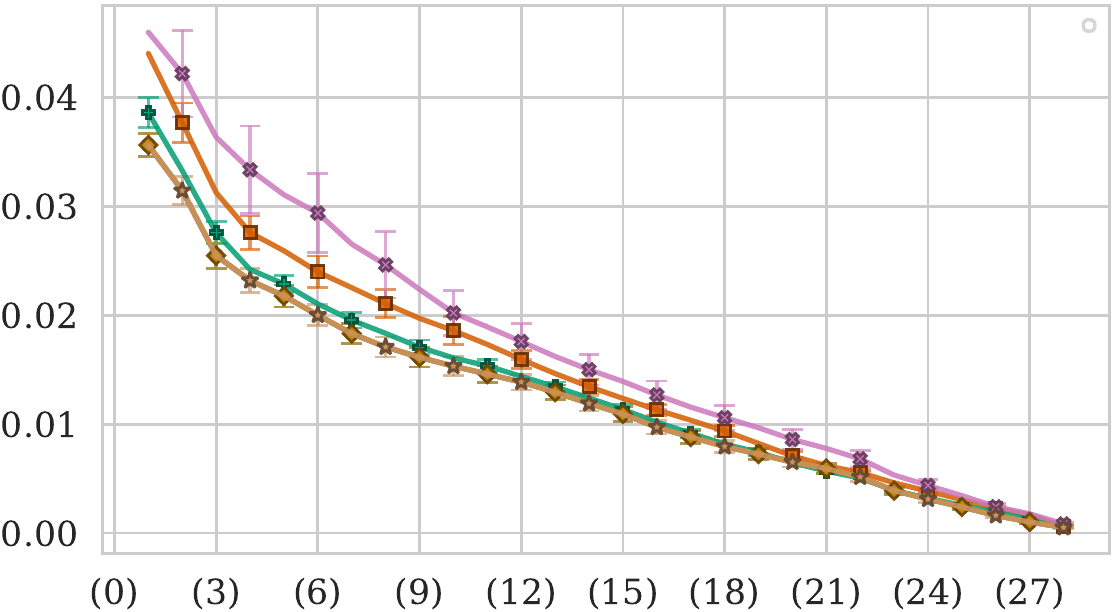}
    \end{minipage}\vspace{0.5em}

\begin{minipage}[t]{0.32\linewidth}
        \includegraphics[width=\linewidth]{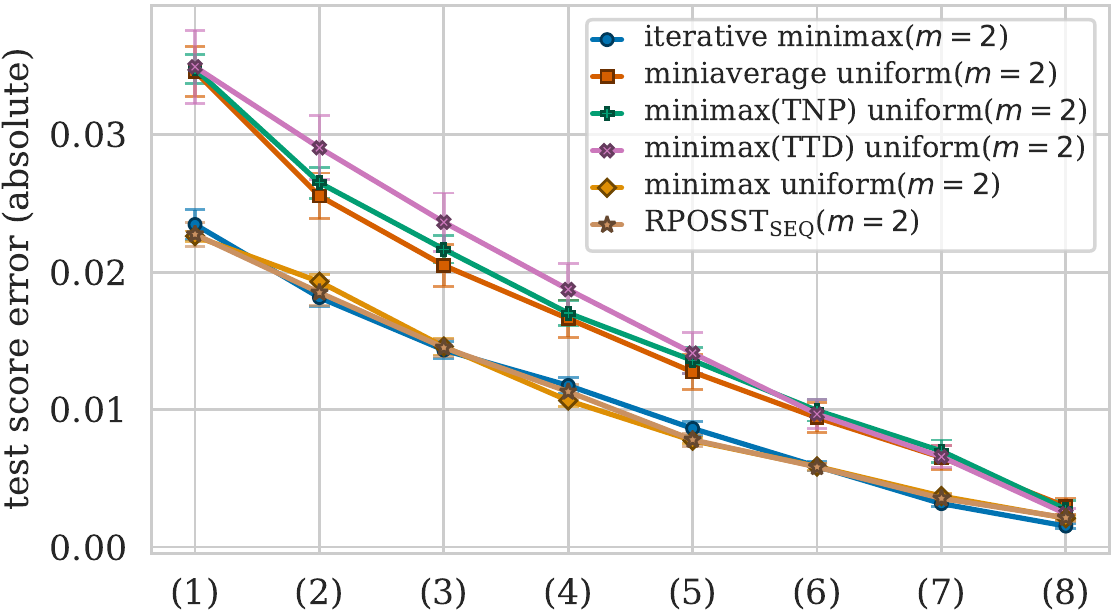}
    \end{minipage}\hfill \begin{minipage}[t]{0.32\linewidth}
        \includegraphics[width=\linewidth]{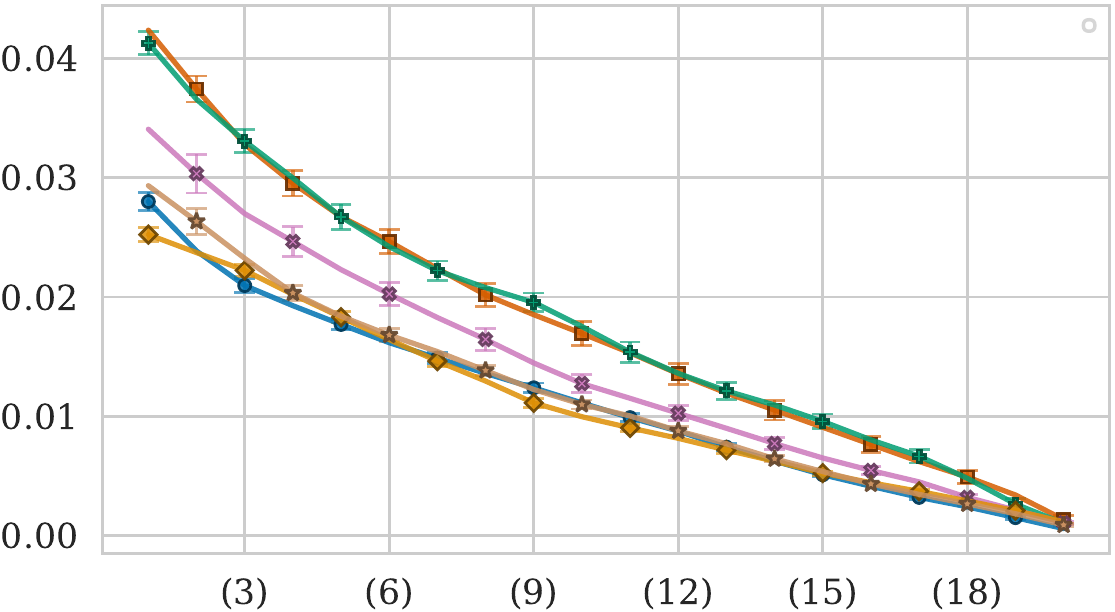}
    \end{minipage}\hfill \begin{minipage}[t]{0.32\linewidth}
        \includegraphics[width=\linewidth]{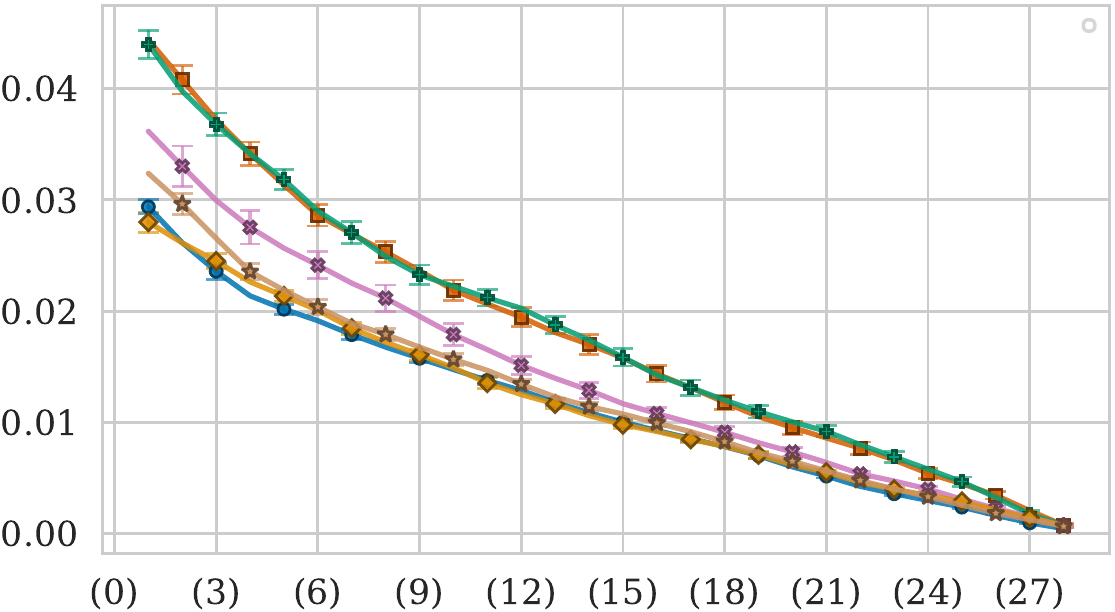}
    \end{minipage}\vspace{0.5em}

\begin{minipage}[t]{0.32\linewidth}
        \includegraphics[width=\linewidth]{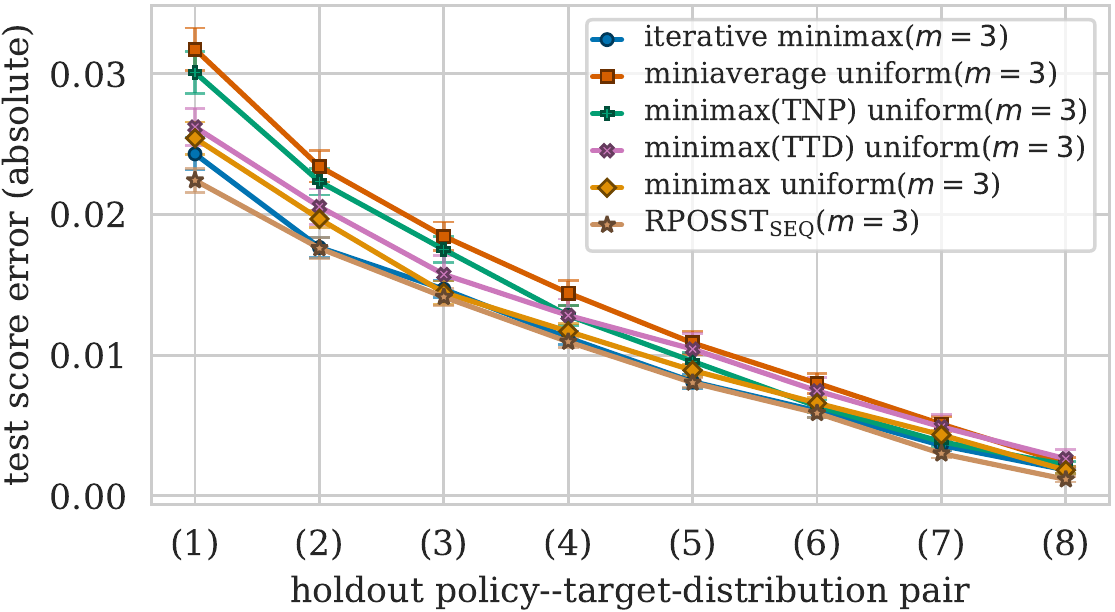}
    \end{minipage}\hfill \begin{minipage}[t]{0.32\linewidth}
        \includegraphics[width=\linewidth]{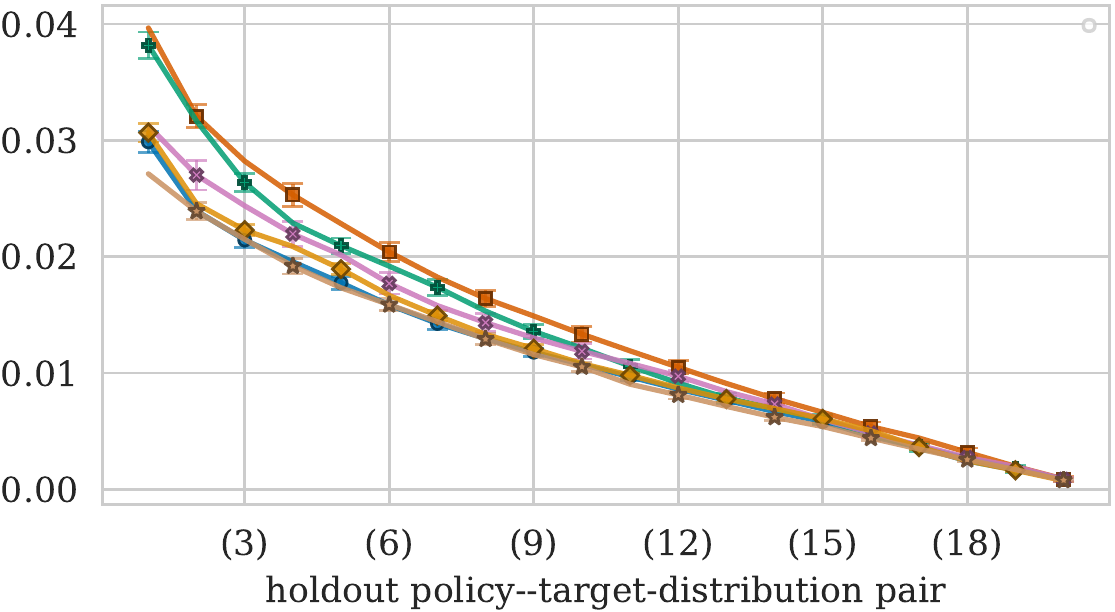}
    \end{minipage}\hfill \begin{minipage}[t]{0.32\linewidth}
        \includegraphics[width=\linewidth]{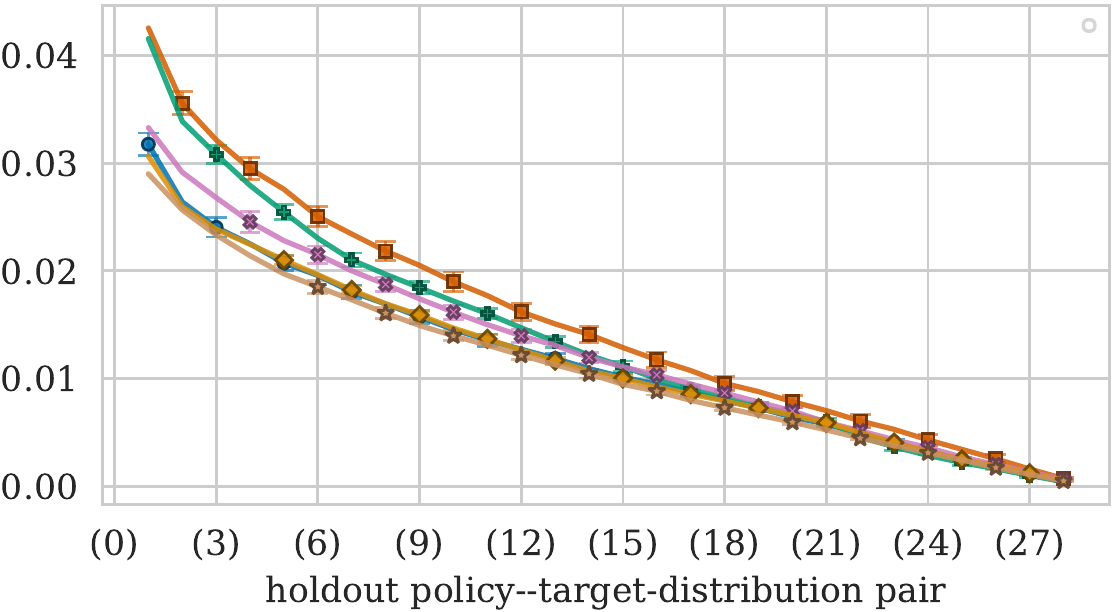}
    \end{minipage}\caption{Expected test score error (absolute difference) across holdout-policy--target-distribution pairs on the ACPC 2012 data. Each row uses a different setting for the test size ($m = 1$ top, $m = 2$ middle, and $m = 3$ bottom) and each column uses a different holdout proportion ($20\%$ held out in the left column, $40\%$ middle, and $60\%$ right). $\numHoldoutReplicas$ sets of holdout policies were sampled. Holdout-policy--target-distribution pairs are sorted according to test score error. Each RPOSST$_{\seqLabel}$ instance was run for $500$ rounds ($T = 500$). Errorbars represent $95\%$ t-distribution confidence intervals.}
    \label{fig:acpc2012}
\end{figure*}

\begin{figure*}[t]
\begin{minipage}[t]{0.32\linewidth}
        \includegraphics[width=\linewidth]{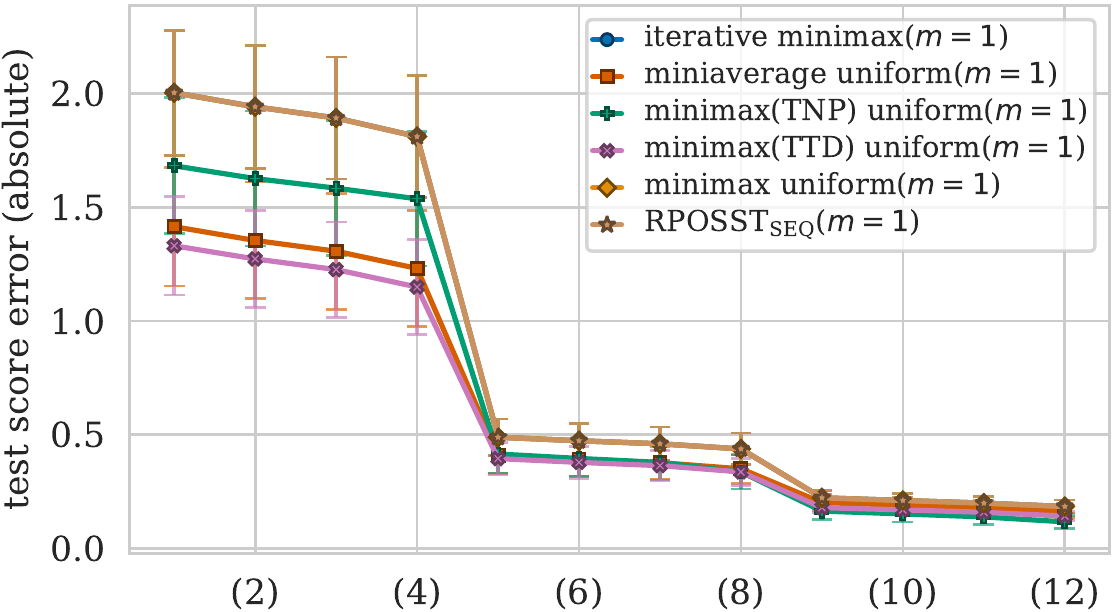}
    \end{minipage}\hfill \begin{minipage}[t]{0.32\linewidth}
        \includegraphics[width=\linewidth]{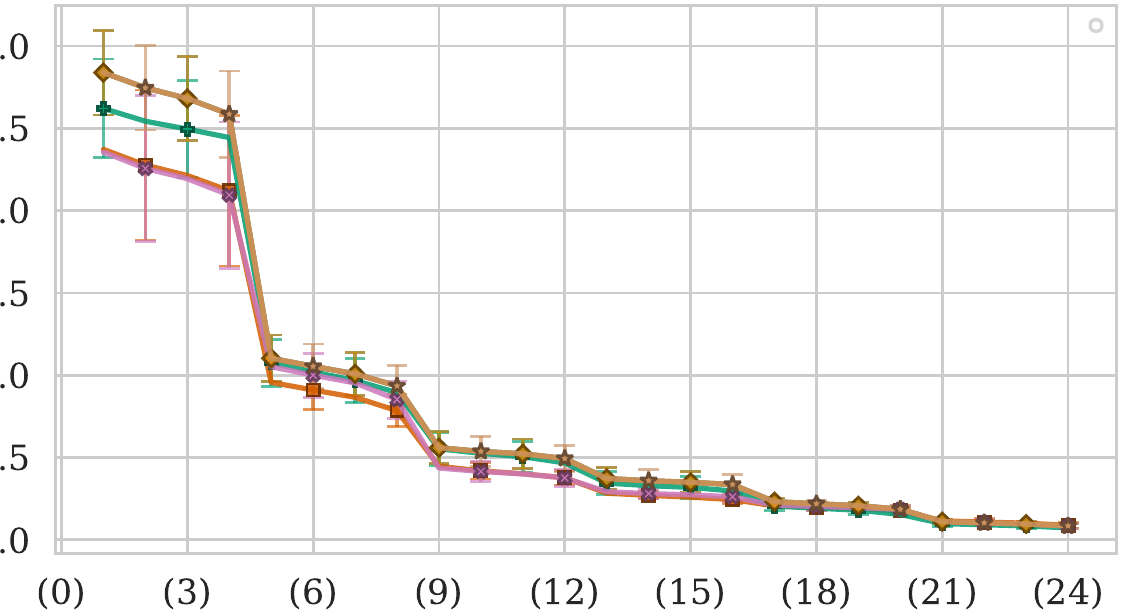}
    \end{minipage}\hfill \begin{minipage}[t]{0.32\linewidth}
        \includegraphics[width=\linewidth]{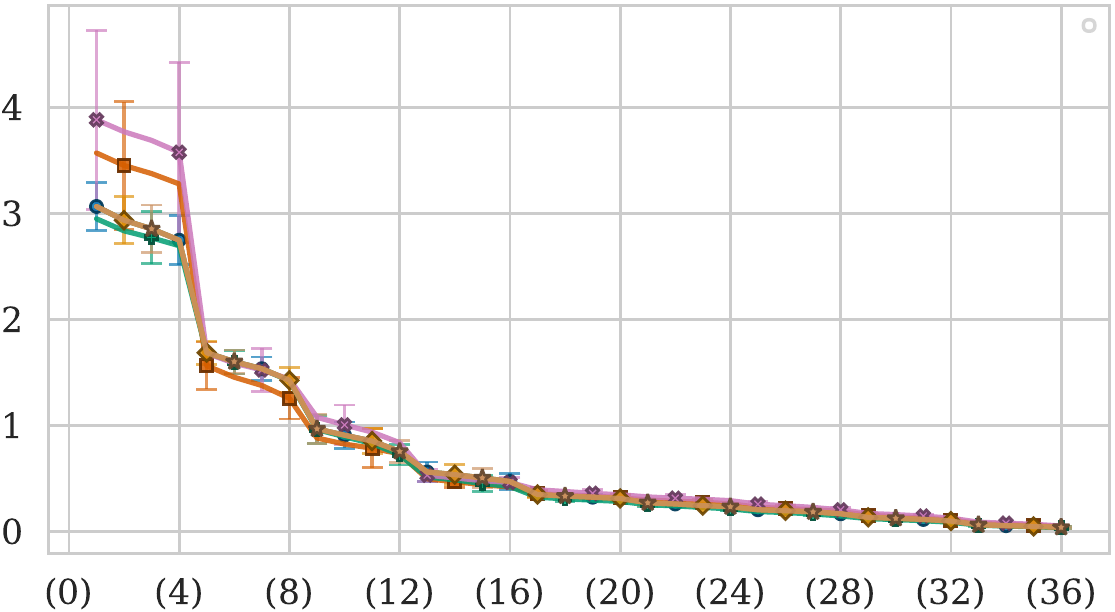}
    \end{minipage}\vspace{0.5em}

\begin{minipage}[t]{0.32\linewidth}
        \includegraphics[width=\linewidth]{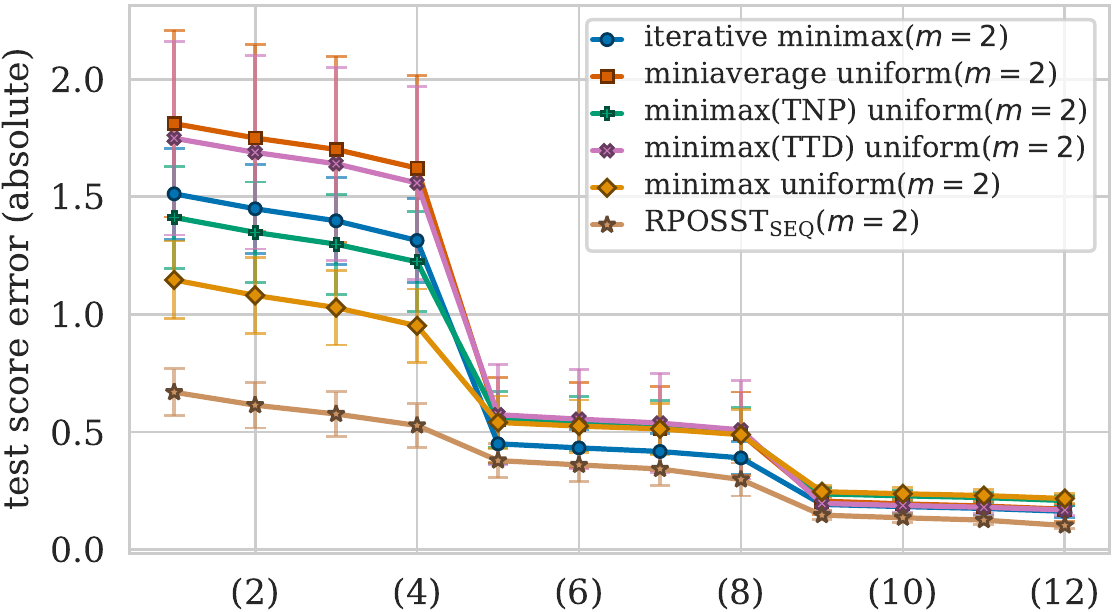}
    \end{minipage}\hfill \begin{minipage}[t]{0.32\linewidth}
        \includegraphics[width=\linewidth]{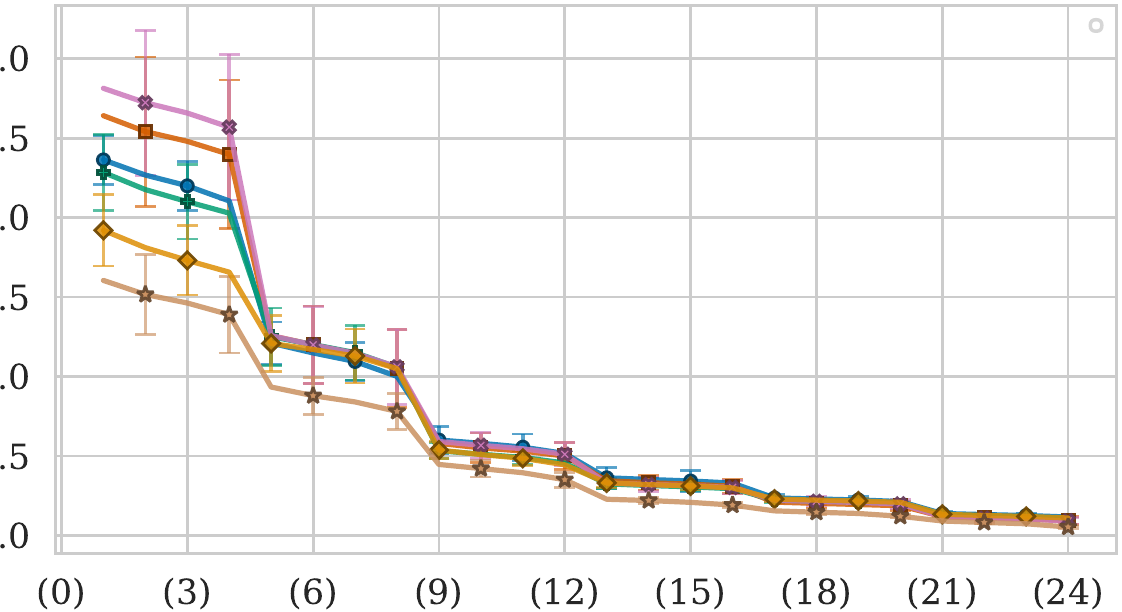}
    \end{minipage}\hfill \begin{minipage}[t]{0.32\linewidth}
        \includegraphics[width=\linewidth]{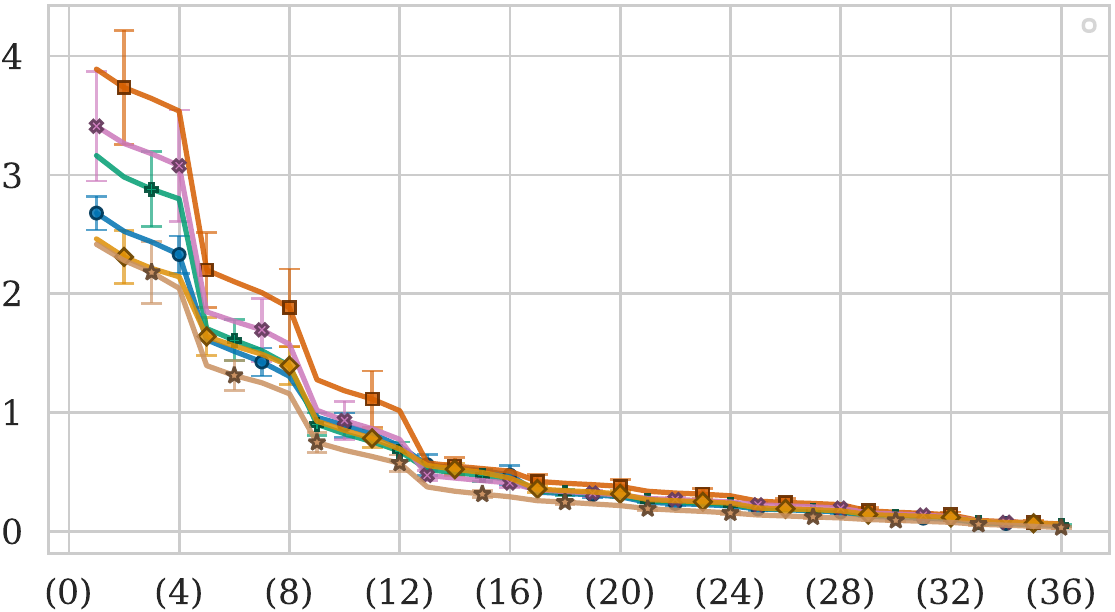}
    \end{minipage}\vspace{0.5em}

\begin{minipage}[t]{0.32\linewidth}
        \includegraphics[width=\linewidth]{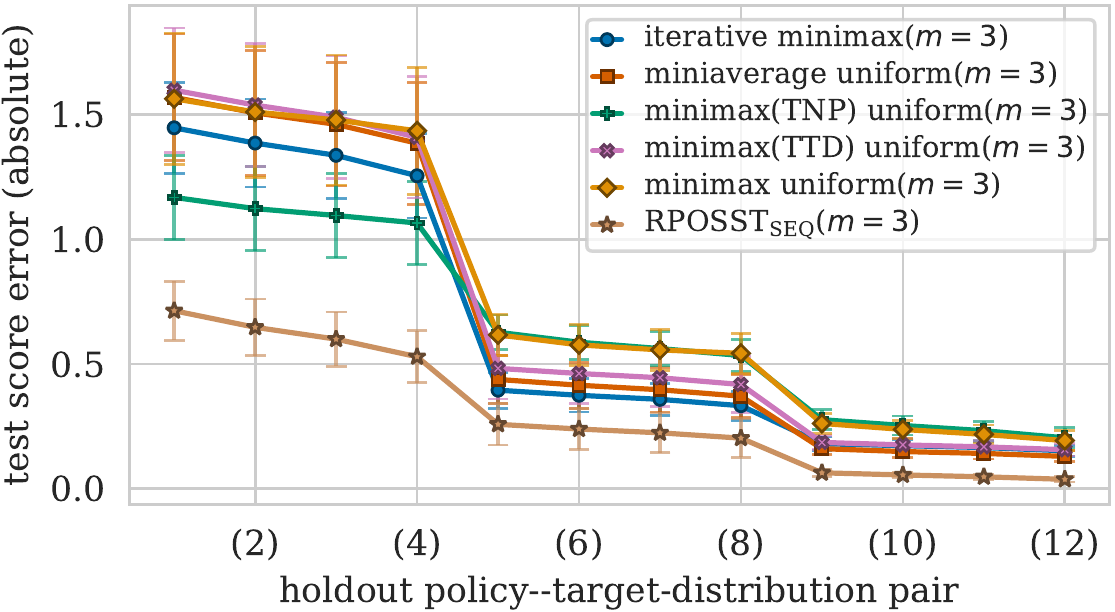}
    \end{minipage}\hfill \begin{minipage}[t]{0.32\linewidth}
        \includegraphics[width=\linewidth]{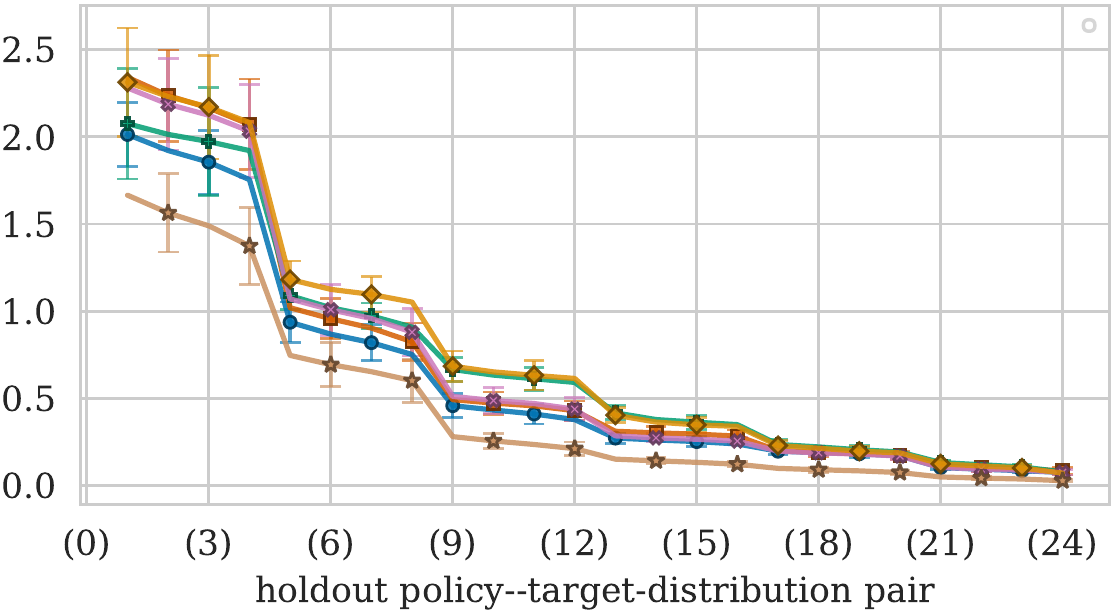}
    \end{minipage}\hfill \begin{minipage}[t]{0.32\linewidth}
        \includegraphics[width=\linewidth]{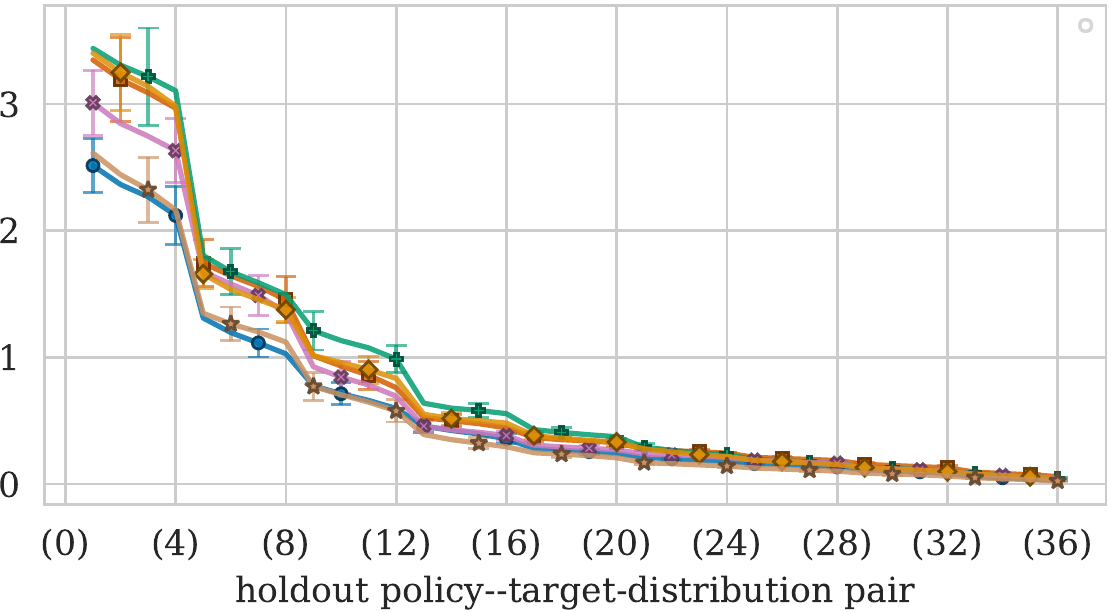}
    \end{minipage}\caption{Expected test score error (absolute difference) across holdout-policy--target-distribution pairs on the ACPC 2017 data. Each row uses a different setting for the test size ($m = 1$ top, $m = 2$ middle, and $m = 3$ bottom) and each column uses a different holdout proportion ($20\%$ held out in the left column, $40\%$ middle, and $60\%$ right). $\numHoldoutReplicas$ sets of holdout policies were sampled. Holdout-policy--target-distribution pairs are sorted according to test score error. Each RPOSST$_{\seqLabel}$ instance was run for $500$ rounds ($T = 500$). Errorbars represent $95\%$ t-distribution confidence intervals.}
    \label{fig:acpc2017}
\end{figure*}

\begin{figure*}[t]
\begin{minipage}[t]{0.32\linewidth}
        \includegraphics[width=\linewidth]{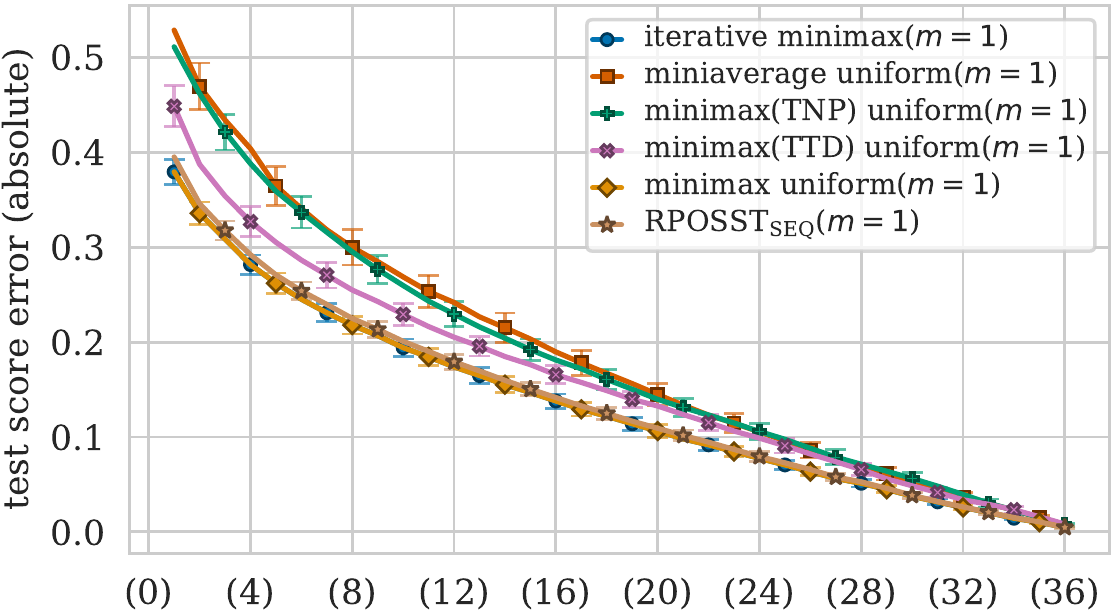}
    \end{minipage}\hfill \begin{minipage}[t]{0.32\linewidth}
        \includegraphics[width=\linewidth]{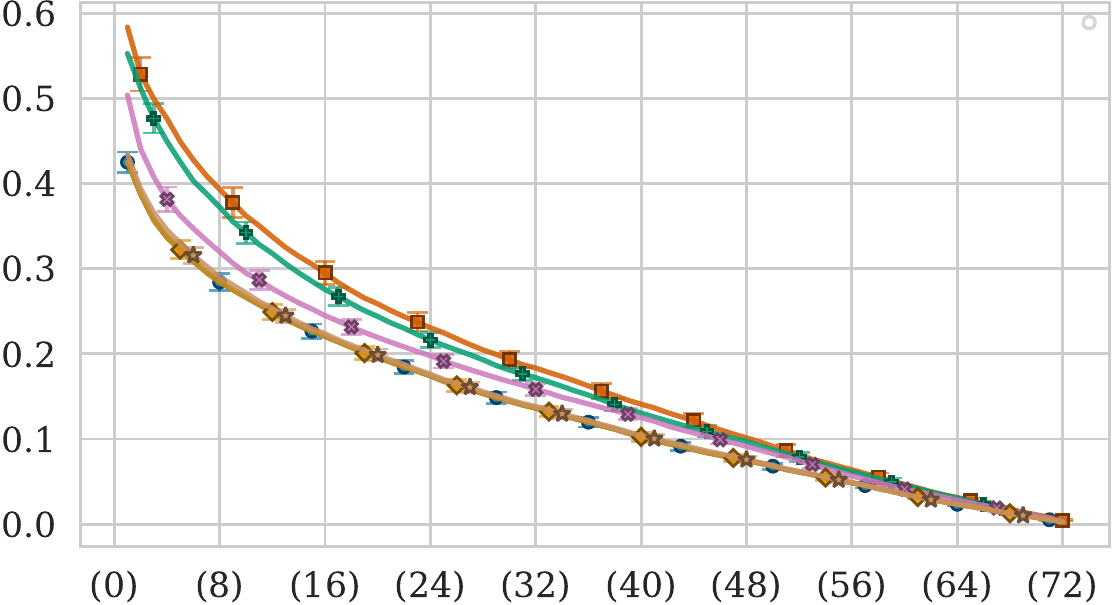}
    \end{minipage}\hfill \begin{minipage}[t]{0.32\linewidth}
        \includegraphics[width=\linewidth]{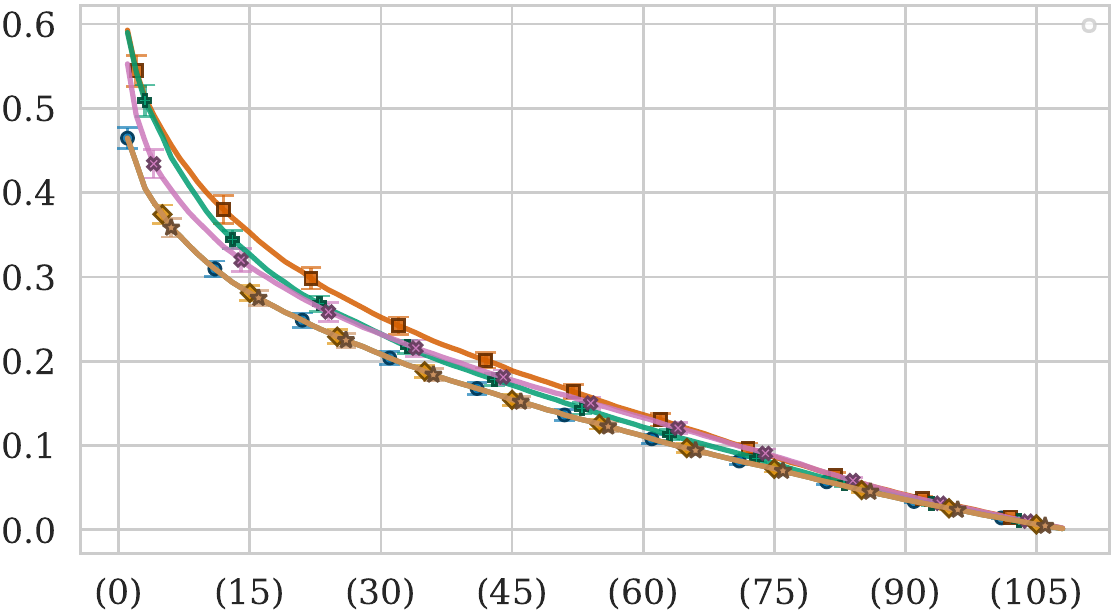}
    \end{minipage}\vspace{0.5em}

\begin{minipage}[t]{0.32\linewidth}
        \includegraphics[width=\linewidth]{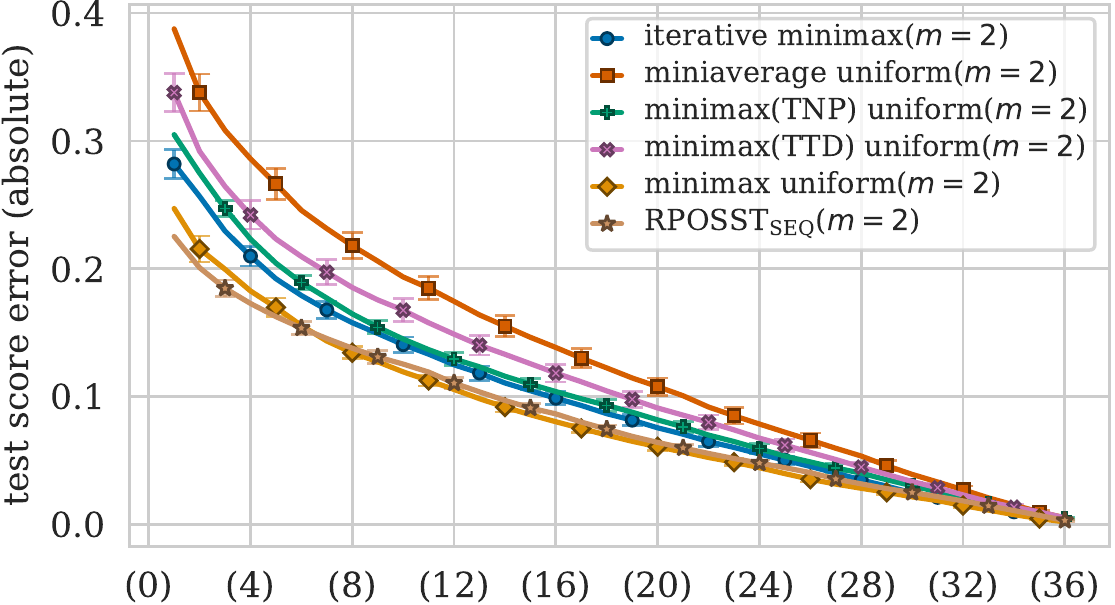}
    \end{minipage}\hfill \begin{minipage}[t]{0.32\linewidth}
        \includegraphics[width=\linewidth]{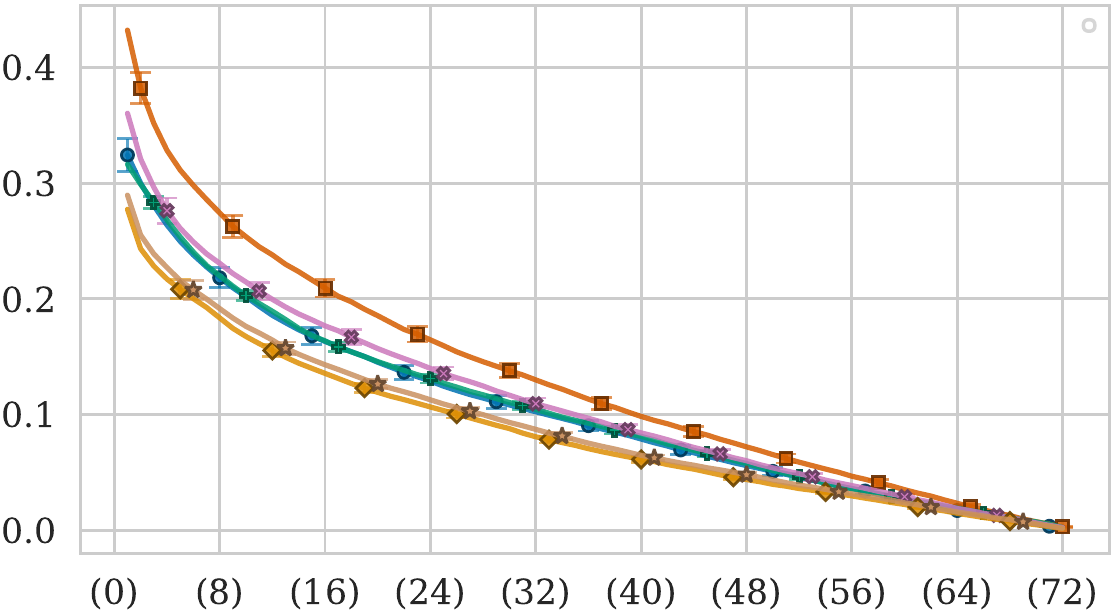}
    \end{minipage}\hfill \begin{minipage}[t]{0.32\linewidth}
        \includegraphics[width=\linewidth]{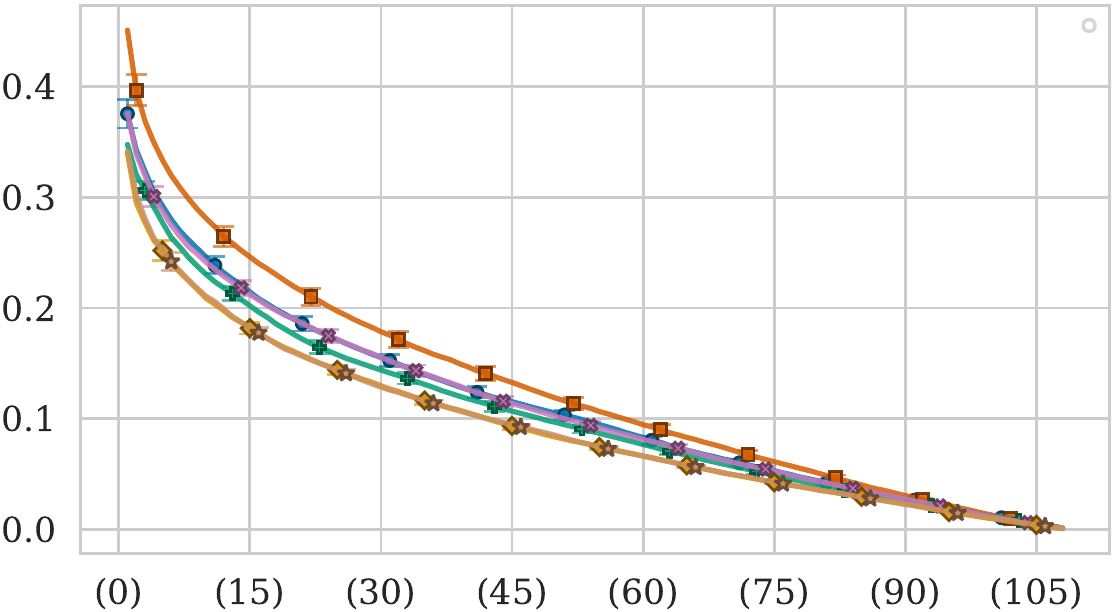}
    \end{minipage}\vspace{0.5em}

\begin{minipage}[t]{0.32\linewidth}
        \includegraphics[width=\linewidth]{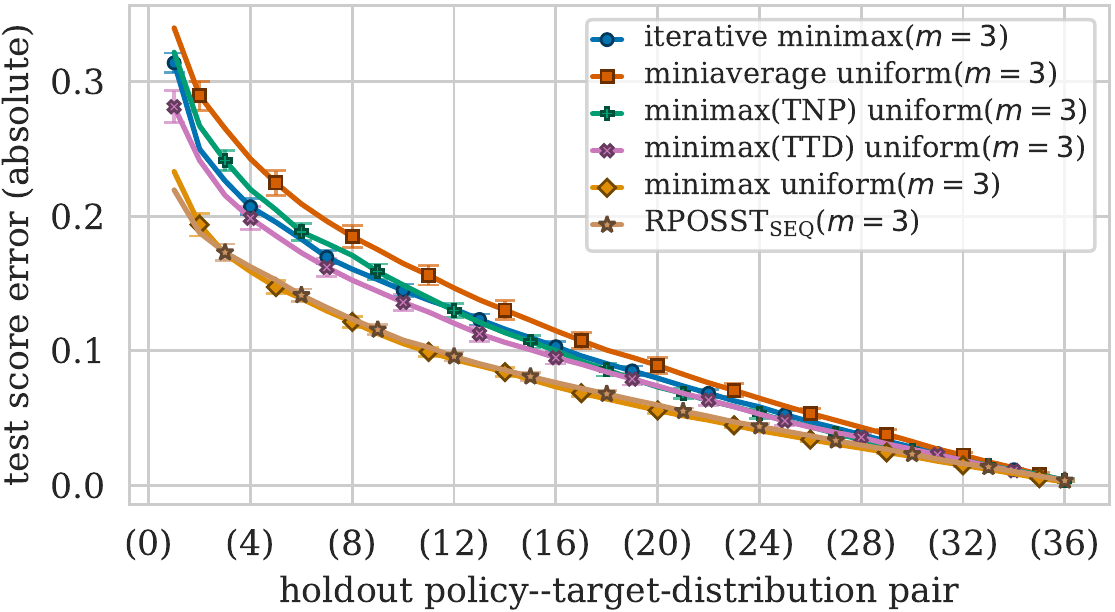}
    \end{minipage}\hfill \begin{minipage}[t]{0.32\linewidth}
        \includegraphics[width=\linewidth]{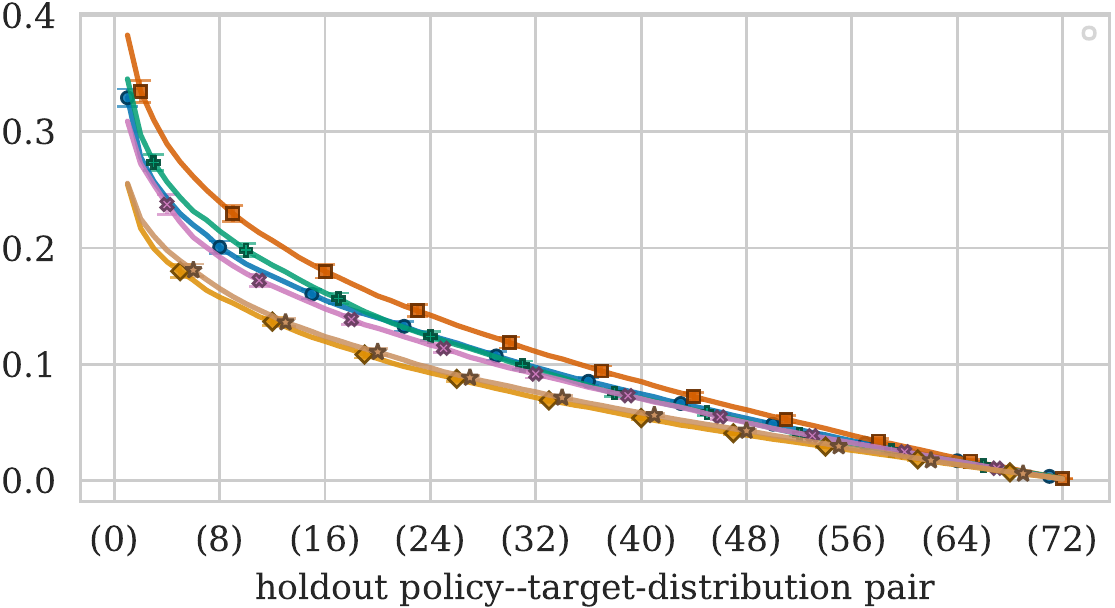}
    \end{minipage}\hfill \begin{minipage}[t]{0.32\linewidth}
        \includegraphics[width=\linewidth]{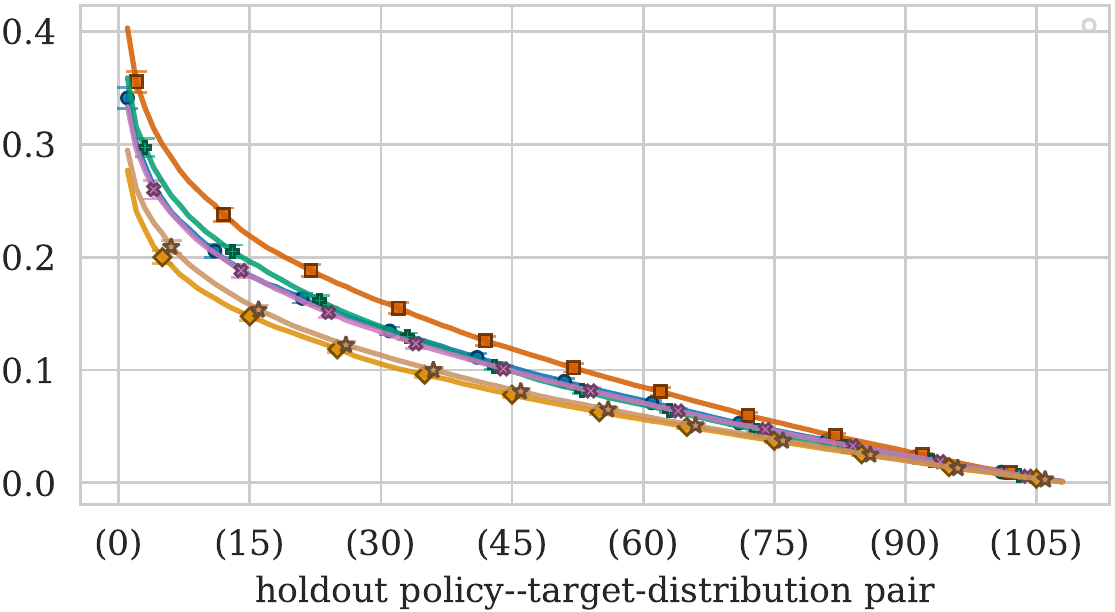}
    \end{minipage}\caption{Expected test score error (absolute difference) across holdout-policy--target-distribution pairs in the winrate GT domain. Each row uses a different setting for the test size ($m = 1$ top, $m = 2$ middle, and $m = 3$ bottom) and each column uses a different holdout proportion ($20\%$ held out in the left column, $40\%$ middle, and $60\%$ right). $\numHoldoutReplicas$ sets of holdout policies were sampled. Holdout-policy--target-distribution pairs are sorted according to test score error. Each RPOSST$_{\seqLabel}$ instance was run for $500$ rounds ($T = 500$). Errorbars represent $95\%$ t-distribution confidence intervals.}
    \label{fig:gt-large-matrix}
\end{figure*}

\begin{figure*}[t]
\begin{minipage}[t]{0.32\linewidth}
        \includegraphics[width=\linewidth]{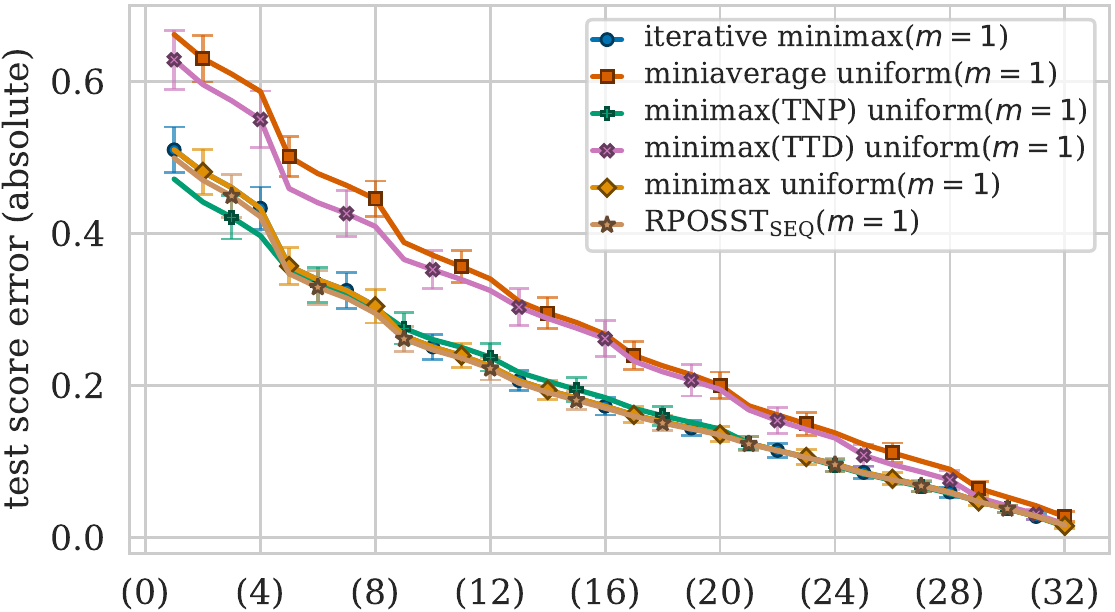}
    \end{minipage}\hfill \begin{minipage}[t]{0.32\linewidth}
        \includegraphics[width=\linewidth]{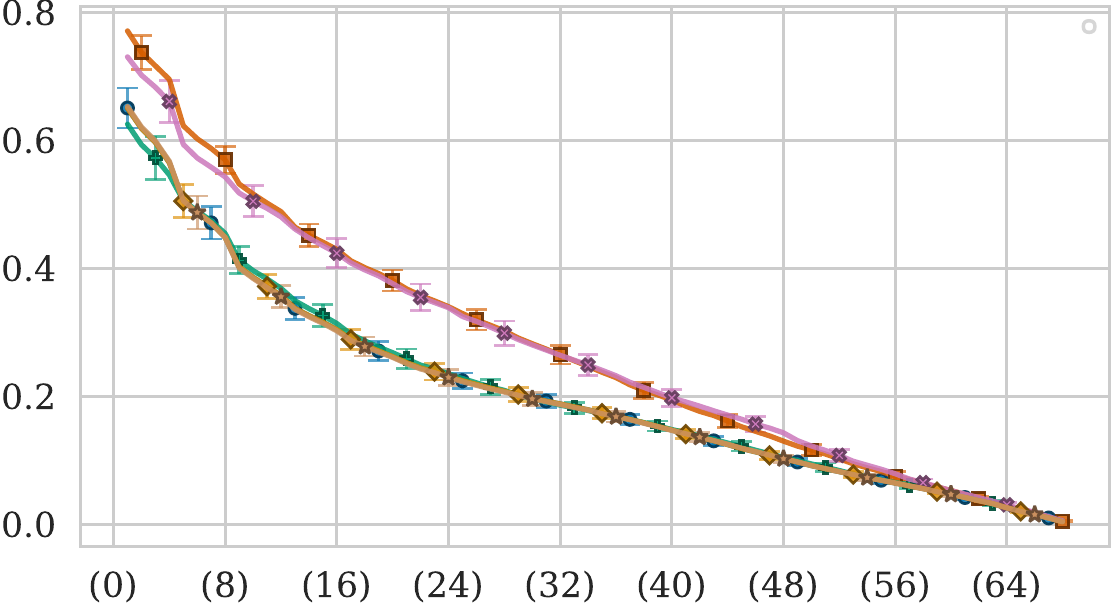}
    \end{minipage}\hfill \begin{minipage}[t]{0.32\linewidth}
        \includegraphics[width=\linewidth]{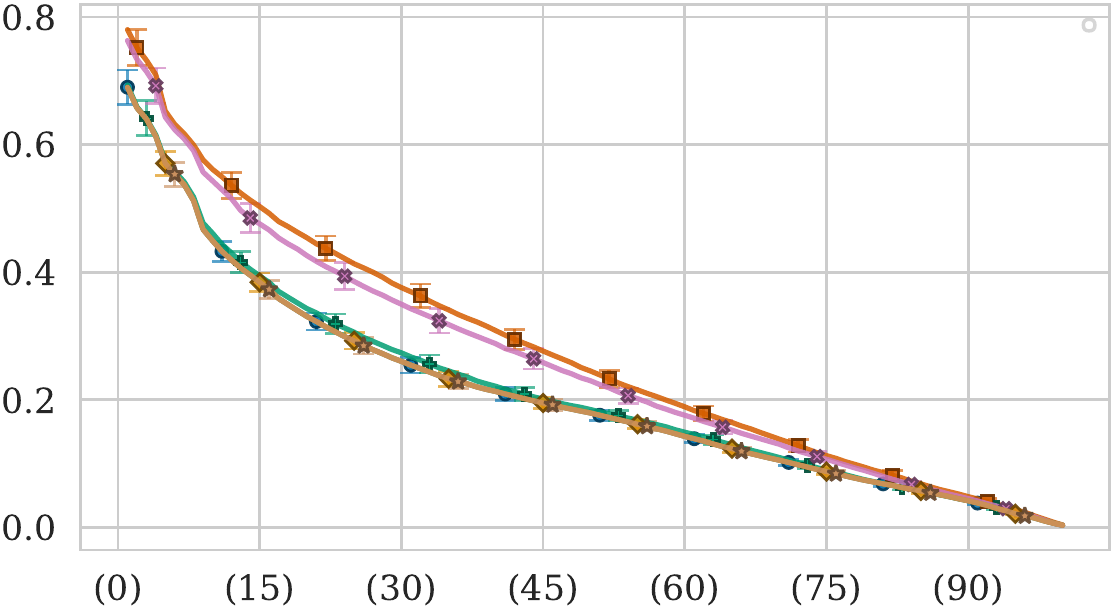}
    \end{minipage}\vspace{0.5em}

\begin{minipage}[t]{0.32\linewidth}
        \includegraphics[width=\linewidth]{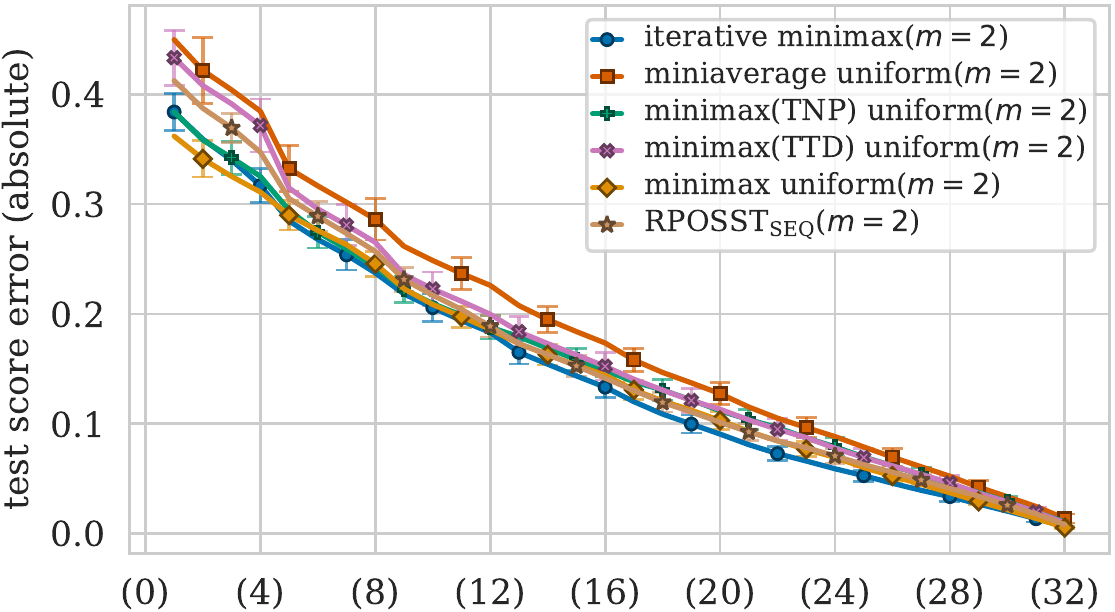}
    \end{minipage}\hfill \begin{minipage}[t]{0.32\linewidth}
        \includegraphics[width=\linewidth]{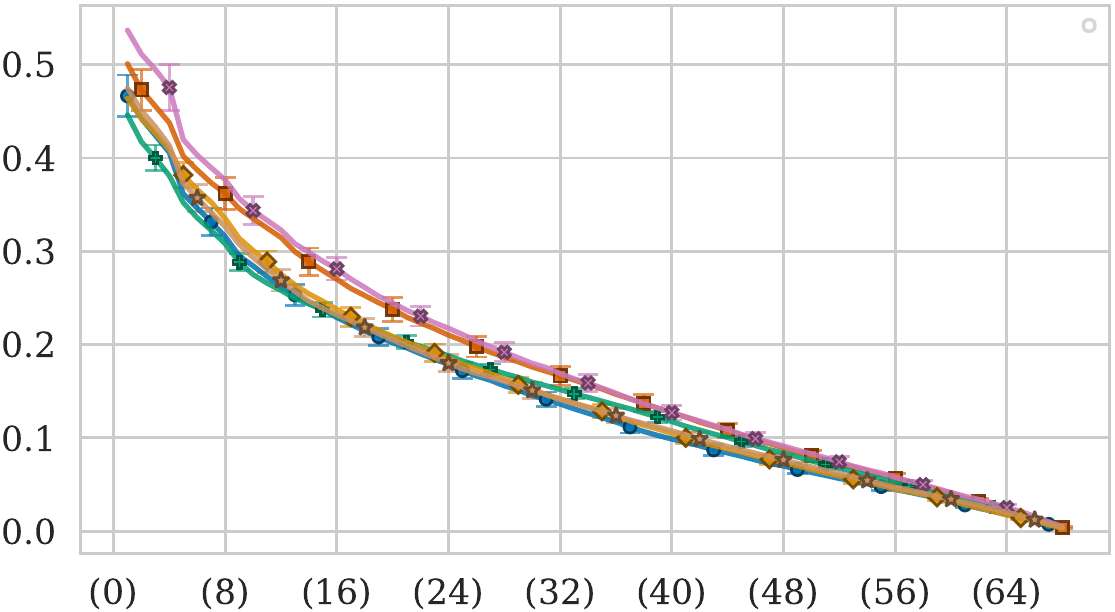}
    \end{minipage}\hfill \begin{minipage}[t]{0.32\linewidth}
        \includegraphics[width=\linewidth]{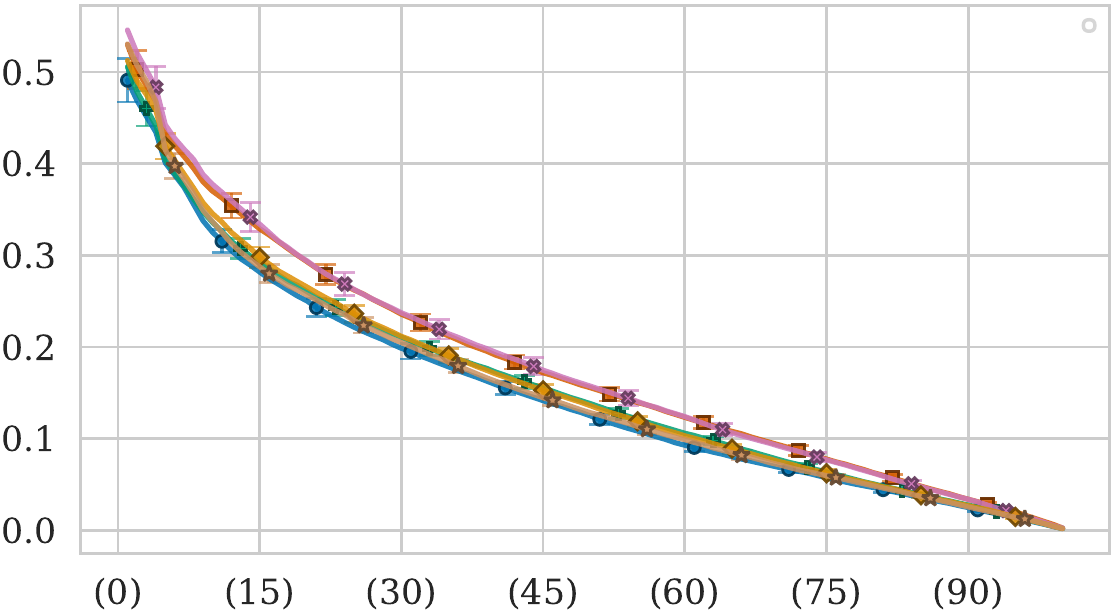}
    \end{minipage}\vspace{0.5em}

\begin{minipage}[t]{0.32\linewidth}
        \includegraphics[width=\linewidth]{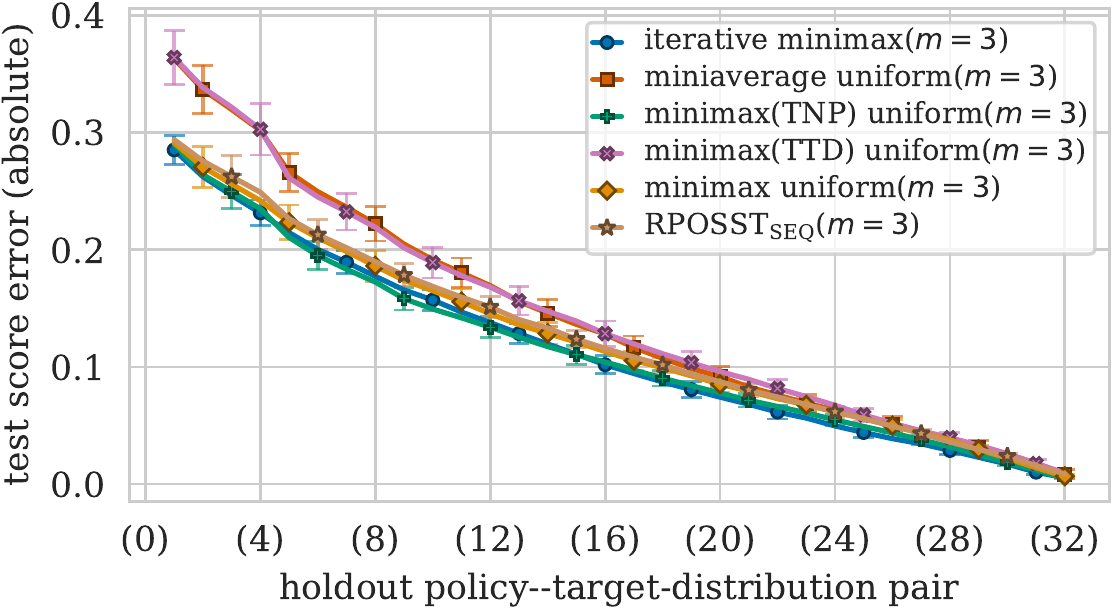}
    \end{minipage}\hfill \begin{minipage}[t]{0.32\linewidth}
        \includegraphics[width=\linewidth]{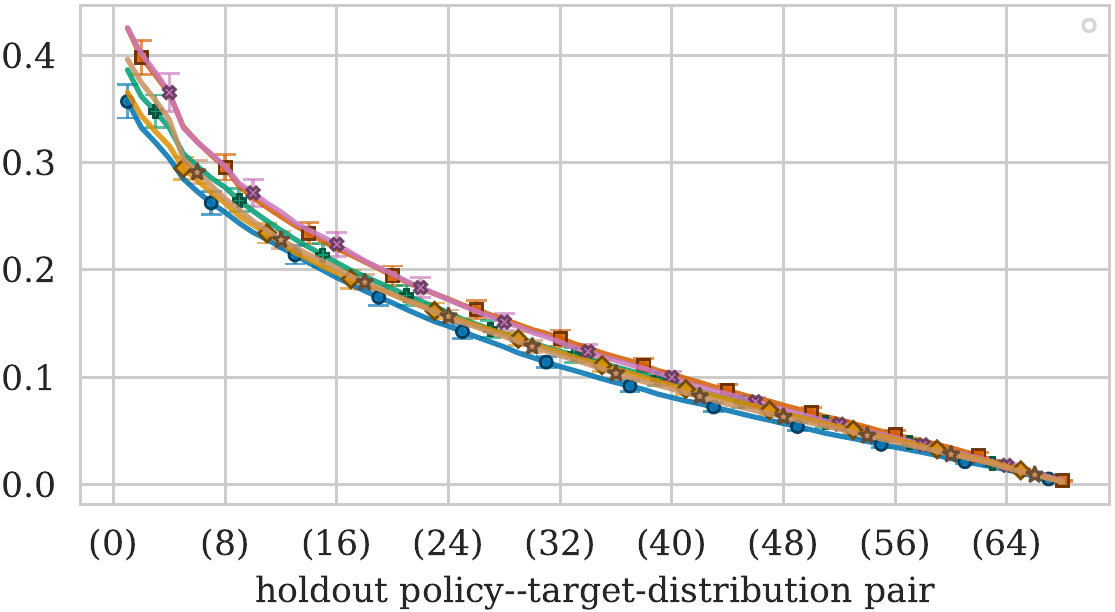}
    \end{minipage}\hfill \begin{minipage}[t]{0.32\linewidth}
        \includegraphics[width=\linewidth]{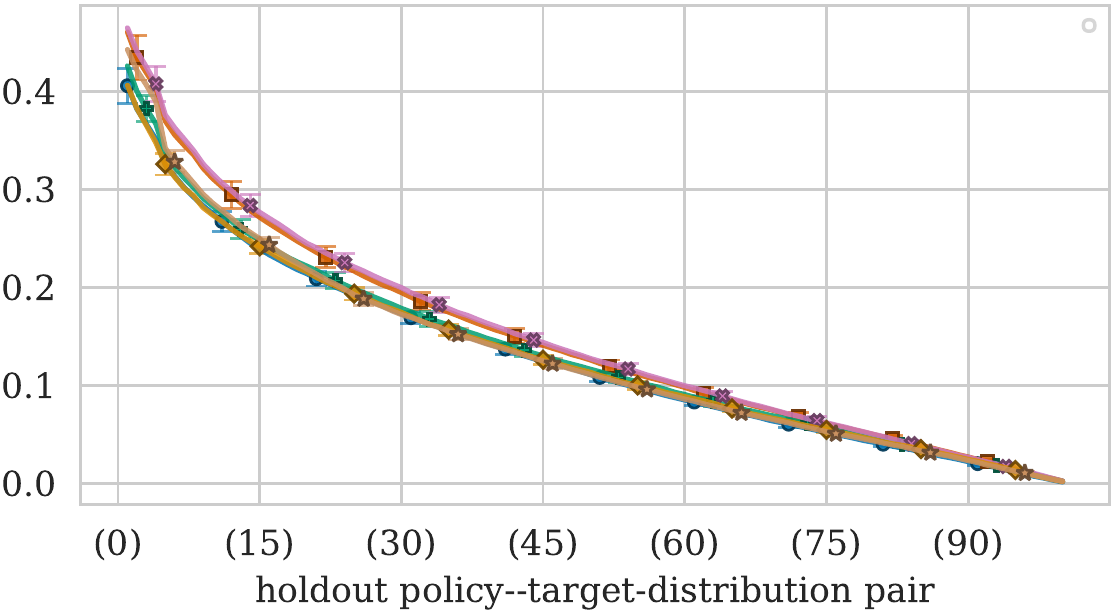}
    \end{minipage}\caption{Expected test score error (absolute difference) across holdout-policy--target-distribution pairs in the GT domain where -1 is given for a collision. Each row uses a different setting for the test size ($m = 1$ top, $m = 2$ middle, and $m = 3$ bottom) and each column uses a different holdout proportion ($20\%$ held out in the left column, $40\%$ middle, and $60\%$ right). $\numHoldoutReplicas$ sets of holdout policies were sampled. Holdout-policy--target-distribution pairs are sorted according to test score error. Each RPOSST$_{\seqLabel}$ instance was run for $500$ rounds ($T = 500$). Errorbars represent $95\%$ t-distribution confidence intervals.}
    \label{fig:gt-non-zero-sum}
\end{figure*}

\begin{figure*}[t]
    \begin{subfigure}[t]{0.32\linewidth}
      \includegraphics[width=\linewidth]{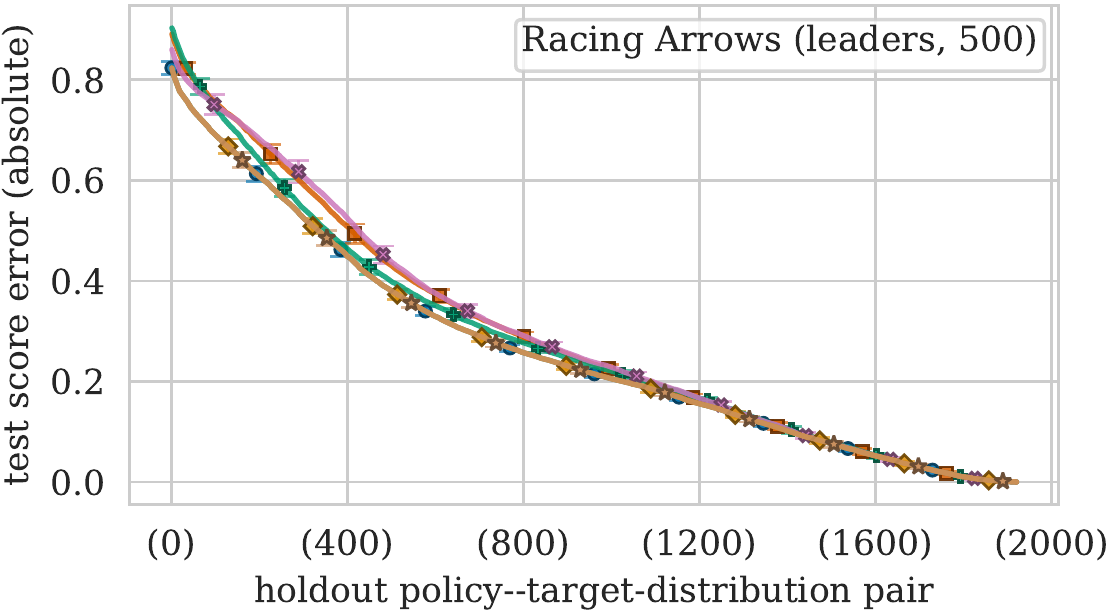}
      \caption{$\mTestCasesToSelect = 1$}
  \end{subfigure}\hfill \begin{subfigure}[t]{0.32\linewidth}
      \includegraphics[width=\linewidth]{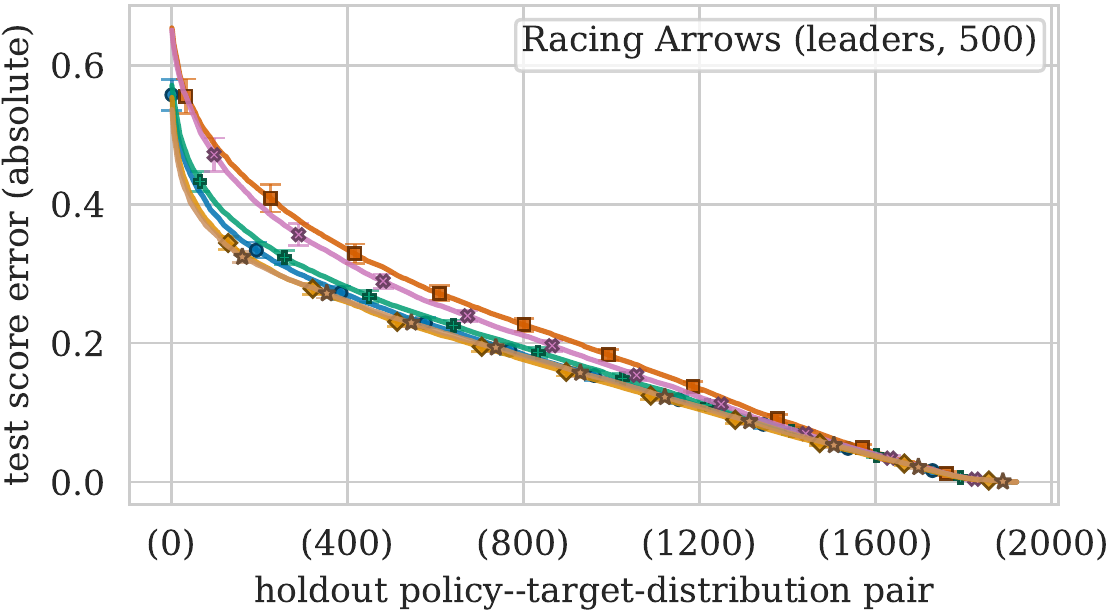}
      \caption{$\mTestCasesToSelect = 2$}
  \end{subfigure}\hfill \begin{subfigure}[t]{0.32\linewidth}
      \includegraphics[width=\linewidth]{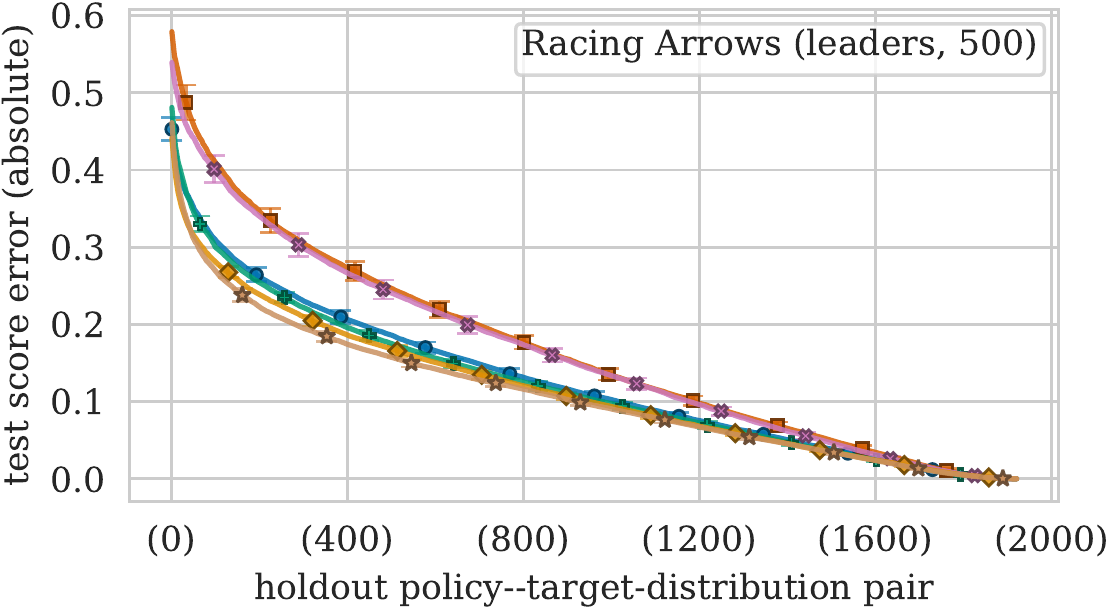}
      \caption{$\mTestCasesToSelect = 3$}
  \end{subfigure}
    \caption{Expected test score error (absolute difference) across holdout-policy--target-distribution pairs on Racing Arrows where test cases are leader policies. Here, 500 Racing Arrows policies were sampled for both the follower and leader role and then $96\%$ of policies of both roles were held out before running RPOSST and each baseline. Each column uses a different setting for the test size ($m = 1$ top, $m = 2$ middle, and $m = 3$ bottom). $\numHoldoutReplicas$ sets of holdout policies were sampled. Holdout-policy--target-distribution pairs are sorted according to test score error. Each RPOSST$_{\seqLabel}$ instance was run for $500$ rounds ($T = 500$). Errorbars represent $95\%$ t-distribution confidence intervals.}
    \label{fig:racing-arrows-l-500}
\end{figure*}

We note that as $\mTestCasesToSelect$ increases, the error on the holdout set typically decreases, particularly for RPOSST, since larger tests have the capacity to be strictly more accurate. A qualitative analysis of these results suggests that there are few substantial differences between RPOSST tests of different sizes or with different, reasonably sized, holdout sets. Furthermore, the performance ordering of the tested algorithms remains the same as the results presented in the main paper.
 \end{document}